\documentclass{article}

\PassOptionsToPackage{numbers, compress}{natbib}

\usepackage[final]{neurips_2023}

\usepackage[utf8]{inputenc} 
\usepackage[T1]{fontenc}    
\usepackage{hyperref}       
\usepackage{url}            
\usepackage{booktabs}       
\usepackage{amsfonts}       
\usepackage{nicefrac}       
\usepackage{microtype}      
\usepackage{xcolor}         
\usepackage{amsthm, amssymb, amsfonts, latexsym, mathtools}
\usepackage{algorithm}
\usepackage[noend]{algpseudocode}
\usepackage{comment}
\usepackage{here}
\usepackage{caption}
\usepackage{subcaption}
\usepackage{wrapfig}
\usepackage{multirow}
\usepackage{color}
\usepackage{wrapfig}

\definecolor{limegreen}{rgb}{0.2, 0.8, 0.2}
\definecolor{cherryblossompink}{rgb}{1.0, 0.72, 0.77}
\definecolor{darkpastelblue}{rgb}{0.47, 0.62, 0.8}
\definecolor{darkpastelpurple}{rgb}{0.59, 0.44, 0.84}
\definecolor{columbiablue}{rgb}{0.61, 0.87, 1.0}

\newcommand{\proposed}[1]{\textsc{Base-#1 Graph}}
\newcommand{\Proposed}[1]{\textsc{Base-#1 Graph}}
\newcommand{\simpleProposed}[1]{\textsc{Simple Base-#1 Graph}}
\newcommand{\SimpleProposed}[1]{\textsc{Simple Base-#1 Graph}}
\newcommand{\hyperhypercube}[1]{\textsc{#1-peer Hyper-hypercube Graph}}

\newtheorem{theorem}{Theorem}
\newtheorem{lemma}{Lemma}
\newtheorem{corollary}{Corollary}

\newtheorem{definition}{Definition}
\newtheorem{assumption}{Assumption}

\usepackage{amsmath,amsfonts,bm}

\def\eqref#1{Eq.~(\ref{#1})}

\def\ceil#1{\lceil #1 \rceil}

\def\1{\bm{1}}

\def\vx{{\bm{x}}}

\def\mW{{\bm{W}}}
\def\mX{{\bm{X}}}

\DeclareMathAlphabet{\mathsfit}{\encodingdefault}{\sfdefault}{m}{sl}
\SetMathAlphabet{\mathsfit}{bold}{\encodingdefault}{\sfdefault}{bx}{n}

\title{Beyond Exponential Graph:\\ Communication-Efficient Topologies for Decentralized Learning via Finite-time Convergence}

\author{%
  Yuki Takezawa$^{1,2}$\thanks{Equal Contribution}, Ryoma Sato$^{1,2*}$, Han Bao$^{1,2}$, Kenta Niwa$^{3}$, Makoto Yamada$^{2}$ \\
  $^1$Kyoto University, $^2$OIST, $^3$NTT Communication Science Laboratories
}

\begin{document}
\maketitle

\begin{abstract}
Decentralized learning has recently been attracting increasing attention for its applications in parallel computation and privacy preservation.
Many recent studies stated that the underlying network topology with a faster consensus rate (a.k.a. spectral gap) leads to a better convergence rate and accuracy for decentralized learning.
However, a topology with a fast consensus rate, e.g., the exponential graph, generally has a large maximum degree,
which incurs significant communication costs.
Thus, seeking topologies with both a fast consensus rate and small maximum degree is important.
In this study, we propose a novel topology combining both a fast consensus rate and small maximum degree called the \proposed{$\left(k+1\right)$}.
Unlike the existing topologies,
the \proposed{$(k+1)$} enables all nodes to reach the exact consensus after a finite number of iterations for any number of nodes and maximum degree $k$.
Thanks to this favorable property, 
the \proposed{$(k+1)$} endows Decentralized SGD (DSGD) with both a faster convergence rate and more communication efficiency than the exponential graph.
We conducted experiments with various topologies,
demonstrating that the \proposed{$(k+1)$} enables various decentralized learning methods to achieve higher accuracy with better communication efficiency than the existing topologies.
Our code is available at \url{https://github.com/yukiTakezawa/BaseGraph}.
\end{abstract}

\section{Introduction}
Distributed learning, which allows training neural networks in parallel on multiple nodes,
has become an important paradigm owing to the increased utilization of privacy preservation and large-scale machine learning.
In a centralized fashion, such as All-Reduce and Federated Learning \cite{kairouz2021adcanves,karimireddy2020scaffold,tian2020federated,mcmahan2017communication,murata2021bias},
all or some selected nodes update their parameters by using their local dataset
and then compute the average parameter of these nodes,
although computing the average of many nodes is the major bottleneck in the training time \cite{lian2017can,lian2018asynchronous,lu2020moniqua}.
To reduce communication costs, decentralized learning gains significant attention \cite{koloskova2020unified,lian2017can,lu2021optimal}.
Because decentralized learning allows nodes to exchange parameters only with a few neighbors in the underlying network topology,
decentralized learning is more communication efficient than All-Reduce and Federated Learning.

While decentralized learning can improve communication efficiency,
it may degrade the convergence rate and accuracy due to its sparse communication characteristics \cite{koloskova2020unified,zhu2022topology}.
Specifically, the smaller the maximum degree of an underlying network topology is, the fewer the communication cost becomes \cite{song2022communication,wang2019matcha}; 
meanwhile, the faster the consensus rate (a.k.a. spectral gap) of a topology is, the faster the convergence rate of decentralized learning becomes \cite{koloskova2020unified}.
Thus, developing a topology with both a fast consensus rate and small maximum degree is essential for decentralized learning.
Table \ref{table:comparison_of_graphs} summarizes the properties of various topologies.
For instance, the ring and exponential graph are commonly used \cite{assran2019stochastic,cyffers2022muffliato,koloskova2019decentralized,lu2020moniqua}.
The ring is a communication-efficient topology because its maximum degree is two 
but its consensus rate quickly deteriorates as the number of nodes $n$ increases \cite{nedic2018network}.
The exponential graph has a fast consensus rate, which does not deteriorate much as $n$ increases,
but it incurs significant communication costs 
because its maximum degree increases as $n$ increases \cite{ying2021exponential}.
Thus, these topologies sacrifice either communication efficiency or consensus rate.

Recently, the $1$-peer exponential graph \cite{ying2021exponential} and $1$-peer hypercube graph \cite{shi2016finite} were proposed as topologies that combine both a small maximum degree and fast consensus rate (see Sec.~\ref{sec:1_peer_exp_and_1_peer_hypercube} for examples).
As Fig.~\ref{fig:consensus_5000} shows,
in the ring and exponential graph, node parameters only reach the consensus asymptotically by repeating exchanges of parameters with neighbors.
Contrarily, in the $1$-peer exponential and $1$-peer hypercube graphs, 
parameters reach the exact consensus after a finite number of iterations when $n$ is a power of $2$
(see Fig.~\ref{fig:consensus_with_power_of_two} in Sec.~\ref{sec:other_consensns_optimization}).
Thanks to this property of finite-time convergence,
the $1$-peer exponential and $1$-peer hypercube graphs
enable Decentralized SGD (DSGD) \cite{lian2017can} to converge
at the same convergence rate as the exponential graph when $n$ is a power of $2$,
even though the maximum degree of the $1$-peer exponential and $1$-peer hypercube graphs is only one \cite{ying2021exponential}.
\begin{wrapfigure}{r}[0pt]{0.45\textwidth}
\vskip - 0.15 in
 \centering
   \includegraphics[width=0.45\textwidth]{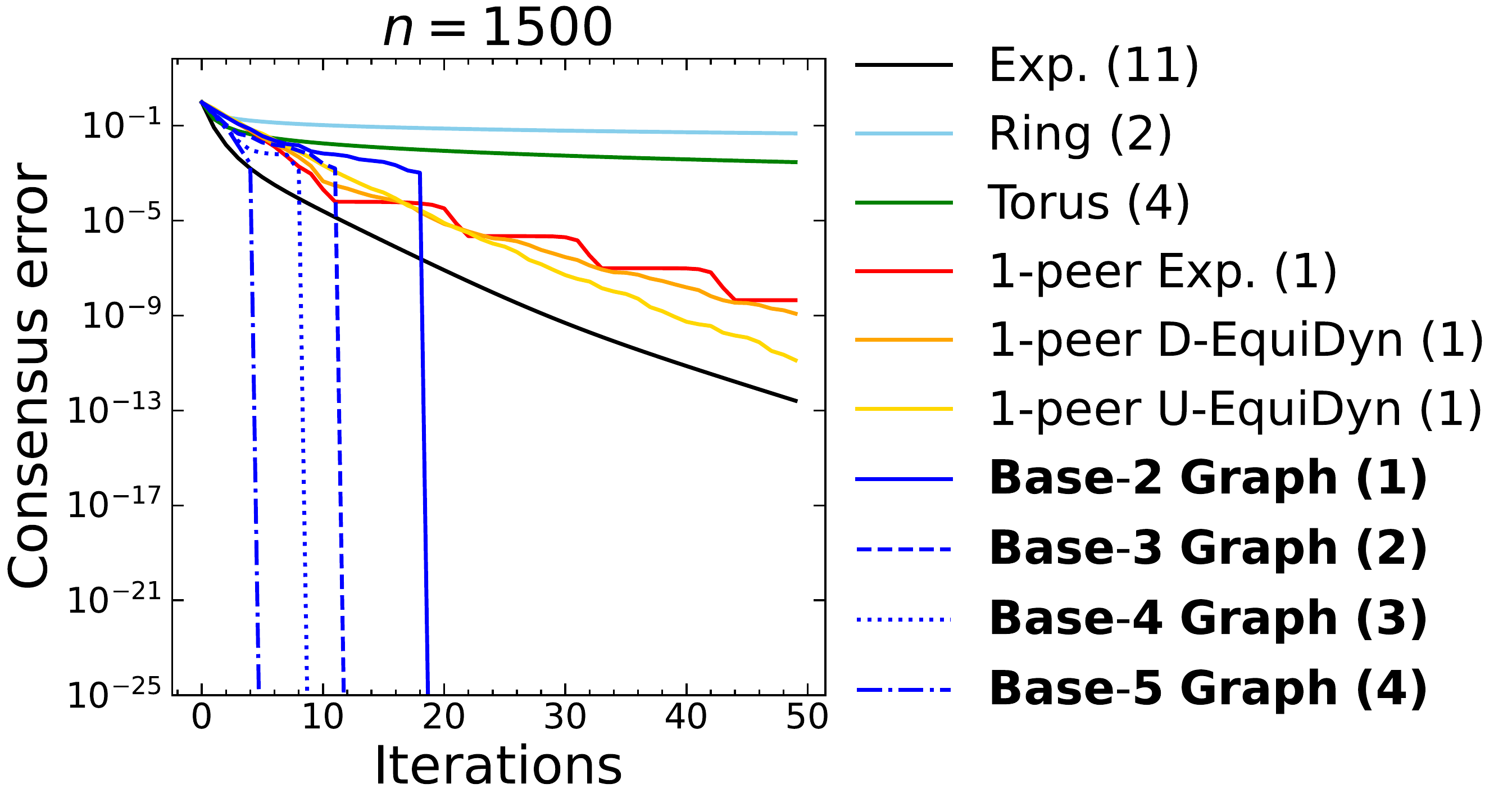}
   \vskip - 0.05 in
 \caption{Comparison of consensus rate. See Sec.~\ref{sec:experiment} for detailed experimental settings. The number in the bracket is the maximum degree.}
\label{fig:consensus_5000}
\vskip - 0.2 in
\end{wrapfigure}
However, this favorable property only holds when $n$ is a power of $2$.
When $n$ is not a power of $2$, 
the $1$-peer hypercube graph cannot be constructed,
and the $1$-peer exponential graph only reaches the consensus asymptotically as well as the ring and exponential graph, as Fig.~\ref{fig:consensus_5000} illustrates.
Thus, the $1$-peer exponential and $1$-peer hypercube graphs cannot enable DSGD to converge as fast as the exponential graph when $n$ is not a power of $2$.
Moreover, even if $n$ is a power of $2$, the $1$-peer hypercube and $1$-peer exponential graphs still cannot enable DSGD to converge faster than the exponential graph.

In this study, we ask the following question: \textit{Can we construct topologies that provide DSGD with both a faster convergence rate and better communication efficiency than the exponential graph for any number of nodes?} 
Our work provides the affirmative answer by proposing the \proposed{$(k+1)$},\footnote{Note that the maximum degree of the \proposed{$(k+1)$} is not $k+1$, but at most $k$.}
which is finite-time convergence for any number of nodes $n$ and maximum degree $k$ (see Fig.~\ref{fig:consensus_5000}).
Thanks to this favorable property,
the \proposed{$2$} enables DSGD to converge faster than the ring and torus and as fast as the exponential graph for any $n$,
while the \proposed{$2$} is more communication-efficient than the ring, torus, and exponential graph because its maximum degree is only one.
Furthermore, when $2 \leq k < \ceil{\log_2 (n)}$, 
the \proposed{$(k+1)$} enables DSGD to converge faster with fewer communication costs than the exponential graph
because the maximum degree of the \proposed{$(k+1)$} is still less than that of the exponential graph.
Experimentally, we compared the \proposed{$(k+1)$} with various existing topologies,
demonstrating that the \proposed{$(k+1)$} enables various decentralized learning methods to more successfully reconcile accuracy and communication efficiency than the existing topologies.

\begin{table}[t!]
\vskip - 0.3 in
\centering
\caption{Comparison among different topologies with $n$ nodes. The definition of the consensus rate and finite-time convergence is shown in Sec.~\ref{sec:preliminary}.}
\label{table:comparison_of_graphs}
\resizebox{\linewidth}{!}{
\begin{tabular}{lcccc}
\toprule
\textbf{Topology} & \textbf{Consensus Rate} & \textbf{Connection} & \textbf{Maximum Degree} & \textbf{\#Nodes $n$} \\
\midrule
Ring \cite{nedic2018network} &  $1 - O(n^{-2})$                 & Undirected &     $2$                             & $\forall n \in \mathbb{N}$ \\
Torus \cite{nedic2018network} &  $1 - O(n^{-1})$                & Undirected &     $4$                             & $\forall n \in \mathbb{N}$ \\
Exp. \cite{ying2021exponential} & $1 - O((\log_2 (n))^{-1})$    & Directed & $\ceil{\log_2 (n)}$          & $\forall n \in \mathbb{N}$ \\
1-peer Exp. \cite{ying2021exponential}  & $O(\log_2 (n))$-finite time conv.  & Directed   & $1$                             & A power of 2 \\
1-peer Hypercube \cite{shi2016finite} & $O(\log_2 (n))$-finite time conv.    & Undirected & $1$ & A power of 2 \\
\textbf{Base-$(k+1)$ Graph (ours)} & $O(\log_{k+1} (n))$-finite time conv.   & Undirected & $k$ & $\forall n \in \mathbb{N}$ \\
\bottomrule
\end{tabular}}
\vskip - 0.25 in
\end{table}

\section{Related Work}
\textbf{Decentralized Learning.}
The most widely used decentralized learning methods are DSGD \cite{lian2017can} and its adaptations \cite{assran2019stochastic,chen2021accelerating,lian2018asynchronous}.
Many researchers have improved DSGD and proposed DSGD with momentum \cite{gao2020periodic,lin2021quasi,yu2019on,yuan202decentlam}, communication compression methods \cite{horvath2021better,koloskova2020decentralized,lu2020moniqua,tang2018communication,vogels2020practical}, etc.
While DSGD is a simple and efficient method, DSGD is sensitive to data heterogeneity \cite{tang2018d2}.
To mitigate this issue, various methods have been proposed, which can eliminate the effect of data heterogeneity on the convergence rate, including gradient tracking \cite{lorenzo2016next,nedic2017achieving,pu2018distributed,takezawa2022momentum,xin2020distributed,zhao2022beer}, 
$D^2$ \cite{tang2018d2}, etc. \cite{li2019decentralized,vogels2021relaysum,yuan2019exact}.

\textbf{Effect of Topologies.}
Many prior studies indicated that topologies with a fast consensus rate improve the accuracy of decentralized learning \cite{koloskova2020unified,kong2021consensus,marfoq2020throughput,wang2019matcha}.
For instance, DSGD and gradient tracking converge faster as the topology has a faster consensus rate \cite{lian2017can,takezawa2022momentum}.
\citet{zhu2022topology} revealed that topology with a fast consensus rate improves the generalization bound of DSGD.
Especially when the data distributions are statistically heterogeneous, 
the topology with a fast consensus rate prevents the parameters of each node from drifting away and can improve accuracy \cite{koloskova2020unified,takezawa2022momentum}.
However, communication costs increase as the maximum degree increases \cite{song2022communication,ying2021exponential}.
Thus, developing a topology with a fast consensus rate and small maximum degree is important for decentralized learning.

\section{Preliminary and Notation}
\label{sec:preliminary}

\textbf{Notation.}
A graph $G$ is represented by $(V, E)$ where $V$ is a set of nodes and $E$ is a set of edges.
If $G$ is a graph, $V(G)$ (resp. $E(G)$) denotes the set of nodes (resp. edges) of $G$.
For any $a_1, \cdots, a_n$, $(a_1, \cdots, a_n)$ denotes the ordered set.
An empty (ordered) set is denoted by $\emptyset$.
For any $n \in \mathbb{N}$, let $[n] \coloneqq \{1, \cdots, n\}$.
For any $n, a \in \mathbb{N}$,
$\text{mod}(a, n)$ is the remainder of dividing $a$ by $n$.
$\| \cdot \|_F$ denotes Frobenius norm,
and $\mathbf{1}_n$ denotes an $n$-dimensional vector with all ones.

\textbf{Topology.}
Let $G$ be an underlying network topology with $n$ nodes,
and $\mW \in [0,1]^{n\times n}$ be a mixing matrix associated with $G$.
That is, $W_{ij}$ is the weight of the edge $(i,j)$, and $W_{ij} > 0$ if and only if $(i,j) \in E(G)$.
Most of the decentralized learning methods require $\mW$ to be doubly stochastic (i.e., $\mW \mathbf{1}_n = \mathbf{1}_n$ and $\mW^\top \mathbf{1}_n = \mathbf{1}_n$) \cite{lian2017can,lin2021quasi,nedic2017achieving,tang2018d2}.
Then, the consensus rate of $G$ is defined below.
\begin{definition}
\label{definition:consensus_rate}
Let $\mW$ be a mixing matrix associated with a graph $G$ with $n$ nodes. 
Let $\vx_i \in \mathbb{R}^d$ be a parameter that node $i$ has.
Let $\mX \coloneqq (\vx_1, \cdots, \vx_n) \in \mathbb{R}^{d \times n}$ and $\bar{\mX} \coloneqq \tfrac{1}{n} \mX \mathbf{1}_n \mathbf{1}^\top_n$.
The consensus rate $\beta \in [0, 1)$ is the smallest value that satisfies $\left\| \mX \mW - \bar{\mX} \right\|^2_F \leq \beta^2 \left\| \mX - \bar{\mX} \right\|^2_F$ for any $\mX$.
\end{definition}

Thanks to $\beta \in [0,1)$, $\vx_i$ asymptotically converge to consensus $\tfrac{1}{n} \sum_{i=1}^n \vx_i$ 
by repeating parameter exchanges with neighbors.
However, this does not mean that all nodes reach the exact consensus within a finite number of iterations except when $\beta=0$, that is, when $G$ is fully connected. 
Then, utilizing time-varying topologies, \citet{ying2021exponential} and \citet{shi2016finite} aimed to obtain sequences of graphs that can make all nodes reach the exact consensus within finite iterations and proposed the $1$-peer exponential and $1$-peer hypercube graphs respectively
(see Sec.~\ref{sec:1_peer_exp_and_1_peer_hypercube} for illustrations).

\begin{definition}
\label{def:finite_time_convergence}
Let $(G^{(1)}, \cdots, G^{(m)})$ be a sequence of graphs with the same set of nodes (i.e., $V(G^{(1)})=\cdots=V(G^{(m)})$).
Let $n$ be the number of nodes. 
Let $\mW^{(1)}, \cdots, \mW^{(m)}$ be mixing matrices associated with $G^{(1)}, \cdots, G^{(m)}$, respectively.
Suppose that $\mW^{(1)}, \cdots, \mW^{(m)}$ satisfy $\mX \mW^{(1)} \mW^{(2)} \cdots \mW^{(m)} = \bar{\mX}$ for any $\mX \in \mathbb{R}^{d \times n}$, where $\bar{\mX} = \tfrac{1}{n} \mX \mathbf{1}_n \mathbf{1}_n^\top$.
Then, $(G^{(1)}, \cdots, G^{(m)})$ is called $m$-finite time convergence or an $m$-finite time convergent sequence of graphs.
\end{definition}
Because Definition \ref{def:finite_time_convergence} assumes that $V(G^{(1)}) = \cdots = V(G^{(m)})$ holds,
we often write a sequence of graphs $(G^{(1)}, \cdots, G^{(m)})$ as $(E(G^{(1)}), \cdots, E(G^{(m)}))$
using a slight abuse of notation.
Additionally, in the following section, we often abbreviate the weights of self-loops because they are uniquely determined due to the condition that the mixing matrix is doubly stochastic.

\begin{figure}[b!]
    \vskip - 0.25 in
    \centering
    \begin{subfigure}{0.375\hsize}
        \includegraphics[width=\hsize]{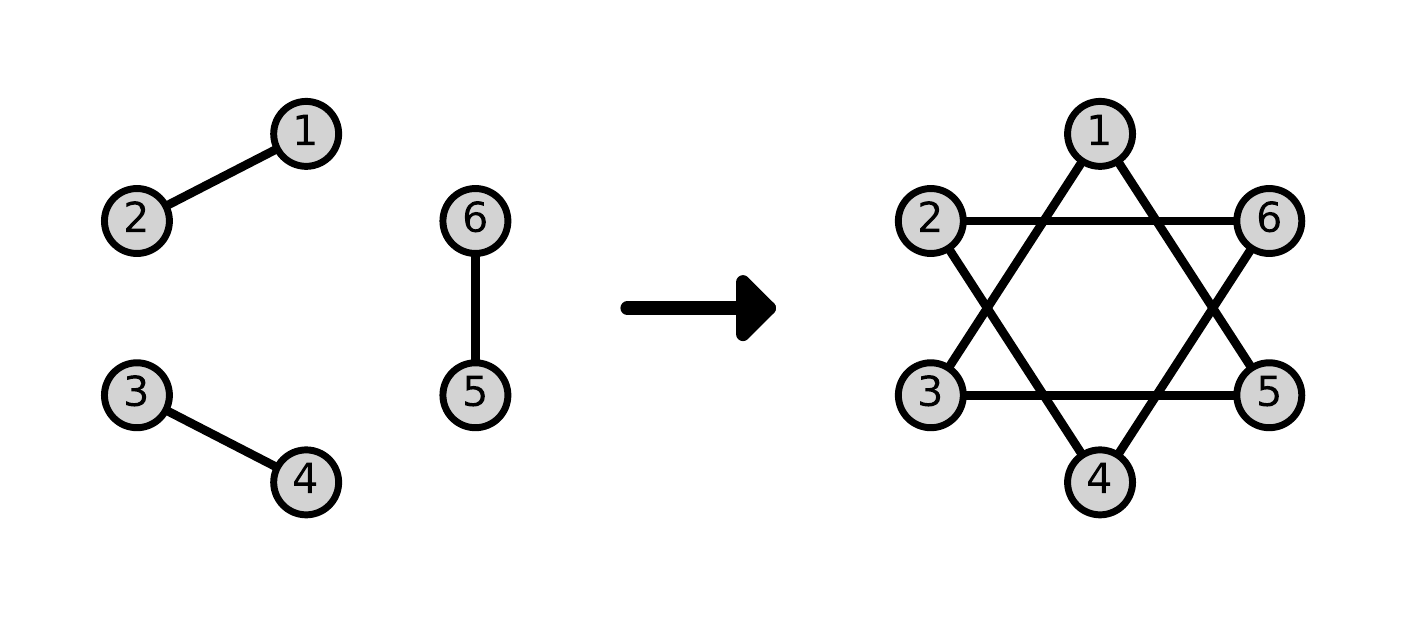}
        \vskip - 0.15 in
        \caption{$n=6 (= 2\times 3)$}
        \label{fig:2_peer_hyperhypercube_6}
    \end{subfigure}
    \hfill
    \begin{subfigure}{0.575\hsize}
        \includegraphics[width=\hsize]{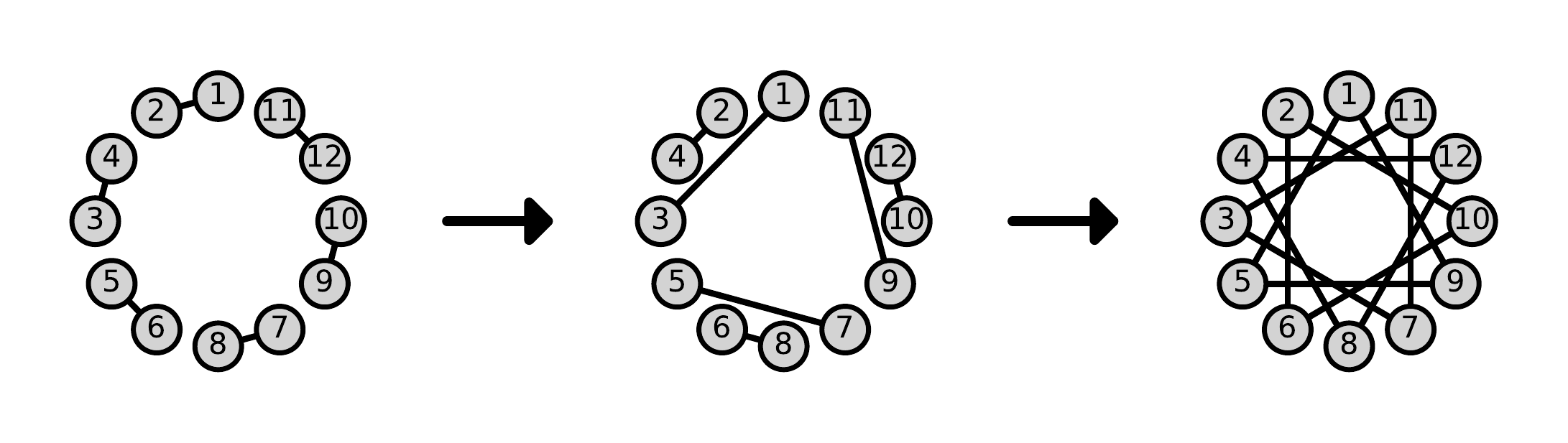}
        \vskip - 0.15 in
        \caption{$n=12 (= 2 \times 2 \times 3)$}
        \label{fig:2_peer_hyperhypercube_12}
    \end{subfigure}
    \vskip - 0.05 in
    \caption{Illustration of the \hyperhypercube{$2$}.}
    \label{fig:hyperhypercube}
\vskip - 0.2 in
\end{figure}

\section{Construction of Finite-time Convergent Sequence of Graphs}
In this section, 
we propose the \proposed{$(k+1)$},
which is finite-time convergence for any number of nodes $n \in \mathbb{N}$ and maximum degree $k \in [n-1]$.
Specifically, we consider the setting where node $i$ has a parameter $\vx_i$ and propose a graph sequence whose maximum degree is at most $k$ that makes all nodes reach the exact consensus $\tfrac{1}{n} \sum_{i=1}^n \vx_i$.
To this end, we first propose the \hyperhypercube{$k$},
which is finite-time convergence when $n$ does not have prime factors larger than $k+1$.
Using it, 
we propose the \simpleProposed{$(k+1)$} and \proposed{$(k+1)$},
which are finite-time convergence for any $n$.

\subsection{$k$-peer Hyper-hypercube Graph}
\label{sec:hyperhypercube}
\begin{algorithm}[t!]
\caption{\textsc{$k$-peer Hyper-hypercube Graph} $\mathcal{H}_k (V)$}
\label{alg:hyperhypercube}
\begin{algorithmic}[1]
\State \textbf{Input:} the set of nodes $V \coloneqq \{v_1, \cdots, v_n\}$ and number of nodes $n$.
\State Decompose $n$ as $n = n_1 \times \cdots \times n_L$ with minimum $L$ such that $n_l \in [k+1]$ for all $l \in [L]$.
\For{$l \in [L]$}
\State Initialize $b_i$ to zero for all $i \in [n]$ and $E^{(l)}$ to $\emptyset$.
\For{$i \in [n]$}
\For{$ m \in [n_{l}]$}
\State $j \leftarrow \text{mod}(i + m \times \prod_{l^\prime=1}^{l-1} n_{l^\prime}-1, n) + 1$.
\If{$b_i < n_{l} - 1$ and $b_j < n_{l} - 1$}
\State Add edge $(v_i, v_j)$ with weight $\tfrac{1}{n_l}$ to $E^{(l)}$ and $b_i \leftarrow b_i + 1$.
\EndIf
\EndFor
\EndFor
\EndFor
\State \Return $(E^{(1)}, E^{(2)}, \cdots, E^{(L)})$.
\end{algorithmic}
\end{algorithm}
\setlength{\textfloatsep}{0pt}
Before proposing the \proposed{$(k+1)$}, we first extend the $1$-peer hypercube graph \cite{shi2016finite} to the $k$-peer setting
and propose the \hyperhypercube{$k$},
which is finite-time convergence when the number of nodes $n$ does not have prime factors larger than $k+1$
and is used as a component in the \proposed{$(k+1)$}.
Let $V$ be a set of $n$ nodes.
We assume that all prime factors of $n$ are less than or equal to $k+1$.
That is, there exists $n_1, n_2, \cdots, n_L \in [k+1]$ such that $n = n_1 \times \cdots \times n_L$.
In this case, we can construct the $L$-finite time convergent sequence of graphs whose maximum degree is at most $k$.
Using Fig.~\ref{fig:2_peer_hyperhypercube_6},
we explain how all nodes reach the exact consensus.
Let $G^{(1)}$ and $G^{(2)}$ denote the graphs in Fig.~\ref{fig:2_peer_hyperhypercube_6} from left to right, respectively.
After the nodes exchange parameters with neighbors in $G^{(1)}$,
nodes $1$ and $2$, nodes $3$ and $4$, and nodes $5$ and $6$ have the same parameter respectively.
Then, after exchanging parameters in $G^{(2)}$,
all nodes reach the exact consensus.
We present the complete algorithms for constructing the \hyperhypercube{$k$} in Alg.~\ref{alg:hyperhypercube}.

\subsection{Simple Base-$(k+1)$ Graph}
\label{sec:k_peer_simple_FAIRY}
As described in Sec.~\ref{sec:hyperhypercube}, when $n$ do not have prime factors larger than $k+1$, we can easily make all nodes reach the exact consensus by the \hyperhypercube{$k$}.
However, when $n$ has prime factors larger than $k+1$, e.g., when $(k,n)=(1,5)$, 
the \hyperhypercube{$k$} cannot be constructed. 
In this section, we extend the \hyperhypercube{$k$} and propose the \simpleProposed{$(k+1)$},
which is finite-time convergence for any number of nodes $n$ and maximum degree $k$. 
Note that the maximum degree of the \simpleProposed{$(k+1)$} is not $k+1$, but at most $k$.

We present the pseudo-code for constructing the \simpleProposed{$(k+1)$} in Alg.~\ref{alg:simple_k_peer_FAIRY}.
For simplicity, here we explain only the case when the maximum degree $k$ is one.
The case with $k \geq 2$ is explained in Sec.~\ref{sec:simple_base_with_large_k}.
The \simpleProposed{$(k+1)$} mainly consists of the following four steps.
The key idea is that splitting $V$ into disjoint subsets to which the \hyperhypercube{$k$} is applicable.
\begin{description}
    \item[Step 1.] As in the base-$2$ number of $n$, we decompose $n$ as $n = 2^{p_1} + \cdots + 2^{p_L}$ in line 1, and then split $V$ into disjoint subsets $V_1, \cdots, V_L$ such that $|V_l| = 2^{p_l}$ for all $l \in [L]$ in line $3$.\footnote{Splitting $V_l$ into $V_{l,1}, \cdots, V_{l,a_l}$ becomes crucial when $k \geq 2$ (see Sec.~\ref{sec:simple_base_with_large_k}).}
    \item[Step 2.] For all $l \in [L]$, we make all nodes in $V_l$ obtain the average of parameters in $V_l$ using the \hyperhypercube{$1$} $\mathcal{H}_1 (V_l)$ in line 11. Then, we initialize $l^\prime$ as one.
    \item[Step 3.] Each node in $V_{l^\prime + 1} \cup \cdots \cup V_L$ exchanges parameters with one node in $V_{l^\prime}$ such that the average in $V_{l^\prime}$ becomes equivalent to the average in $V$. We increase $l^\prime$ by one and repeat step $3$ until $l^\prime = L$. This procedure corresponds to line 15.
    \item[Step 4.] For all $l \in [L]$, we make all nodes in $V_l$ obtain the average in $V_l$ using the \hyperhypercube{$1$} $\mathcal{H}_1 (V_l)$. Because the average in $V_l$ is equivalent to the average in $V$ after step $3$, all nodes can reach the exact consensus. This procedure corresponds to line 25.
\end{description}

\begin{figure}[b!]
    \centering
    \includegraphics[width=1.0\hsize]{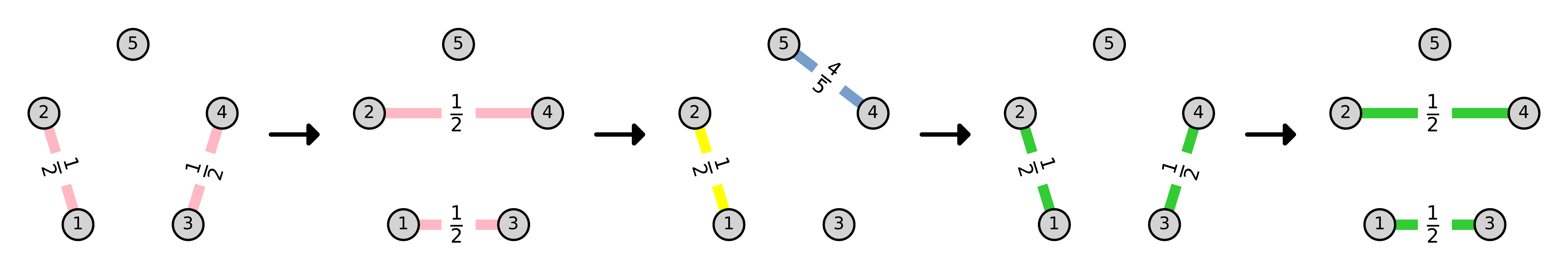}
    \vskip - 0.1 in
    \caption{\SimpleProposed{$2$} with $n=5(=2^2 + 1)$. The value on the edge is the edge weight, and the edges are colored in the same color as the line in Alg.~\ref{alg:simple_k_peer_FAIRY} where they were added.}
    \label{fig:1_peer_simple_FAIRY_5}
    \vskip - 0.2 in
\end{figure}

\begin{algorithm}[b!]
\caption{\textsc{Simple Base-$(k+1)$ Graph} $\mathcal{A}_k^{\text{simple}} (V)$}
\label{alg:simple_k_peer_FAIRY}
\begin{algorithmic}[1]
\State \textbf{Input:} the set of nodes $V$ and number of nodes $n (= a_1 (k+1)^{p_1} + a_2 (k+1)^{p_2} + \cdots + a_L (k+1)^{p_L})$ such that $p_1 > p_2 > \cdots > p_L \geq 0$ and $a_l \in [k]$ for all $l \in [L]$.
\State \textbf{If} all prime factors of $n$ are less than or equal to $k+1$ \textbf{then} \Return $\mathcal{H}_k (V)$.
\State Split $V$ into disjoint subsets $V_1, \cdots, V_{L}$ such that $|V_{l}| = a_l (k+1)^{p_l}$ for all $l \in [L]$. Then, for all $l \in [L]$, split $V_l$ into disjoint subsets $V_{l, 1}, \cdots, V_{l,a_l}$ such that $|V_{l,a}| = (k+1)^{p_l}$ for all $a \in [a_l]$.
\State Construct \hyperhypercube{$k$} $\mathcal{H}_k (V_l)$ for all $l \in [L]$ and $m_1 = |\mathcal{H}_k (V_1)|$.
\State Construct \hyperhypercube{$k$} $\mathcal{H}_k (V_{l,a})$ for all $l \in [L]$ and $a \in [a_l]$.
\State Initialize $b_{l}$ as zero for all $l \in [L]$, and initialize $m$ as zero.
\While{$b_{1} < |\mathcal{H}_k (V_{1,1})|$}
\State $m \leftarrow m+1$ and $E^{(m)} \leftarrow \emptyset$.
\For{$l \in \{ L, L-1, \cdots, 1 \}$}
\If{$m \leq m_1$}
\State \fcolorbox{cherryblossompink}{cherryblossompink}{Add $E(\mathcal{H}_k(V_l)^{(m^\prime)})$ to $E^{(m)}$ where $m^\prime = \text{mod}(m-1, |\mathcal{H}_k (V_l)|) + 1$.}
\ElsIf{$m < m_1 + l$}
\For{$v \in V_{l}$}
\State Select isolated node $u_1, \cdots, u_{a_{m-m_1}}$ from $V_{m - m_1, 1}, \cdots, V_{m - m_1, a_{m - m_1}}$.
\State \fcolorbox{darkpastelblue}{darkpastelblue}{Add edges $(v,u_1), \cdots, (v, u_{a_{m-m_1}})$ with weight $\tfrac{|V_{m - m_1}|}{a_{m-m_1} \sum_{l^\prime = m - m_1}^{L} |V_{l^\prime}|}$ to $E^{(m)}$.}
\EndFor
\ElsIf{$m = m_1 + l$ and $l \not = L$}
\While{There are two or more isolated nodes in $V_l$}
\State $c \leftarrow \text{the number of isolated nodes in $V_l$}$.
\State Select $\min\{ k+1, c\}$ isolated nodes $V^\prime$ in $V_l$.
\State \fcolorbox{yellow}{yellow}{Add edges with weights $\tfrac{1}{|V^\prime|}$ to $E^{(m)}$ such that $V^\prime$ compose the complete graph.}
\EndWhile
\Else 
\State $b_{l} \leftarrow b_{l} + 1$.
\If{$p_l \not = 0$}
\For{$a \in [a_l]$}
\State \fcolorbox{limegreen}{limegreen}{Add $E(\mathcal{H}_k (V_{l,a})^{(m^\prime)})$ to $E^{(m)}$ where $m^\prime = \text{mod}(b_{l}-1, |\mathcal{H}_k (V_{l,a})|) + 1$.}
\EndFor
\Else
\State \fcolorbox{columbiablue}{columbiablue}{Add $E(\mathcal{H}_k (V_{l})^{(m^\prime)})$ to $E^{(m)}$ where $m^\prime = \text{mod}(b_{l}-1, |\mathcal{H}_k (V_{l})|) + 1$.}
\EndIf
\EndIf
\EndFor
\EndWhile
\State \Return $(E^{(1)}, E^{(2)}, \cdots, E^{(m)})$.
\end{algorithmic}
\end{algorithm}
\setlength{\textfloatsep}{0pt}

Using the example presented in Fig.~\ref{fig:1_peer_simple_FAIRY_5},
we explain Alg.~\ref{alg:simple_k_peer_FAIRY} in more detail.
Let $G^{(1)}, \cdots, G^{(5)}$ denote the graphs in Fig.~\ref{fig:1_peer_simple_FAIRY_5} from left to right, respectively.
In step $1$, we split $V \coloneqq \{1, \cdots, 5\}$ into $V_1 \coloneqq \{1, \cdots, 4\}$ and $V_2 \coloneqq \{ 5 \}$.
In step $2$, nodes in $V_1$ obtain the average in $V_1$
by exchanging parameters in $G^{(1)}$ and $G^{(2)}$.
In step $3$, 
the average in $V_1$ becomes equivalent to the average in $V$ by exchanging parameters in $G^{(3)}$.
In step $4$, nodes in $V_1$ can get the average in $V$ by exchanging parameters in $G^{(4)}$ and $G^{(5)}$.
Because node $5$ also obtains the average in $V$ after exchanging parameters in $G^{(3)}$,
all nodes reach the exact consensus after exchanging parameters in $G^{(5)}$.

Note that edges added in lines 20 and 27 are not necessary if we only need to make all nodes reach the exact consensus. 
Nonetheless, these edges are effective in keeping the parameters of nodes close in value to each other in decentralized learning
because the parameters are updated by gradient descent before the parameter exchange with neighbors.
For instance, edge $(1,2)$ in $G^{(3)}$, which is added in line 20, 
is not necessary for finite-time convergence
because nodes $1$ and $2$ already have the same parameter after exchanging parameters in $G^{(1)}$ and $G^{(2)}$.
We provide more examples in Sec.~\ref{sec:illustration}.

\subsection{Base-$(k+1)$ Graph}
\label{sec:k_peer_FAIRY}
\begin{algorithm}[t!]
\caption{\proposed{$(k+1)$} $\mathcal{A}_k (V)$}
\label{alg:k_peer_FAIRY}
\begin{algorithmic}[1]
\State \textbf{Input:} the set of nodes $V$ and number of nodes $n$.
\State Decompose $n$ as $n = p \times q$ such that $p$ is a multiple of $2,3,\cdots,(k+1)$ and $q$ is prime to $2,3,\cdots,(k+1)$.
\State Split $V$ into disjoint subsets $V_1, \cdots, V_{p}$ such that $|V_l| = q$ for all $l$.
\State Construct \simpleProposed{$(k+1)$} $\mathcal{A}_k^{\text{simple}} (V_l)$ for all $l \in [p]$. 
\For{$m \in \{ 1, 2, \cdots, |\mathcal{A}_k^\text{simple} (V_1)|\}$}
\State $E^{(m)} \leftarrow \bigcup_{l\in[p]} E(\mathcal{A}_k^\text{simple} (V_l)^{(m)})$.
\EndFor
\State Split $V$ into disjoint subsets $U_1, \cdots, U_{q}$ such that $|U_l| = p$ and $|U_l \cap V_{l^\prime}|=1$ for all $l, l^\prime$.
\State Construct \hyperhypercube{$k$} $\mathcal{H}_k (U_l)$ for all $l \in [q]$.
\For{$m \in \{1, 2, \cdots, |\mathcal{H}_k(U_1)| \}$}
\State $E^{(m + |\mathcal{A}_k^\text{simple} (V_1)|)} \leftarrow \bigcup_{l\in [q]} E(\mathcal{H}_k (U_l)^{(m)})$.
\EndFor
\State $\mathcal{E} \leftarrow (E^{(1)}, E^{(2)}, \cdots, E^{(|\mathcal{A}_k^\text{simple} (V_1)| + |\mathcal{H}_k (U_1) |)})$.
\State \textbf{If} $|\mathcal{A}_k^{\text{simple}}(V)| < |\mathcal{E}|$ \textbf{then} \Return $\mathcal{A}_k^{\text{simple}}(V)$ \textbf{else} \Return $\mathcal{E}$.
\end{algorithmic}
\end{algorithm}
\begin{figure}[t!]
    \vskip - 0.15 in
    \centering
    \begin{subfigure}{0.8\hsize}
        \includegraphics[width=\hsize]{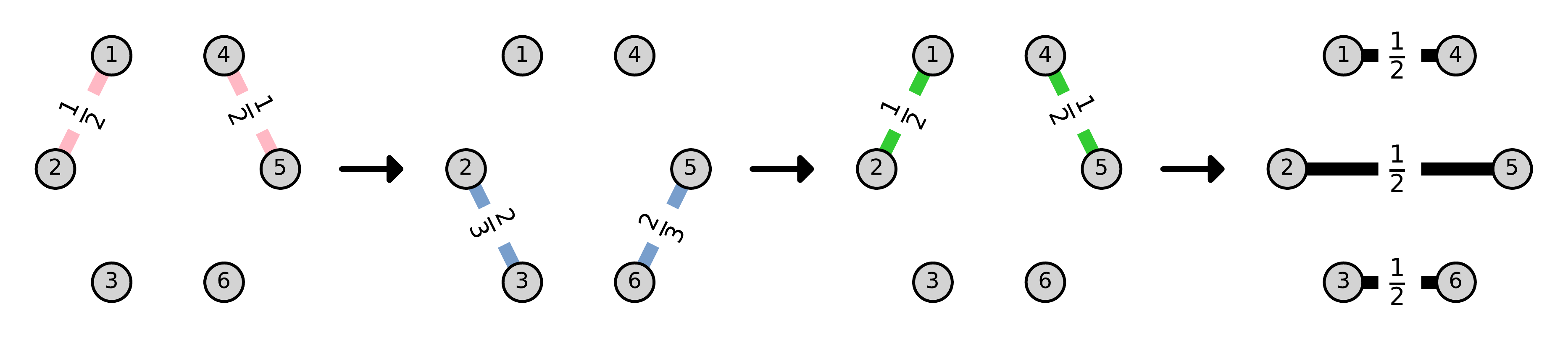}
        \vskip - 0.1 in
        \caption{\Proposed{$2$} with $n=6 (=2 \times (2 + 1))$}
        \label{fig:1_peer_FAIRY_6}
    \end{subfigure}
    \begin{subfigure}{\hsize}
        \includegraphics[width=\hsize]{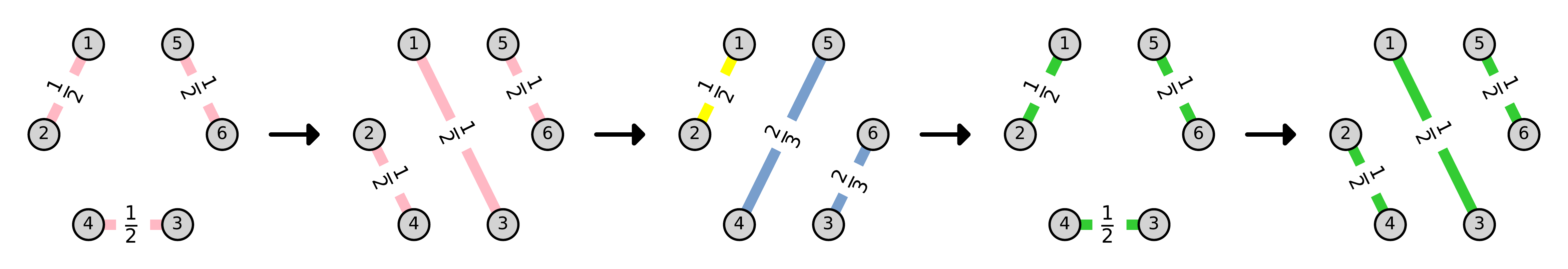}
        \vskip - 0.1 in
        \caption{\SimpleProposed{$2$} with $n=6 (=2^2 + 2)$}
        \label{fig:1_peer_SimpleADIC_6}
    \end{subfigure}
    \caption{Comparison of \SimpleProposed{$2$} and \Proposed{$2$} with $n = 6$. The value on the edge indicates the edge weight. The edges added in line 10 in Alg.~\ref{alg:k_peer_FAIRY} are colored black, and the edges added in line 6 are colored the same color as the line in Alg.~\ref{alg:simple_k_peer_FAIRY} where they are added.}
\end{figure}

The \simpleProposed{$(k+1)$} is finite-time convergence for any $n$ and $k$, 
while the \simpleProposed{$(k+1)$} contains graphs that are not necessary for the finite-time convergence and becomes a redundant sequence of graphs in some cases, e.g., $(k,n) = (1,6)$ (see the example in Fig.~\ref{fig:1_peer_SimpleADIC_6} and a detailed explanation in Sec.~\ref{sec:ullustrative_comparison}). 
To remove this redundancy,
this section proposes the \proposed{$(k+1)$} that can make all nodes reach the exact consensus after fewer iterations than the \simpleProposed{$(k+1)$}.

The pseudo-code for constructing the \proposed{$(k+1)$} is shown in Alg.~\ref{alg:k_peer_FAIRY}.
The \proposed{$(k+1)$} consists of the following three steps.
\begin{description}
    \item[Step 1.] We decompose $n$ as $p \times q$ such that $p$ is a multiple of $2, \cdots, (k+1)$ and $q$ is prime to $2, \cdots, (k+1)$, and split $V$ into disjoint subsets $V_1, \cdots, V_p$ such that $|V_l|=q$ for all $l \in [p]$.
    \item[Step 2.] For all $l \in [p]$, we make all nodes in $V_l$ reach the average in $V_l$ by the \simpleProposed{$(k+1)$} $\mathcal{A}^{\text{simple}}_k (V_l)$. Then, we take $p$ nodes from $V_1, \cdots, V_{p}$ respectively and construct a set $U_1$. Similarly, we construct $U_2, \cdots, U_q$ such that $U_1, \cdots, U_q$ are disjoint sets.
    \item[Step 3.] For all $l \in [q]$, we make all nodes in $U_{l}$ reach the average in $U_{l}$ by the \hyperhypercube{$k$} $\mathcal{H}_k (U_l)$. Because the average in $U_{l}$ is equivalent to the average in $V$ after step $2$, all nodes reach the exact consensus.
\end{description}

Using the example in Fig.~\ref{fig:1_peer_FAIRY_6}, we explain the \proposed{$(k+1)$} in more detail.
Let $G^{(1)}, \cdots, G^{(4)}$ denote the graphs in Fig.~\ref{fig:1_peer_FAIRY_6} from left to right, respectively.
In step $1$, we split $V$ into $V_1 \coloneqq \{1, 2, 3\}$ and $V_2 \coloneqq \{4, 5, 6\}$.
In step $2$, 
nodes in $V_1$ and nodes in $V_2$ have the same parameter respectively by exchanging parameters on $G^{(1)}, \cdots, G^{(3)}$
because the subgraphs composed on $V_1$ and $V_2$ are same as the \simpleProposed{$(k+1)$} (see Fig.~\ref{fig:1_peer_simple_FAIRY_3}).
Then, we construct $U_1 \coloneqq \{1,4\}$, $U_2 \coloneqq \{2, 5\}$, and $U_3 \coloneqq \{3, 6\}$.
Finally, in step $3$, all nodes reach the exact consensus by exchanging parameters in $G^{(4)}$.

\begin{wrapfigure}{r}[0pt]{0.35\textwidth}
\vskip - 0.2 in
\centering
   \includegraphics[width=0.35\textwidth]{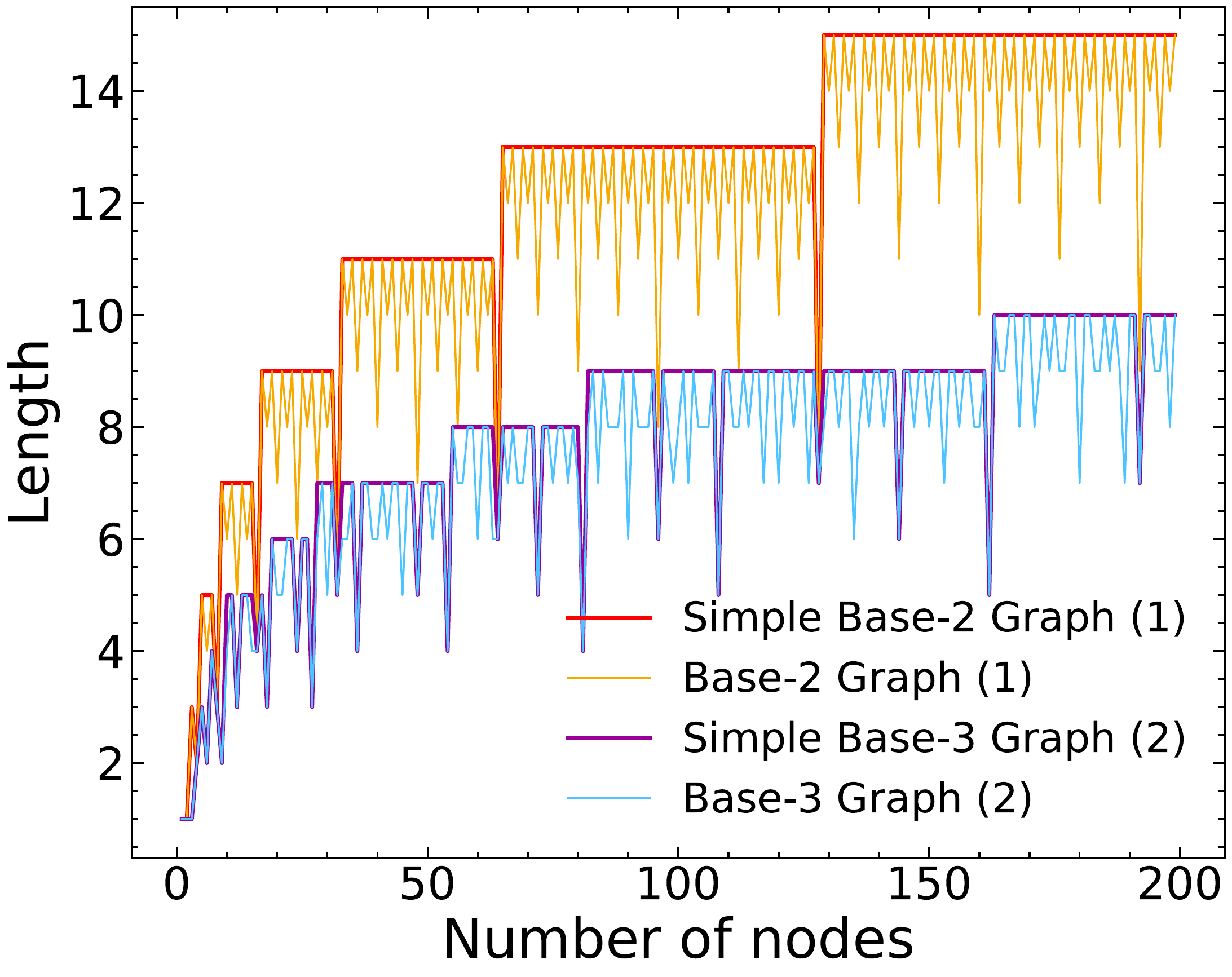}
   \vskip - 0.075 in
   \caption{Comparison of length.}
\label{fig:1_peer_length}
\vskip - 0.2 in
\end{wrapfigure}
Fig.~\ref{fig:1_peer_length} and Sec.~\ref{sec:comparison_of_FAIRY_and_simple_FAIRY} compare the \proposed{$(k+1)$} with the \simpleProposed{$(k+1)$},
demonstrating that the length of the \proposed{$(k+1)$} is less than that of the \simpleProposed{$(k+1)$} in many cases.
Moreover, Theorem \ref{theorem:length} show the upper bound of the length of the \simpleProposed{$(k+1)$} and \proposed{$(k+1)$}.
The proof is provided in Sec.~\ref{sec:proof_of_length}.

\begin{theorem}
\label{theorem:length}
For any number of nodes $n \in \mathbb{N}$ and maximum degree $k \in [n-1]$,
the length of the \simpleProposed{$(k+1)$} and \proposed{$(k+1)$} is less than or equal to $2 \log_{k+1} (n) + 2$.
\end{theorem}
\begin{corollary}
\label{cor:finite_time_convergence}
For any number of nodes $n \in \mathbb{N}$ and maximum degree $k \in [n-1]$,
the \simpleProposed{$(k+1)$} and \proposed{$(k+1)$} are $\mathcal{O}(\log_{k+1} (n))$-finite time convergence.
\end{corollary}
Therefore, the \proposed{$(k+1)$} is a powerful extension of the $1$-peer exponential \cite{ying2021exponential} and $1$-peer hypercube graphs \cite{shi2016finite}
because they are $\mathcal{O}(\log_2 (n))$-finite time convergence only if $n$ is a power of 2 and their maximum degree cannot be set to any number other than $1$.

\section{Decentralized SGD on Base-$(k+1)$ Graph}
In this section, we verify the effectiveness of the \proposed{$(k+1)$} for decentralized learning,
demonstrating that the \proposed{$(k+1)$} can endow DSGD with both a faster convergence rate and fewer communication costs than the existing topologies, including the ring, torus, and exponential graph.
We consider the following decentralized learning problem:
\begin{align*}
    \min_{\vx \in \mathbb{R}^d} \left[ f (\vx) \coloneqq \frac{1}{n} \sum_{i=1}^n f_i (\vx) \right], \;\; 
    f_i (\vx) \coloneqq \mathbb{E}_{\xi_i \sim \mathcal{D}_i} \left[ F_i (\vx ; \xi_i) \right],
\end{align*}
where $n$ is the number of nodes,
$f_i$ is the loss function of node $i$, $\mathcal{D}_i$ is the data distribution held by node $i$,
$F_i (\vx ; \xi_i)$ is the loss of node $i$ at data sample $\xi_i$,
and $\nabla F_i (\vx ; \xi_i)$ denotes the stochastic gradient.
Then, we assume that the following hold, 
which are commonly used for analyzing decentralized learning methods \cite{lian2017can,lin2021quasi,lu2020moniqua,ying2021exponential}.

\begin{assumption}
\label{assumption:lower_bound}
There exists $f^\star > - \infty$ that satisfies $f(\vx) \geq f^\star$ for any $\vx \in \mathbb{R}^d$.
\end{assumption}
\begin{assumption}
\label{assumption:smoothness}
$f_i$ is $L$-smooth for all $i\in[n]$.
\end{assumption}
\begin{assumption}
\label{assumption:stochastic}
There exists $\sigma^2$ that satisfies $\mathbb{E}_{\xi_i \sim \mathcal{D}_i} \| \nabla F_i (\vx ; \xi_i) - \nabla f_i (\vx) \|^2 \leq \sigma^2$ for all $\vx \in \mathbb{R}^d$.
\end{assumption}
\begin{assumption}
\label{assumption:heterogeneity}
There exists $\zeta^2$ that satisfies $\frac{1}{n} \sum_{i=1}^n \| \nabla f_i (\vx) - \nabla f (\vx)\|^2 \leq \zeta^2$ for all $\vx \in \mathbb{R}^d$.
\end{assumption}
We consider the case when DSGD \cite{lian2017can}, the most widely used decentralized learning method, is used as an optimization method.
Let $\mW^{(1)}, \cdots, \mW^{(m)}$ be mixing matrices of the \proposed{$(k+1)$}.
In DSGD on the \proposed{$(k+1)$},
node $i$ updates its parameter $\vx_i$ as follows:
\begin{align}
\label{eq:dsgd_on_FAIRY}
    \vx_i^{(r+1)} = \sum_{j=1}^n W_{ij}^{(1+\text{mod}(r, m))} \left( \vx_j^{(r)} - \eta \nabla F_j (\vx_j^{(r)}; \xi_j^{(r)} )\right),
\end{align}
where $\eta$ is the learning rate.
In this case, thanks to the property of finite-time convergence,
DSGD on the \proposed{$(k+1)$} converges at the following convergence rate.

\begin{theorem}
\label{theorem:convergence_rate}
Suppose that Assumptions \ref{assumption:lower_bound}-\ref{assumption:heterogeneity} hold.
Then, for any number of nodes $n \in \mathbb{N}$ and maximum degree $k \in [n-1]$, there exists $\eta$ such that $\bar{\vx} \coloneqq \frac{1}{n} \sum_{i=1}^n \vx_i$ generated by \eqref{eq:dsgd_on_FAIRY} satisfies $\frac{1}{R+1} \sum_{r=0}^{R} \mathbb{E} \left\| \nabla f (\bar{\vx}^{(r)}) \right\|^2 \leq \epsilon$ after
\begin{align}
\label{eq:convergence_rate}
    R = \mathcal{O} \left( \frac{\sigma^2}{n \epsilon^2}
    + \frac{\zeta \log_{k+1} (n) + \sigma \sqrt{\log_{k+1} (n)}}{\epsilon^{3/2}}
    + \frac{\log_{k+1} (n)}{\epsilon} \right) \cdot L F_0
\end{align}
iterations, where $F_0 \coloneqq f (\bar{\vx}^{(0)}) - f^\star$.
\end{theorem}
The above theorem follows immediately from Theorem 2 stated in \citet{koloskova2020unified} and Corollary \ref{cor:finite_time_convergence}.
The convergence rates of DSGD over commonly used topologies are summarized in Sec.~\ref{sec:convergence_rate_of_various_topologies}.
From Theorem \ref{theorem:convergence_rate} and Sec.~\ref{sec:convergence_rate_of_various_topologies}, 
we can conclude that
for any number of nodes $n$,
the \proposed{$2$} enables DSGD to converge faster than the ring and torus and as fast as the exponential graph,
although the maximum degree of the \proposed{$2$} is only one.
Moreover, if we set the maximum degree $k$ to the value between $2$ to $\ceil{\log_2 (n)}$,
the \proposed{$(k+1)$} enables DSGD to converge faster than the exponential graph,
even though the maximum degree of the \proposed{$(k+1)$} remains less than that of the exponential graph.
It is worth noting that if we increase the maximum degree of the $1$-peer exponential and $1$-peer hypercube graphs (i.e., $k$-peer exponential and $k$-peer hypercube graphs with $k \geq 2$),
these topologies cannot enable DSGD to converge faster than the exponential graph
because these topologies are no longer finite-time convergence even when the number of nodes is a power of $2$.

\section{Experiments}
\label{sec:experiment}
In this section, we validate the effectiveness of the \proposed{$(k+1)$}.
First, we experimentally verify that 
the \proposed{$(k+1)$} is finite-time convergence for any number of nodes in Sec.~\ref{sec:consensus},
and we verify the effectiveness of the \proposed{$(k+1)$} for decentralized learning in Sec.~\ref{sec:decentralized_learning}.
\subsection{Consensus Rate}
\label{sec:consensus}
\begin{figure}[b!]
    \centering
    \vskip 0.1 in
    \includegraphics[width=\hsize]{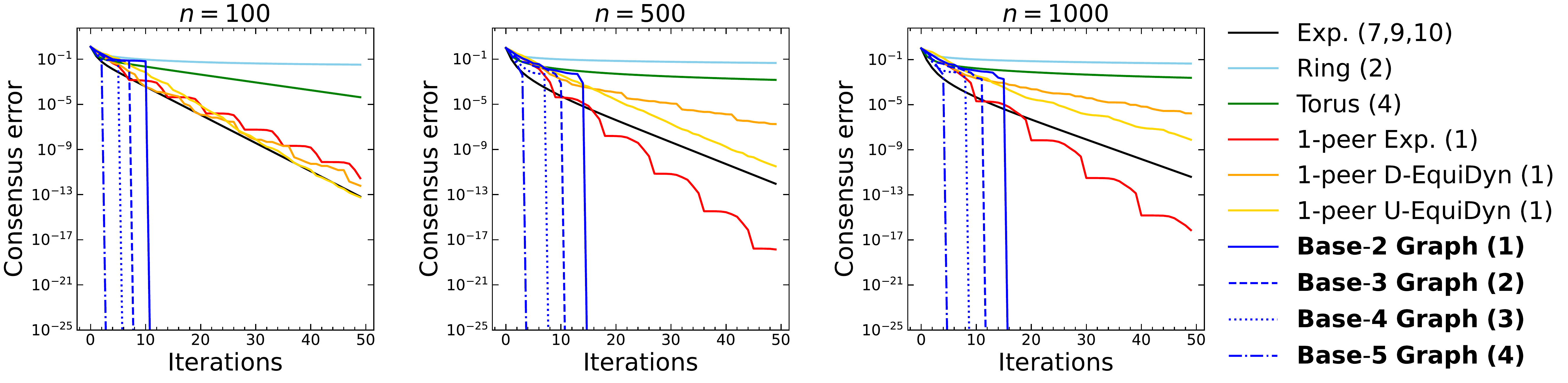}
    \vskip - 0.1 in
    \caption{Comparison of consensus rates among various topologies. The number in the bracket indicates the maximum degree of a topology. Because the maximum degree of the exponential graph depends on $n$, the three numbers in the bracket indicate the maximum degree for each $n$.}
    \label{fig:consensus}
    \vskip - 0.1 in
\end{figure}
\begin{figure}[b!]
    \centering
    \vskip - 0.1 in
    \begin{subfigure}{\hsize}
        \centering
        \includegraphics[width=0.263\hsize]{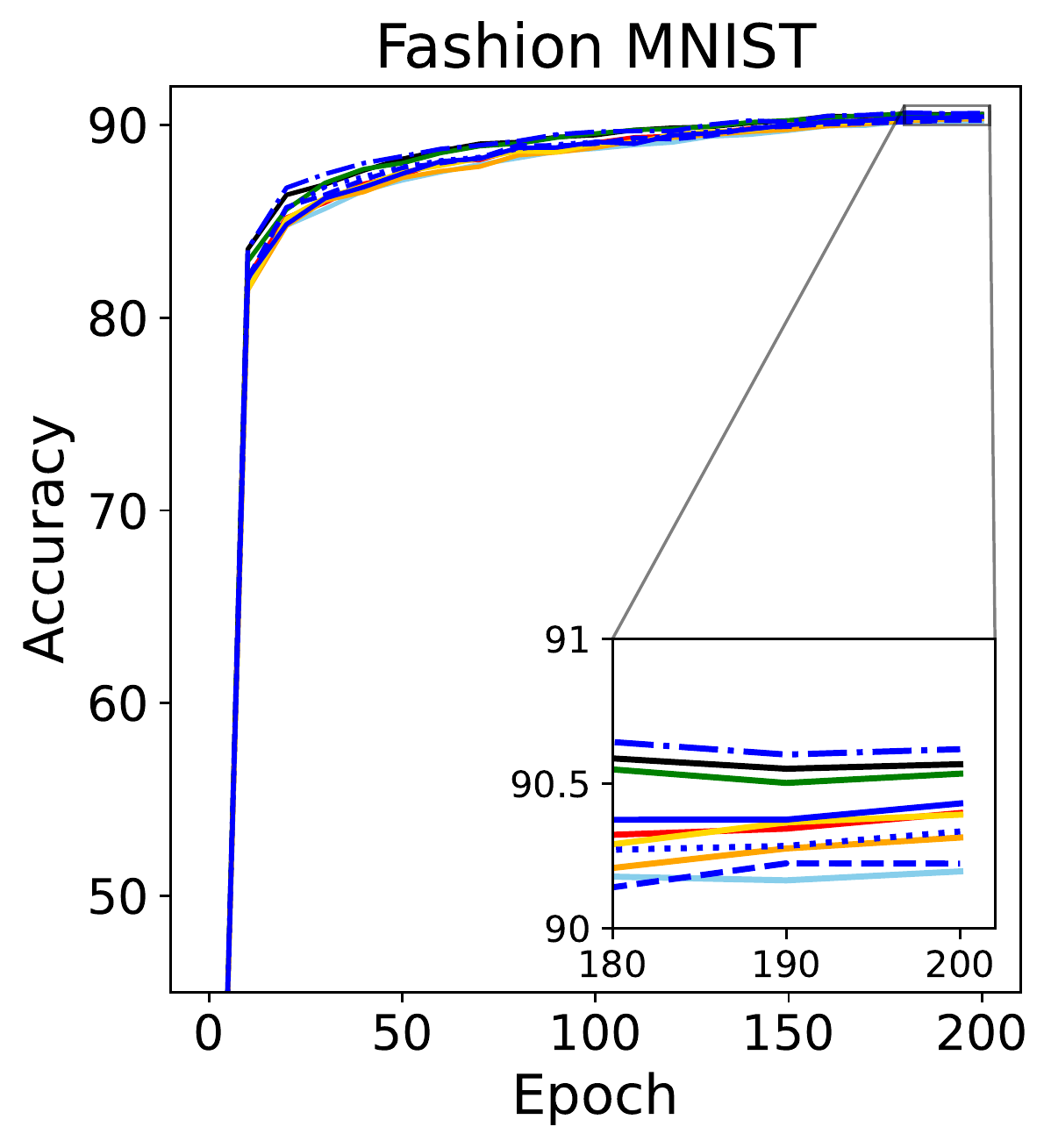}
        \includegraphics[width=0.263\hsize]{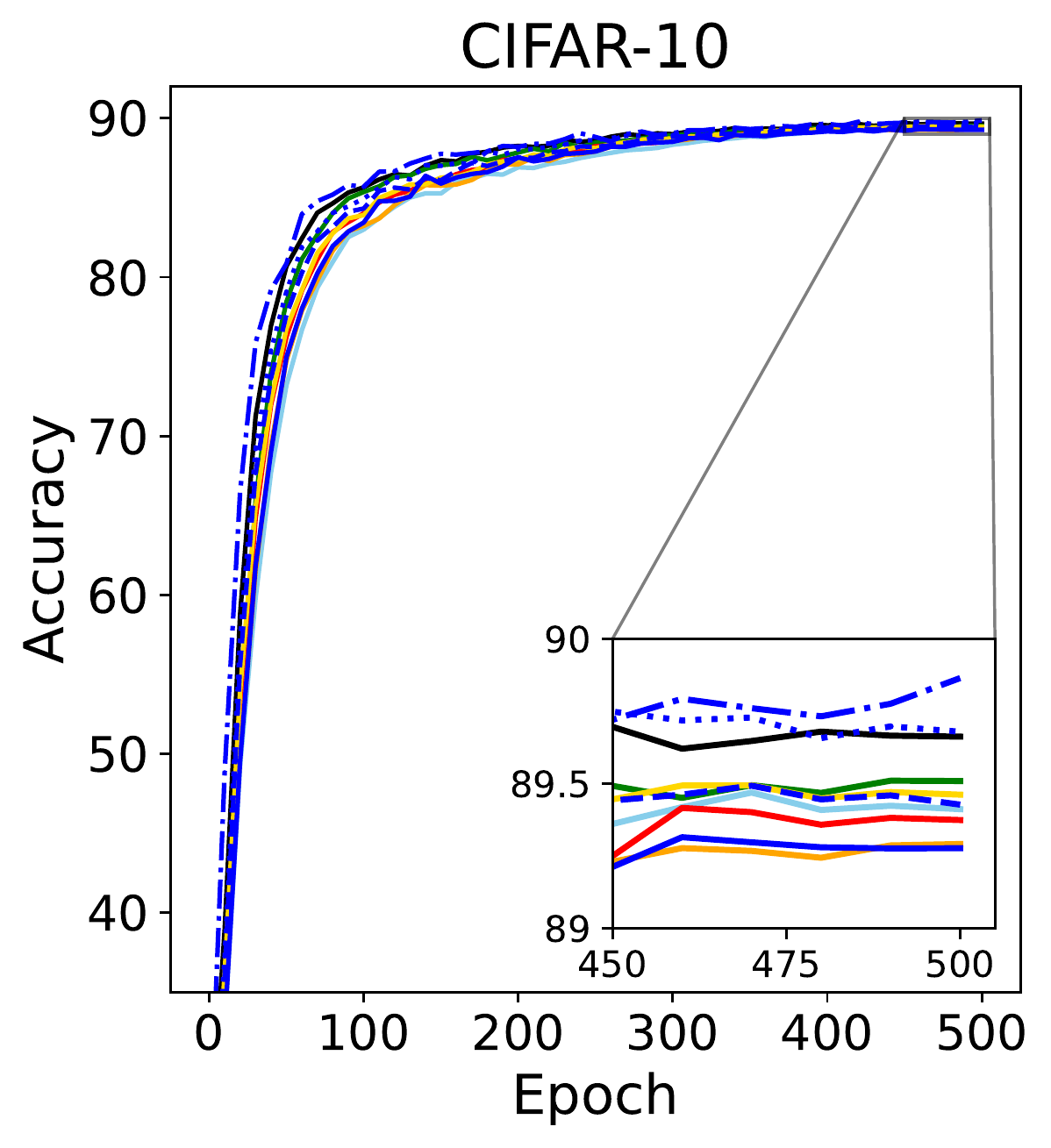}
        \includegraphics[width=0.46\hsize]{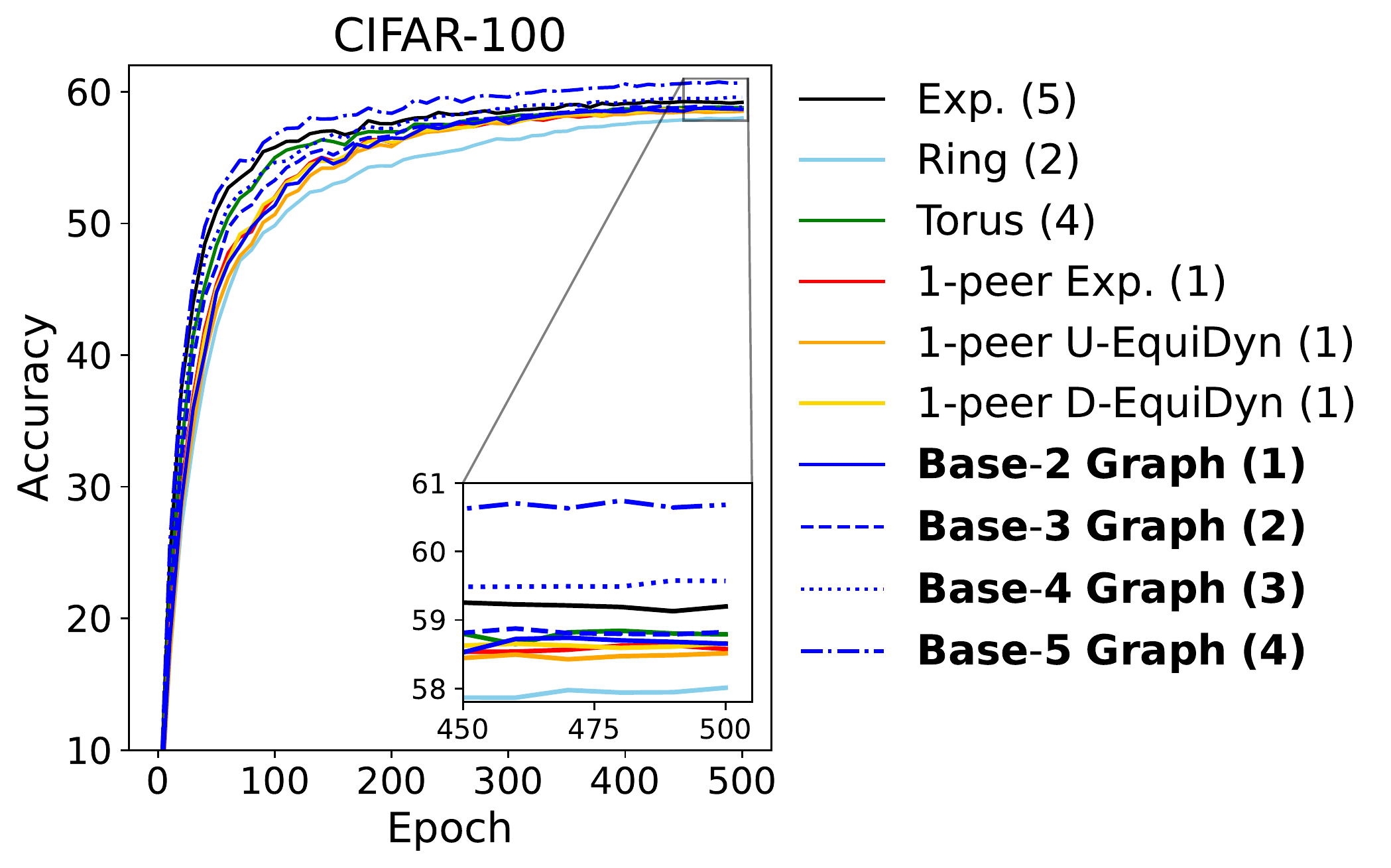}
        \vskip - 0.1 in
        \caption{$\alpha = 10$}
        \label{fig:learning_curves_iid}
    \end{subfigure}
    \begin{subfigure}{\hsize}
        \centering
            \includegraphics[width=0.263\hsize]{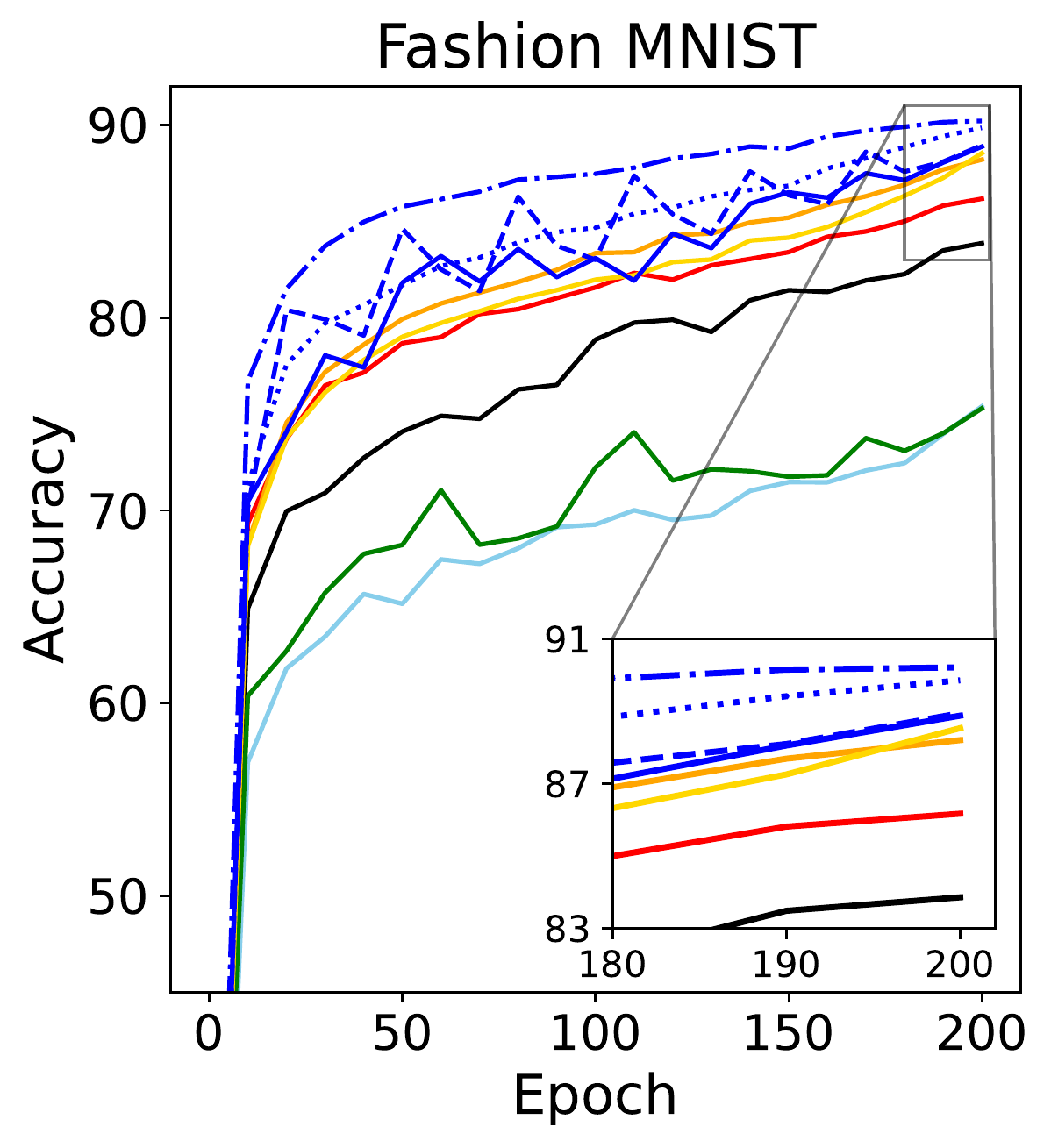}
            \includegraphics[width=0.263\hsize]{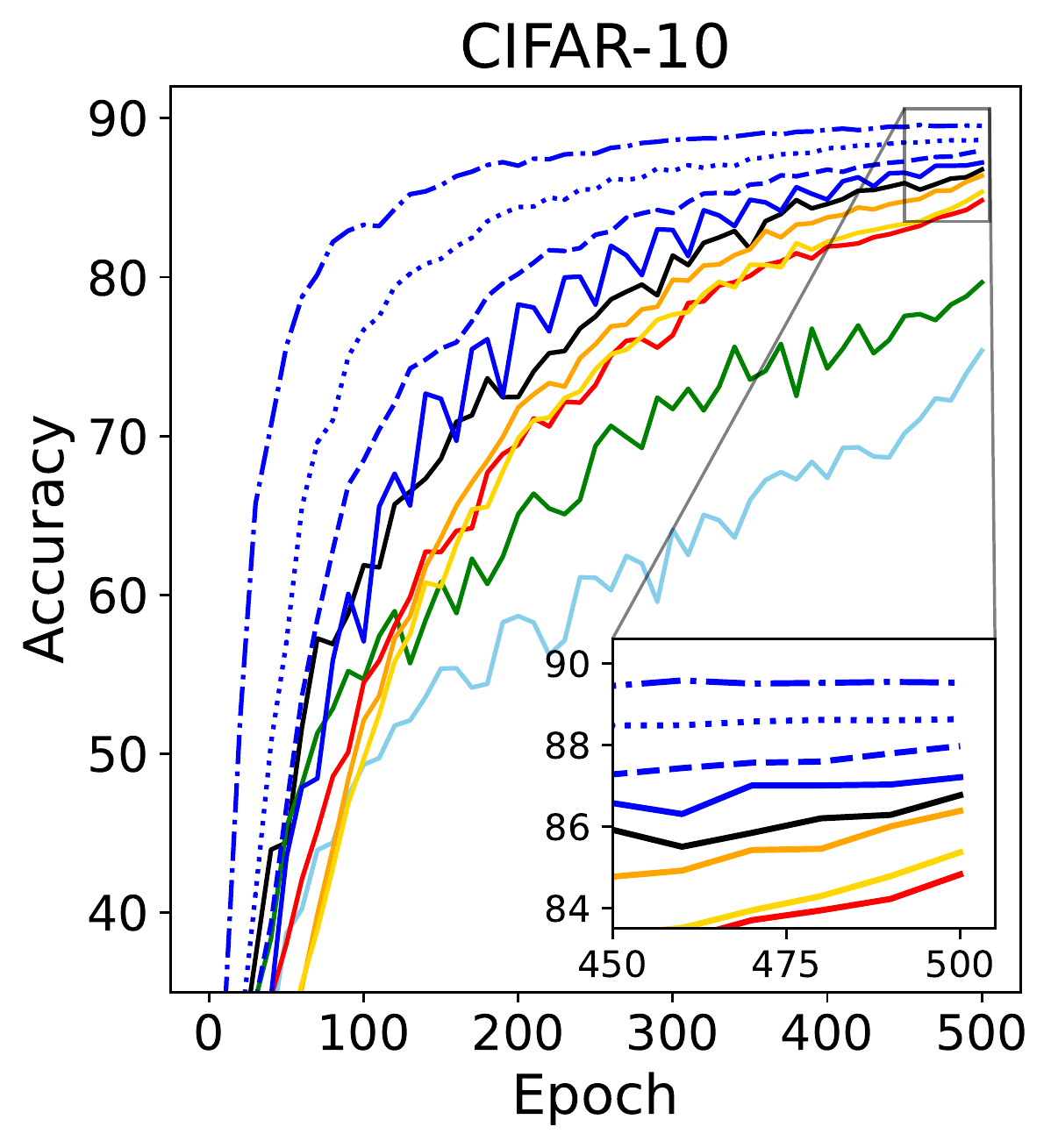}
            \includegraphics[width=0.46\hsize]{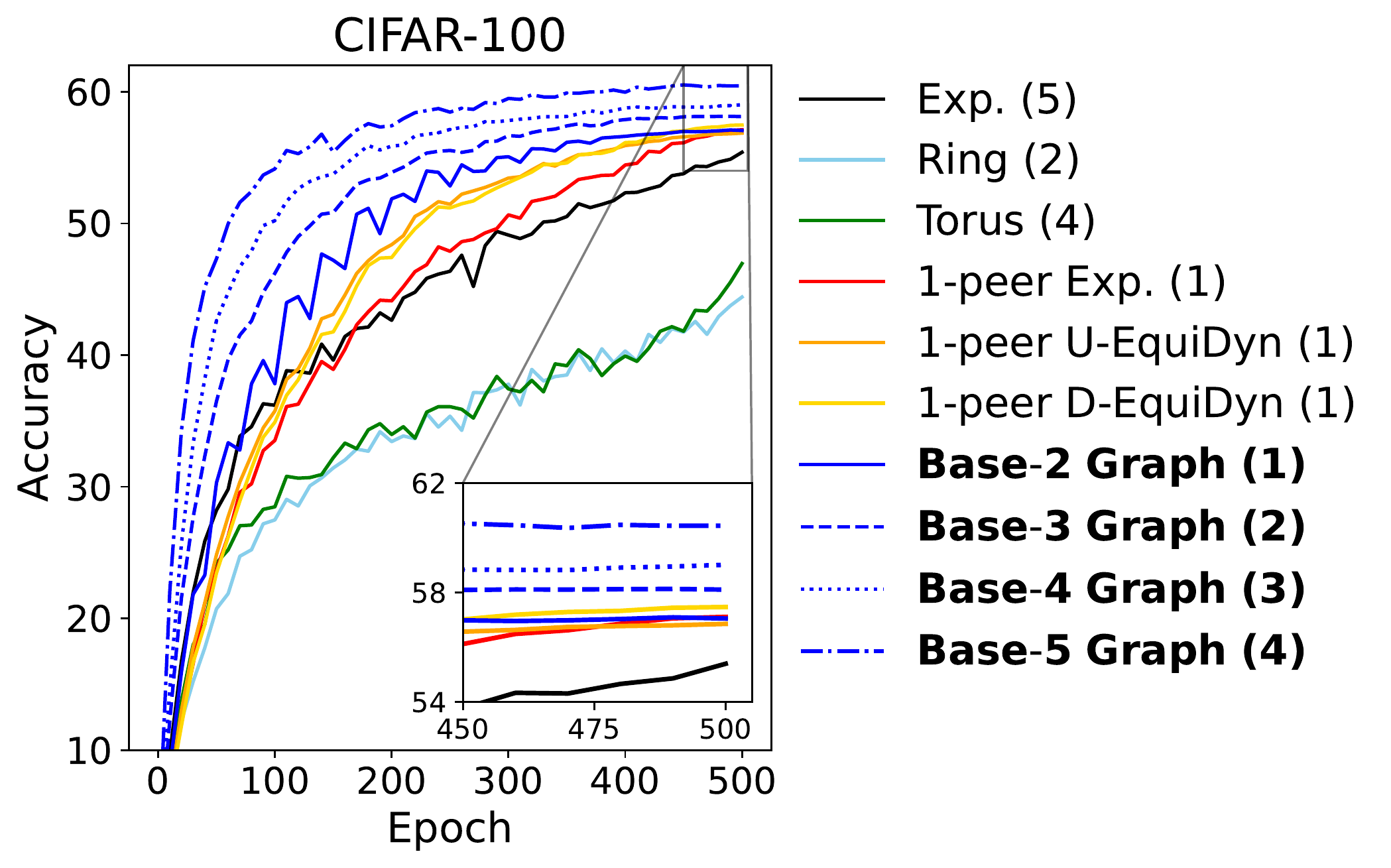}
        \vskip - 0.1 in
        \caption{$\alpha = 0.1$}
        \label{fig:learning_curves_noniid}
    \end{subfigure}
    \vskip - 0.05 in
    \caption{Test accuracy (\%) of DSGD on various topologies with $n=25$. The number in the bracket indicates the maximum degree of a topology. We also compared with dense variants of the $1$-peer \{U, D\}-EquiDyn \cite{song2022communication} in Sec.~\ref{sec:equistatic}, showing the superior performance of the \proposed{$(k+1)$}.}
    \label{fig:learning_curves}
\end{figure}

\textbf{Setup.}
Let $x_i \in \mathbb{R}$ be the parameter that node $i$ has, and let $\bar{x} \coloneqq \frac{1}{n} \sum_{i=1}^n x_i$.
For each $i$, the initial value of $x_i$ was drawn from Gaussian distribution with mean $0$ and standard variance $1$.
Then, we evaluated how the consensus error $\frac{1}{n} \sum_{i=1}^n ( x_i - \bar{x} )^2$ decreases
when $x_i$ is updated as $x_i \leftarrow \sum_{j=1}^n W_{ij} x_j$
where $\mW$ is the mixing matrix associated with a given topology.

\textbf{Results.}
Figs. \ref{fig:consensus_5000} and \ref{fig:consensus} present how the consensus errors decrease on various topologies.
The results indicate that the \proposed{$(k+1)$} reaches the exact consensus after a finite number of iterations, 
while the other topologies only reach the consensus asymptotically.
Moreover, as the maximum degree $k$ increases, 
the \proposed{$(k+1)$} reaches the exact consensus with fewer iterations.
We also present the results when $n$ is a power of 2 in Sec.~\ref{sec:other_consensns_optimization},
demonstrating that the $1$-peer exponential graph can reach the exact consensus as well as the \proposed{$2$},
but requires more iterations than the \proposed{$4$}.

\subsection{Decentralized Learning}
\label{sec:decentralized_learning}
Next, we examine the effectiveness of the \proposed{$(k+1)$} in decentralized learning.

\textbf{Setup.}
We used three datasets, Fashion MNIST \cite{xiao2017fashion}, CIFAR-$\{10, 100\}$ \cite{krizhevsky09learningmultiple},
and used LeNet \cite{lecun1998gradientbased} for Fashion MNIST and VGG-11 \cite{simonyanZ2014very} for CIFAR-$\{10, 100\}$.
Additionally, we present the results using ResNet-18 \cite{he2016deep} in Sec.~\ref{sec:resnet}.
The learning rate was tuned by the grid search and we used the cosine learning rate scheduler \cite{loshchilov2017sgdr}.
We distributed the training dataset to nodes by using Dirichlet distributions with hyperparameter $\alpha$ \cite{hsu2019measuring},
conducting experiments in both homogeneous and heterogeneous data distribution settings.
As $\alpha$ approaches zero, the data distributions held by each node become more heterogeneous.
We repeated all experiments with three different seed values
and reported their averages.
See Sec.~\ref{sec:hyperparameter} for more detailed settings.

\textbf{Results of DSGD on Various Topologies.}
We compared various topologies combined with the DSGD with momentum \cite{gao2020periodic,lian2017can},
showing the results in Fig.~\ref{fig:learning_curves}.
From Fig.~\ref{fig:learning_curves}, the accuracy differences among topologies are larger as the data distributions are more heterogeneous.
From Fig.~\ref{fig:learning_curves_noniid}, 
the \proposed{$\{2,3,4,5\}$} reach high accuracy faster than the other topologies.
Furthermore, comparing the final accuracy,
the final accuracy of the \proposed{$2$} is comparable to or higher than that of the existing topologies, including the exponential graph.
\begin{wrapfigure}{r}[0pt]{0.425\textwidth}
\vskip - 0.15 in
 \centering
   \includegraphics[width=\hsize]{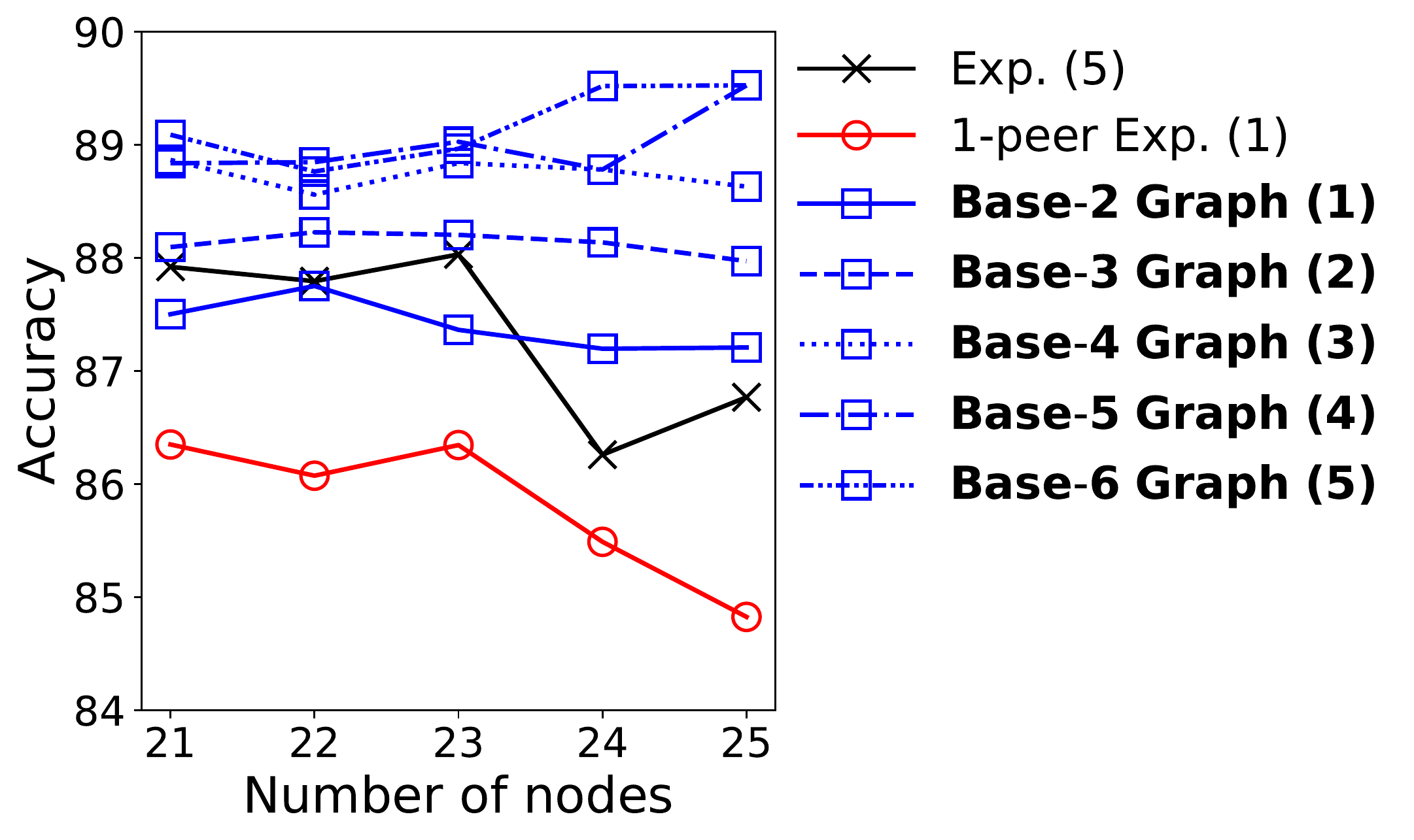}
   \vskip - 0.075 in
   \caption{Test accuracy (\%) of DSGD with CIFAR-10 and $\alpha=0.1$.}
   \label{fig:various_number_of_nodes_cifar}
\vskip - 0.15 in
\end{wrapfigure}
Moreover, the final accuracy of the \proposed{$\{3,4,5\}$} is higher than that of all existing topologies.
From Fig.~\ref{fig:learning_curves_iid},
the accuracy differences among topologies become small when $\alpha = 10$;
however, the \proposed{$5$} still outperforms the other topologies.
In Fig.~\ref{fig:various_number_of_nodes_cifar},
we present the results in cases other than $n=25$,
demonstrating that the \proposed{$2$} outperforms the $1$-peer exponential graph
and the \proposed{$\{3,4,5\}$} can consistently outperform the exponential and $1$-peer exponential graphs for all $n$.
In Sec.~\ref{sec:various_number_of_nodes}, we show the learning curves and the comparison of the consensus rate when $n$ is $21$, $22$, $23$, $24$, and $25$.

\textbf{Results of $D^2$ and QG-DSGDm on Various Topologies.}
The above results demonstrated that the \proposed{$(k+1)$} outperforms the existing topologies, especially when the data distributions
\begin{wrapfigure}{r}[0pt]{0.65\textwidth}
\vskip - 0.15 in
 \centering
   \includegraphics[width=0.355\hsize]{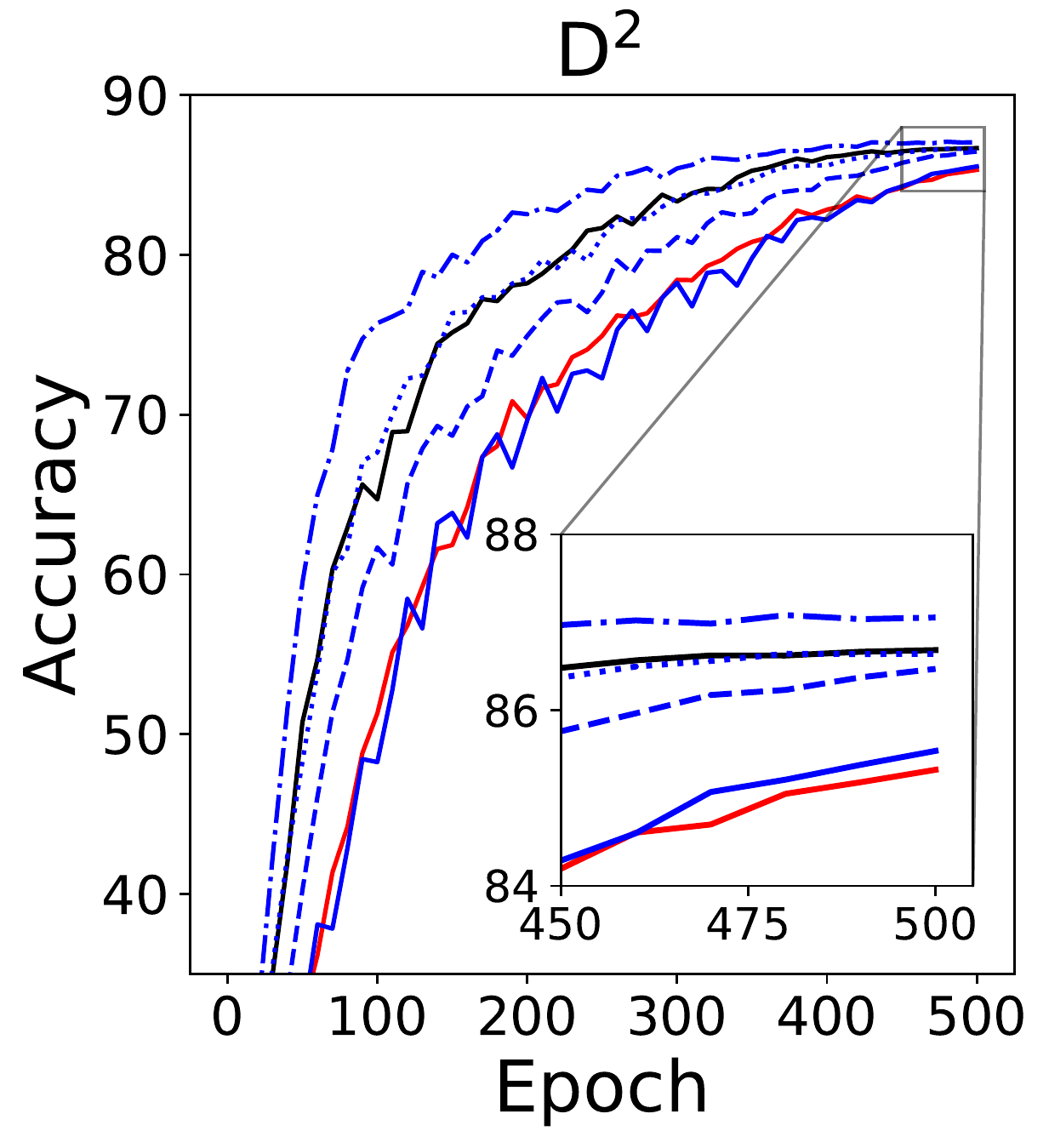}
   \includegraphics[width=0.635\hsize]{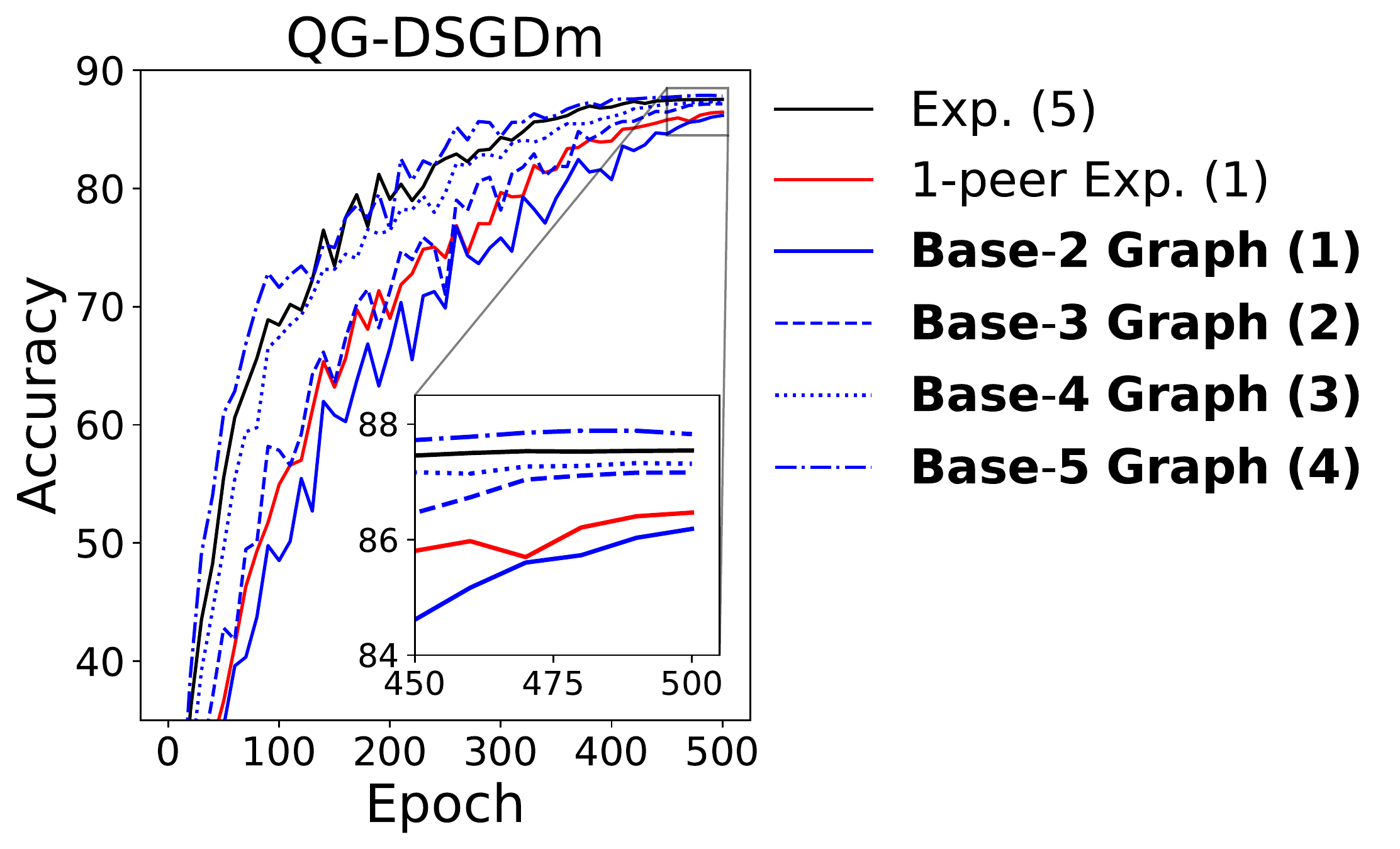}
   \vskip - 0.075 in
   \caption{Test accuracy (\%) of $D^2$ and QG-DSGDm with CIFAR-10, $n=25$, and $\alpha=0.1$.}
   \label{fig:d2_learning_curves}
    \vskip - 0.2 in
\end{wrapfigure}
are heterogeneous. Hence, we next compared the \proposed{$(k+1)$} with the existing topologies in the case where $D^2$ \cite{tang2018d2} and QG-DSGDm \cite{lin2021quasi}, which are robust to data heterogeneity, 
are used as decentralized learning methods.
From Fig.~\ref{fig:d2_learning_curves}, the \proposed{$2$} can achieve
comparable or higher accuracy than the $1$-peer exponential graph,
and the \proposed{$5$} consistently outperforms the exponential graph.
Thus, the \proposed{$(k+1)$} is useful not only for DSGD but also for $D^2$ and QG-DSGDm
and then enables these methods to achieve a reasonable balance between accuracy and communication efficiency.

\section{Conclusion}
In this study, we propose the \proposed{$(k+1)$},
a novel topology with both a fast consensus rate and small maximum degree.
Unlike the existing topologies,
the \proposed{$(k+1)$} is finite-time convergence for any number of nodes and maximum degree $k$.
Thanks to this favorable property, 
the \proposed{$(k+1)$} enables DSGD to obtain both a faster convergence rate and more communication efficiency than the existing topologies,
including the ring, torus, and exponential graph.
Through experiments, we compared the \proposed{$(k+1)$} with various existing topologies,
demonstrating that the \proposed{$(k+1)$} enables various decentralized learning methods to more successfully reconcile accuracy and communication efficiency than the existing topologies.

\section*{Acknowledgments}
Yuki Takezawa, Ryoma Sato, and Makoto Yamada were supported by JSPS KAKENHI Grant Number 23KJ1336, 21J22490, and MEXT KAKENHI Grant Number 20H04243, respectively.

\bibliography{ref}

\begin{thebibliography}{}

\bibitem[Assran et~al., 2019]{assran2019stochastic}
Assran, M., Loizou, N., Ballas, N., and Rabbat, M. (2019).
\newblock Stochastic gradient push for distributed deep learning.
\newblock In {\em International Conference on Machine Learning}.

\bibitem[Chen et~al., 2021]{chen2021accelerating}
Chen, Y., Yuan, K., Zhang, Y., Pan, P., Xu, Y., and Yin, W. (2021).
\newblock Accelerating gossip sgd with periodic global averaging.
\newblock In {\em International Conference on Machine Learning}.

\bibitem[Cyffers et~al., 2022]{cyffers2022muffliato}
Cyffers, E., Even, M., Bellet, A., and Massouli{\'e}, L. (2022).
\newblock Muffliato: Peer-to-peer privacy amplification for decentralized
  optimization and averaging.
\newblock In {\em Advances in Neural Information Processing Systems}.

\bibitem[Gao and Huang, 2020]{gao2020periodic}
Gao, H. and Huang, H. (2020).
\newblock Periodic stochastic gradient descent with momentum for decentralized
  training.
\newblock In {\em arXiv}.

\bibitem[He et~al., 2016]{he2016deep}
He, K., Zhang, X., Ren, S., and Sun, J. (2016).
\newblock Deep residual learning for image recognition.
\newblock In {\em IEEE Conference on Computer Vision and Pattern Recognition}.

\bibitem[Horv{\'a}th and Richtarik, 2021]{horvath2021better}
Horv{\'a}th, S. and Richtarik, P. (2021).
\newblock A better alternative to error feedback for communication-efficient
  distributed learning.
\newblock In {\em International Conference on Learning Representations}.

\bibitem[Hsu et~al., 2019]{hsu2019measuring}
Hsu, T.-M.~H., Qi, and Brown, M. (2019).
\newblock Measuring the effects of non-identical data distribution for
  federated visual classification.
\newblock In {\em arXiv}.

\bibitem[Kairouz et~al., 2021]{kairouz2021adcanves}
Kairouz, P., McMahan, H.~B., Avent, B., Bellet, A., Bennis, M., Bhagoji, A.~N.,
  Bonawitz, K., Charles, Z., Cormode, G., Cummings, R., D’Oliveira, R. G.~L.,
  Eichner, H., Rouayheb, S.~E., Evans, D., Gardner, J., Garrett, Z., Gascón,
  A., Ghazi, B., Gibbons, P.~B., Gruteser, M., Harchaoui, Z., He, C., He, L.,
  Huo, Z., Hutchinson, B., Hsu, J., Jaggi, M., Javidi, T., Joshi, G., Khodak,
  M., Konecný, J., Korolova, A., Koushanfar, F., Koyejo, S., Lepoint, T., Liu,
  Y., Mittal, P., Mohri, M., Nock, R., Özgür, A., Pagh, R., Qi, H., Ramage,
  D., Raskar, R., Raykova, M., Song, D., Song, W., Stich, S.~U., Sun, Z.,
  Suresh, A.~T., Tramèr, F., Vepakomma, P., Wang, J., Xiong, L., Xu, Z., Yang,
  Q., Yu, F.~X., Yu, H., and Zhao, S. (2021).
\newblock Advances and open problems in federated learning.
\newblock In {\em Foundations and Trends in Machine Learning}.

\bibitem[Karimireddy et~al., 2020]{karimireddy2020scaffold}
Karimireddy, S.~P., Kale, S., Mohri, M., Reddi, S., Stich, S., and Suresh,
  A.~T. (2020).
\newblock {SCAFFOLD}: Stochastic controlled averaging for federated learning.
\newblock In {\em International Conference on Machine Learning}.

\bibitem[Koloskova et~al., 2020a]{koloskova2020decentralized}
Koloskova, A., Lin, T., Stich, S.~U., and Jaggi, M. (2020a).
\newblock Decentralized deep learning with arbitrary communication compression.
\newblock In {\em International Conference on Learning Representations}.

\bibitem[Koloskova et~al., 2020b]{koloskova2020unified}
Koloskova, A., Loizou, N., Boreiri, S., Jaggi, M., and Stich, S. (2020b).
\newblock A unified theory of decentralized {SGD} with changing topology and
  local updates.
\newblock In {\em International Conference on Machine Learning}.

\bibitem[Koloskova et~al., 2019]{koloskova2019decentralized}
Koloskova, A., Stich, S., and Jaggi, M. (2019).
\newblock Decentralized stochastic optimization and gossip algorithms with
  compressed communication.
\newblock In {\em International Conference on Machine Learning}.

\bibitem[Kong et~al., 2021]{kong2021consensus}
Kong, L., Lin, T., Koloskova, A., Jaggi, M., and Stich, S. (2021).
\newblock Consensus control for decentralized deep learning.
\newblock In {\em International Conference on Machine Learning}.

\bibitem[Krizhevsky, 2009]{krizhevsky09learningmultiple}
Krizhevsky, A. (2009).
\newblock Learning multiple layers of features from tiny images.
\newblock Technical report.

\bibitem[LeCun et~al., 1998]{lecun1998gradientbased}
LeCun, Y., Bottou, L., Bengio, Y., and Haffner, P. (1998).
\newblock Gradient-based learning applied to document recognition.
\newblock In {\em IEEE}.

\bibitem[Li et~al., 2020]{tian2020federated}
Li, T., Sahu, A.~K., Zaheer, M., Sanjabi, M., Talwalkar, A., and Smith, V.
  (2020).
\newblock Federated optimization in heterogeneous networks.
\newblock In {\em arXiv}.

\bibitem[Li et~al., 2019]{li2019decentralized}
Li, Z., Shi, W., and Yan, M. (2019).
\newblock A decentralized proximal-gradient method with network independent
  step-sizes and separated convergence rates.
\newblock In {\em IEEE Transactions on Signal Processing}.

\bibitem[Lian et~al., 2017]{lian2017can}
Lian, X., Zhang, C., Zhang, H., Hsieh, C.-J., Zhang, W., and Liu, J. (2017).
\newblock Can decentralized algorithms outperform centralized algorithms? a
  case study for decentralized parallel stochastic gradient descent.
\newblock In {\em Advances in Neural Information Processing Systems}.

\bibitem[Lian et~al., 2018]{lian2018asynchronous}
Lian, X., Zhang, W., Zhang, C., and Liu, J. (2018).
\newblock Asynchronous decentralized parallel stochastic gradient descent.
\newblock In {\em International Conference on Machine Learning}.

\bibitem[Lin et~al., 2021]{lin2021quasi}
Lin, T., Karimireddy, S.~P., Stich, S., and Jaggi, M. (2021).
\newblock Quasi-global momentum: Accelerating decentralized deep learning on
  heterogeneous data.
\newblock In {\em International Conference on Machine Learning}.

\bibitem[Lorenzo and Scutari, 2016]{lorenzo2016next}
Lorenzo, P.~D. and Scutari, G. (2016).
\newblock Next: In-network nonconvex optimization.
\newblock In {\em IEEE Transactions on Signal and Information Processing over
  Networks}.

\bibitem[Loshchilov and Hutter, 2017]{loshchilov2017sgdr}
Loshchilov, I. and Hutter, F. (2017).
\newblock {SGDR}: Stochastic gradient descent with warm restarts.
\newblock In {\em International Conference on Learning Representations}.

\bibitem[Lu and De~Sa, 2020]{lu2020moniqua}
Lu, Y. and De~Sa, C. (2020).
\newblock Moniqua: Modulo quantized communication in decentralized {SGD}.
\newblock In {\em International Conference on Machine Learning}.

\bibitem[Lu and De~Sa, 2021]{lu2021optimal}
Lu, Y. and De~Sa, C. (2021).
\newblock Optimal complexity in decentralized training.
\newblock In {\em International Conference on Machine Learning}.

\bibitem[Marfoq et~al., 2020]{marfoq2020throughput}
Marfoq, O., Xu, C., Neglia, G., and Vidal, R. (2020).
\newblock Throughput-optimal topology design for cross-silo federated learning.
\newblock In {\em Advances in Neural Information Processing Systems}.

\bibitem[McMahan et~al., 2017]{mcmahan2017communication}
McMahan, B., Moore, E., Ramage, D., Hampson, S., and Arcas, B. A.~y. (2017).
\newblock Communication-efficient learning of deep networks from decentralized
  data.
\newblock In {\em International Conference on Artificial Intelligence and
  Statistics}.

\bibitem[Murata and Suzuki, 2021]{murata2021bias}
Murata, T. and Suzuki, T. (2021).
\newblock Bias-variance reduced local sgd for less heterogeneous federated
  learning.
\newblock In {\em International Conference on Machine Learning}.

\bibitem[Nedi{\'c} et~al., 2018]{nedic2018network}
Nedi{\'c}, A., Olshevsky, A., and Rabbat, M.~G. (2018).
\newblock Network topology and communication-computation tradeoffs in
  decentralized optimization.
\newblock In {\em IEEE}.

\bibitem[Nedi{\'c} et~al., 2017]{nedic2017achieving}
Nedi{\'c}, A., Olshevsky, A., and Shi, W. (2017).
\newblock Achieving geometric convergence for distributed optimization over
  time-varying graphs.
\newblock In {\em SIAM Journal on Optimization}.

\bibitem[Pu and Nedi{\'c}, 2021]{pu2018distributed}
Pu, S. and Nedi{\'c}, A. (2021).
\newblock Distributed stochastic gradient tracking methods.
\newblock In {\em Mathematical Programming}.

\bibitem[Shi et~al., 2016]{shi2016finite}
Shi, G., Li, B., Johansson, M., and Johansson, K.~H. (2016).
\newblock Finite-time convergent gossiping.
\newblock In {\em IEEE Transactions on Networking}.

\bibitem[Simonyan and Zisserman, 2015]{simonyanZ2014very}
Simonyan, K. and Zisserman, A. (2015).
\newblock Very deep convolutional networks for large-scale image recognition.
\newblock In {\em International Conference on Learning Representations}.

\bibitem[Song et~al., 2022]{song2022communication}
Song, Z., Li, W., Jin, K., Shi, L., Yan, M., Yin, W., and Yuan, K. (2022).
\newblock Communication-efficient topologies for decentralized learning with
  {$O(1)$} consensus rate.
\newblock In {\em Advances in Neural Information Processing Systems}.

\bibitem[Takezawa et~al., 2023]{takezawa2022momentum}
Takezawa, Y., Bao, H., Niwa, K., Sato, R., and Yamada, M. (2023).
\newblock Momentum tracking: Momentum acceleration for decentralized deep
  learning on heterogeneous data.
\newblock In {\em Transactions on Machine Learning Research}.

\bibitem[Tang et~al., 2018a]{tang2018communication}
Tang, H., Gan, S., Zhang, C., Zhang, T., and Liu, J. (2018a).
\newblock Communication compression for decentralized training.
\newblock In {\em Advances in Neural Information Processing Systems}.

\bibitem[Tang et~al., 2018b]{tang2018d2}
Tang, H., Lian, X., Yan, M., Zhang, C., and Liu, J. (2018b).
\newblock {$D^2$}: Decentralized training over decentralized data.
\newblock In {\em International Conference on Machine Learning}.

\bibitem[Vogels et~al., 2021]{vogels2021relaysum}
Vogels, T., He, L., Koloskova, A., Karimireddy, S.~P., Lin, T., Stich, S.~U.,
  and Jaggi, M. (2021).
\newblock Relaysum for decentralized deep learning on heterogeneous data.
\newblock In {\em Advances in Neural Information Processing Systems}.

\bibitem[Vogels et~al., 2020]{vogels2020practical}
Vogels, T., Karimireddy, S.~P., and Jaggi, M. (2020).
\newblock Practical low-rank communication compression in decentralized deep
  learning.
\newblock In {\em Advances in Neural Information Processing Systems}.

\bibitem[Wang et~al., 2019]{wang2019matcha}
Wang, J., Sahu, A.~K., Yang, Z., Joshi, G., and Kar, S. (2019).
\newblock Matcha: Speeding up decentralized sgd via matching decomposition
  sampling.
\newblock In {\em Indian Control Conference}.

\bibitem[Wu and He, 2018]{wu2018group}
Wu, Y. and He, K. (2018).
\newblock Group normalization.
\newblock In {\em European Conference on Computer Vision}.

\bibitem[Xiao et~al., 2017]{xiao2017fashion}
Xiao, H., Rasul, K., and Vollgraf, R. (2017).
\newblock Fashion-mnist: a novel image dataset for benchmarking machine
  learning algorithms.
\newblock In {\em arXiv}.

\bibitem[Xin and Khan, 2020]{xin2020distributed}
Xin, R. and Khan, U.~A. (2020).
\newblock Distributed heavy-ball: A generalization and acceleration of
  first-order methods with gradient tracking.
\newblock In {\em IEEE Transactions on Automatic Control}.

\bibitem[Ying et~al., 2021]{ying2021exponential}
Ying, B., Yuan, K., Chen, Y., Hu, H., Pan, P., and Yin, W. (2021).
\newblock Exponential graph is provably efficient for decentralized deep
  training.
\newblock In {\em Advances in Neural Information Processing Systems}.

\bibitem[Yu et~al., 2019]{yu2019on}
Yu, H., Jin, R., and Yang, S. (2019).
\newblock On the linear speedup analysis of communication efficient momentum
  {SGD} for distributed non-convex optimization.
\newblock In {\em International Conference on Machine Learning}.

\bibitem[Yuan et~al., 2021]{yuan202decentlam}
Yuan, K., Chen, Y., Huang, X., Zhang, Y., Pan, P., Xu, Y., and Yin, W. (2021).
\newblock Decent{L}a{M}: Decentralized momentum sgd for large-batch deep
  training.
\newblock In {\em International Conference on Computer Vision}.

\bibitem[Yuan et~al., 2019]{yuan2019exact}
Yuan, K., Ying, B., Zhao, X., and Sayed, A.~H. (2019).
\newblock Exact diffusion for distributed optimization and learning—part i:
  Algorithm development.
\newblock In {\em IEEE Transactions on Signal Processing}.

\bibitem[Zhao et~al., 2022]{zhao2022beer}
Zhao, H., Li, B., Li, Z., Richt{\'a}rik, P., and Chi, Y. (2022).
\newblock {BEER}: Fast {$O(1/T)$} rate for decentralized nonconvex optimization
  with communication compression.
\newblock In {\em Advances in Neural Information Processing Systems}.

\bibitem[Zhu et~al., 2022]{zhu2022topology}
Zhu, T., He, F., Zhang, L., Niu, Z., Song, M., and Tao, D. (2022).
\newblock Topology-aware generalization of decentralized {SGD}.
\newblock In {\em International Conference on Machine Learning}.

\end{thebibliography}


\newpage
\appendix

\section{Detailed Explanation of $k$-peer Hyperhypercube Graph}
In this section, we explain Alg.~\ref{alg:hyperhypercube} in more detail.
The \hyperhypercube{$k$} mainly consists of the following five steps.
\begin{description}
    \item[Step 1.] Decompose $n$ as $n = n_1 \times \cdots \times n_L$ with minimum $L$ such that $n_l \in [k + 1]$ for all $l \in [L]$.
    \item[Step 2.] If $L=1$, we make all nodes obtain the average of parameters in $V$ by using the complete graph. If $L \geq 2$, we split $V$ into disjoint subsets $V_1, \cdots, V_{n_L}$ such that $|V_l| = \tfrac{n}{n_L}$ for all $l \in [n_L]$ and continue to step $3$.
    \item[Step 3.] For all $l \in [n_L]$, we make all nodes in $V_l$ obtain the average of parameters in $V_l$ by using the \hyperhypercube{$k$} $\mathcal{H}_k (V_l)$.
    \item [Step 4.] We take $n_L$ nodes from $V_1, \cdots, V_{n_L}$ respectively and construct a set $U_1$. Similarly, we construct $U_2, \cdots, U_{n_L}$ such that $U_1, \cdots, U_{n_L}$ are disjoint sets.
    \item [Step 5.] For all $l \in [n_L]$, we make all nodes in $U_l$ obtain the average of parameters in $U_l$ by using the complete graph. Because the average of parameter $U_l$ is equivalent to the average in $V$ after step $4$, all nodes reach the exact consensus.
\end{description}

When $n \leq k+1$, the \hyperhypercube{$k$} becomes the complete graph because of step $2$.
When $n > k+1$, we decompose $n$ in step $1$ and construct the \hyperhypercube{$k$} recursively in step $3$.
Thus, the \hyperhypercube{$k$} can make all nodes reach the exact consensus by the sequence of $L$ graphs.

Using the example provided in Fig.~\ref{fig:hyperhypercube_in_appendix}, we explain the \hyperhypercube{$k$} in a more detailed manner.
When $n=12$, we decompose $12$ as $2 \times 2 \times 3$.
In step $2$, we split $V \coloneqq \{1, \cdots, 12\}$ into $V_1 \coloneqq \{1, \cdots, 4\}$, $V_2 \coloneqq \{5, \cdots, 8\}$, and $V_3 \coloneqq \{9, \cdots, 12\}$.
Step $3$ corresponds to the first two graphs in Fig.~\ref{fig:2_peer_hyperhypercube_12_in_appendix}.
As shown in Fig.~\ref{fig:2_peer_hyperhypercube_4_in_appendix},
the subgraphs consisting of $V_1$, $V_2$, and $V_3$ in the first two graphs in Fig.~\ref{fig:2_peer_hyperhypercube_12_in_appendix} are equivalent to the \hyperhypercube{$k$} with the number of nodes $4$.
Thus, all nodes reach the exact consensus by exchanging parameters in Fig.~\ref{fig:2_peer_hyperhypercube_12_in_appendix}.

\begin{figure}[h!]
    \centering
    \begin{subfigure}{0.35\hsize}
        \includegraphics[width=\hsize]{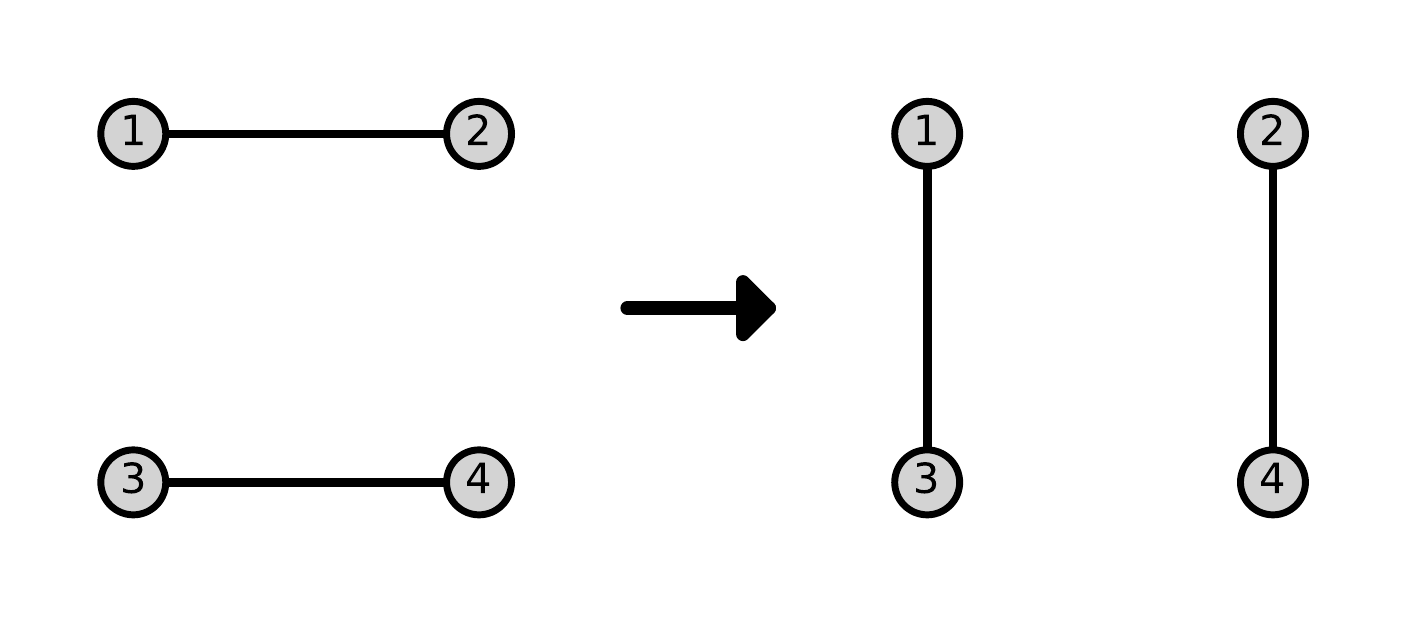}
        \vskip - 0.15 in
        \caption{$n=4 (= 2\times 2)$}
        \label{fig:2_peer_hyperhypercube_4_in_appendix}
    \end{subfigure}
    \hfill
    \begin{subfigure}{0.6\hsize}
        \includegraphics[width=\hsize]{pic/2_peer_hyperhypercube_12.pdf}
        \vskip - 0.15 in
        \caption{$n=12 (= 2 \times 2 \times 3)$}
        \label{fig:2_peer_hyperhypercube_12_in_appendix}
    \end{subfigure}
    \vskip - 0.05 in
    \caption{Illustration of the \hyperhypercube{$2$}. In Fig.~\ref{fig:2_peer_hyperhypercube_4_in_appendix}, all edge weights are $\tfrac{1}{2}$. In Fig.~\ref{fig:2_peer_hyperhypercube_12_in_appendix}, edge weights are $\tfrac{1}{2}$ in the first two graphs and $\tfrac{1}{3}$ in the last graph.}
    \label{fig:hyperhypercube_in_appendix}
\vskip - 0.2 in
\end{figure}

\newpage
\section{Detailed Explanation of Simple Base-$(k+1)$ Graph with $k \geq 2$}
\label{sec:simple_base_with_large_k}
In Sec.~\ref{sec:k_peer_simple_FAIRY}, we explain Alg.~\ref{alg:simple_k_peer_FAIRY} only in the case where maximum degree $k$ is one.
In this section, we explain the details of Alg.~\ref{alg:simple_k_peer_FAIRY} in the case with $k \geq 2$.

The \simpleProposed{$(k+1)$} mainly consists of the following five steps.
\begin{description}
    \item[Step 1.] As in the base-$(k+1)$ number of $n$, we decompose $n$ as $n = a_1 (k+1)^{p_1} + \cdots + a_L (k+1)^{p_L}$ in line 1, and then split $V$ into disjoint subsets $V_1, \cdots, V_L$ such that $|V_l| = a_l (k+1)^{p_l}$ for all $l \in [L]$. 
    \item[Step 2.] For all $l \in [L]$, we split $V_l$ into disjoint subsets $V_{l,1}, \cdots, V_{l,a_l}$ such that $|V_{l,a}| = (k+1)^{p_l}$ for all $a \in [a_l]$ in line $3$.
    \item[Step 3.] For all $l \in [L]$, we make all nodes in $V_l$ obtain the average of parameters in $V_l$ using the \hyperhypercube{$k$} $\mathcal{H}_k (V_l)$ in line 11. Then, we initialize $l^\prime$ as one.
    \item[Step 4.] Each node in $V_{l^\prime + 1} \cup \cdots \cup V_L$ exchange parameters with $a_{l^\prime}$ nodes in $V_{l^\prime} (= V_{l^\prime,1} \cup \cdots \cup V_{l^\prime,a_{l^\prime}})$ such that the average in $V_{l^\prime, a}$ becomes equivalent to the average in $V$ for all $a \in [a_{l^\prime}]$. We increase $l^\prime$ by one and repeat step $4$ until $l^\prime = L$. This procedure corresponds to line 15.
    \item[Step 5.] For all $l \in [L]$ and $a \in [a_l]$, we make all nodes in $V_{l, a}$ obtain the average in $V_{l,a}$ using the \hyperhypercube{$k$} $\mathcal{H}_k (V_{l,a})$. Since the average in $V_{l, a}$ is equivalent to the average in $V$ after step $4$, all nodes reach the exact consensus. This procedure corresponds to line 25.
\end{description}
The major difference compared with the case where $k=1$ is step $4$.
In the case where $k=1$, 
each node in $V_{l^\prime + 1} \cup \cdots \cup V_L$ exchange parameters with one node in $V_{l^\prime}$ such that the average in $V_{l^\prime}$ becomes equivalent to that in $V$,
while in the case where $k \geq 2$, 
each node in $V_{l^\prime + 1} \cup \cdots \cup V_L$ exchange parameters 
such that the average in $V_{l^\prime,a}$ becomes equivalent to that in $V$ for all $a \in [a_l]$.
Thanks to this step, we can make all nodes reach the exact consensus using \hyperhypercube{$k$} $\mathcal{H}_k (V_{l, a})$ instead of $\mathcal{H}_k (V_{l})$ in step $5$, and we can reduce the length of a graph sequence.

Using the example provided in Fig.~\ref{fig:2_peer_simple_FAIRY_7}, we explain Alg.~\ref{alg:simple_k_peer_FAIRY} in a more detailed manner.
Let $G^{(1)}, \cdots, G^{(4)}$ denote the graphs depicted in Fig.~\ref{fig:2_peer_simple_FAIRY_7} from left to right, respectively.
First, we split $V \coloneqq \{1,\cdots,7\}$ into $V_1 \coloneqq \{1, \cdots, 6\}$ and $V_2 \coloneqq \{7\}$,
and then split $V_1$ into $V_{1,1} \coloneqq \{ 1,2,3\}$ and $V_{1,2} \coloneqq \{4,5,6\}$.
In step $3$, all nodes in $V_1$ obtain the same parameter by exchanging parameters in $G^{(1)}$ and $G^{(2)}$.
In step $4$, the average in $V_{1,1}$, that in $V_{1,2}$, and that in $V_2$ become the same as the average in all nodes $V$ by exchanging parameters in $G^{(3)}$.
Thus, in step $5$, all nodes reach the exact consensus by exchanging parameters in $G^{(4)}$.
\begin{figure}[h!]
    \centering
    \includegraphics[width=\hsize]{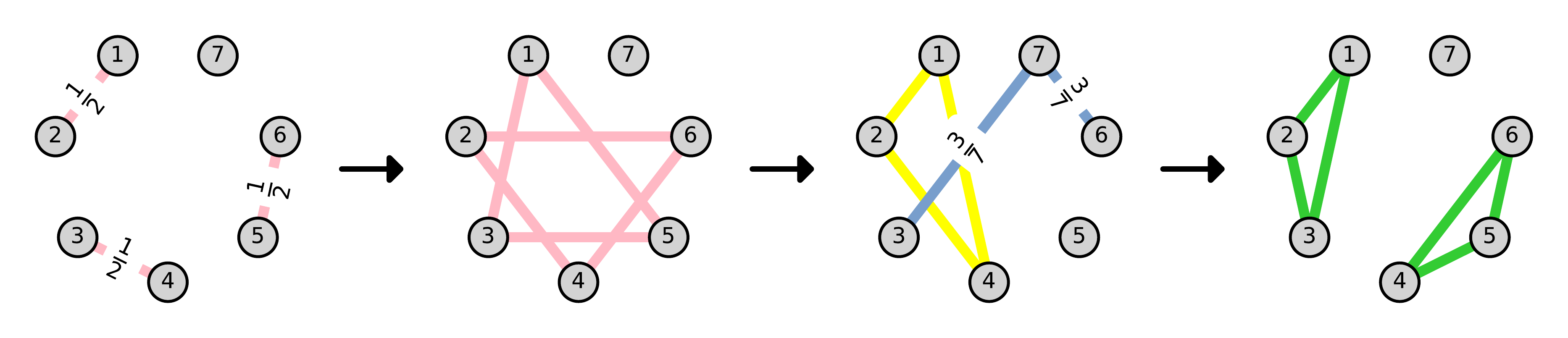}
    \vskip - 0.1 in
    \caption{$k=2, n=7(=2 \times 3 + 1)$. The value on the edge indicates the edge weight. For simplicity, we omit the edge value when it is $\tfrac{1}{3}$.}
    \label{fig:2_peer_simple_FAIRY_7}
\end{figure}

\newpage

\section{Illustration of Topologies}
\label{sec:illustration}
\subsection{Examples}
\label{sec:examples}
Fig.~\ref{fig:other_simple_FAIRY} shows the examples of the \simpleProposed{$(k+1)$}.
Using these examples, we explain how all nodes reach the exact consensus.

We explain the case depicted in Fig.~\ref{fig:2_peer_simple_FAIRY_5}.
Let $G^{(1)}, G^{(2)}, G^{(3)}$ denote the graphs depicted in Fig.~\ref{fig:2_peer_simple_FAIRY_5} from left to right, respectively.
First, we split $V \coloneqq \{1,\cdots,5\}$ into $V_1 \coloneqq \{1, 2, 3\}$ and $V_2 \coloneqq \{4,5\}$,
and then split $V_2$ into $V_{2,1} \coloneqq \{4 \}$ and $V_{2,2} \coloneqq \{ 5 \}$.
After exchanging parameters in $G^{(1)}$,
nodes in $V_1$ and nodes in $V_2$ have the same parameter respectively.
Then, after exchanging parameters in $G^{(2)}$, the average in $V_1$, that in $V_{2,1}$, and that in $V_{2,2}$ become same as the average in $V$.
Thus, by exchanging parameters in $G^{(3)}$, all nodes reach the exact consensus.
Note that edge $(4,5)$ in $G^{(3)}$,
which is added in line 27 in Alg.~\ref{alg:simple_k_peer_FAIRY},
is not necessary for all nodes to reach the exact consensus because nodes $4$ and $5$ already have the same parameter after exchanging parameters in $G^{(2)}$;
however, it is effective in decentralized learning as we explained in Sec.~\ref{sec:k_peer_simple_FAIRY}.
\begin{figure}[h!]
\centering
     \begin{subfigure}{0.6\hsize}
         \centering
         \includegraphics[width=\hsize]{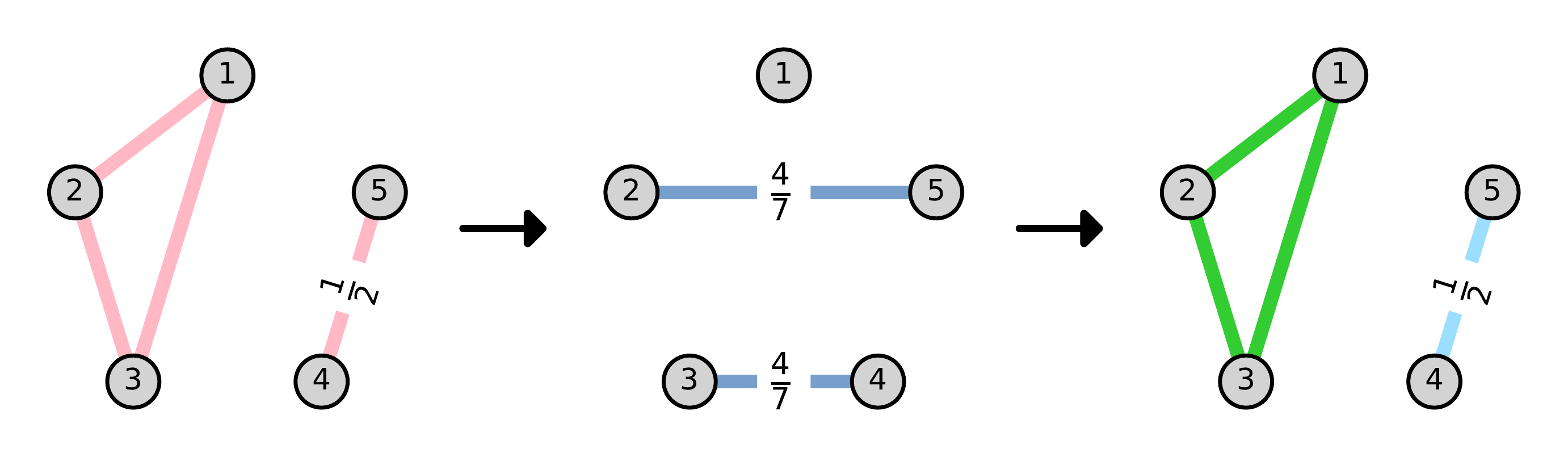}
         \vskip - 0.1 in
         \caption{$k=2, n=5(=3+2)$}
         \label{fig:2_peer_simple_FAIRY_5}
     \end{subfigure}
     \begin{subfigure}{1.0\hsize}
         \centering
         \includegraphics[width=\hsize]{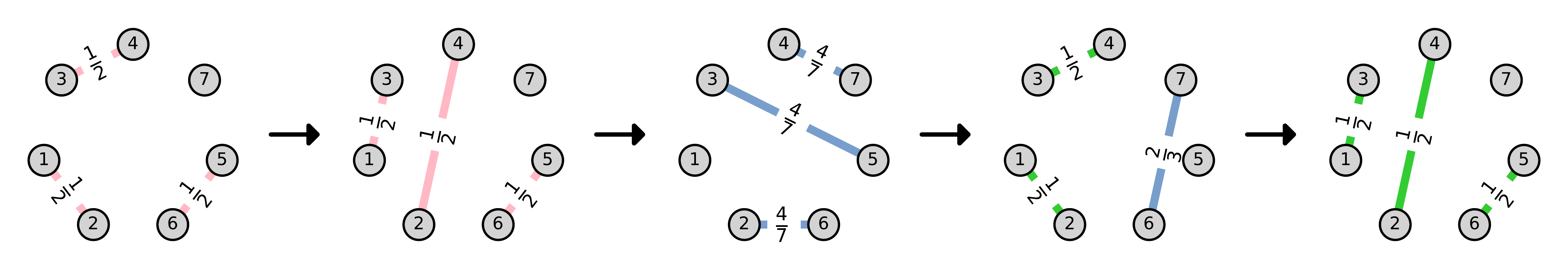}
         \vskip - 0.1 in
         \caption{$k=1, n=7(=2^2 + 2 + 1)$}
     \end{subfigure}
     \caption{Illustration of the \simpleProposed{$(k+1)$}. The edge is colored in the same color as the line of Alg.~\ref{alg:simple_k_peer_FAIRY} where the edge is added. The value on the edge indicates the edge weight. For simplicity, we omit the edge value when it is $\tfrac{1}{3}$.}
     \label{fig:other_simple_FAIRY}
\end{figure}

\subsection{Illustrative Comparison between Simple Base-$(k+1)$ and Base-$(k+1)$ Graphs}
\label{sec:ullustrative_comparison}
In this section, we provide an example of the \simpleProposed{$(k+1)$}, 
explaining the reason why the length of the \proposed{$(k+1)$} is less than that of the \simpleProposed{$(k+1)$}.

Let $G^{(1)}, \cdots, G^{(5)}$ denote the graphs depicted in Fig.~\ref{fig:1_peer_simple_FAIRY_6} from left to right, respectively.
$(G^{(1)}, G^{(2)}, G^{(3)}, G^{(4)}, G^{(5)})$ is finite-time convergence, 
but $(G^{(1)}, G^{(2)}, G^{(3)}, G^{(5)})$ is also finite-time convergence 
because after exchanging parameters in $G^{(3)}$, nodes $3$ and $4$ already have the same parameters.
Then, using the technique proposed in Sec.~\ref{sec:k_peer_FAIRY},
we can remove such unnecessary graphs contained in the \simpleProposed{$(k+1)$} (see Fig.~\ref{fig:1_peer_FAIRY_6}).
Consequently, the \proposed{$(k+1)$} can make all nodes reach the exact consensus faster than the \simpleProposed{$(k+1)$}.

\begin{figure}[h!]
\centering
    \includegraphics[width=\hsize]{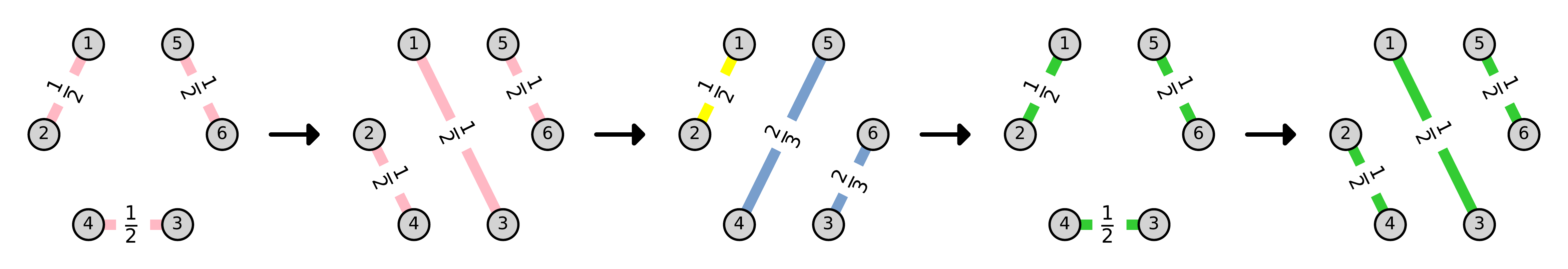}
    \caption{Illustration of the \simpleProposed{$2$} with $n = 6 (=2^2 + 2)$. The edge is colored in the same color as the line of Alg.~\ref{alg:simple_k_peer_FAIRY} where the edge is added.}
    \label{fig:1_peer_simple_FAIRY_6}
\end{figure}

\newpage
\subsection{Additional Examples}
\label{sec:additional_examples}

\subsubsection{Simple Base-$(k+1)$ Graph}
\label{sec:examples_of_simple_FAIRY}
\begin{figure}[h!]
     \centering
     \begin{subfigure}{0.43\hsize}
         \centering
         \includegraphics[width=\hsize]{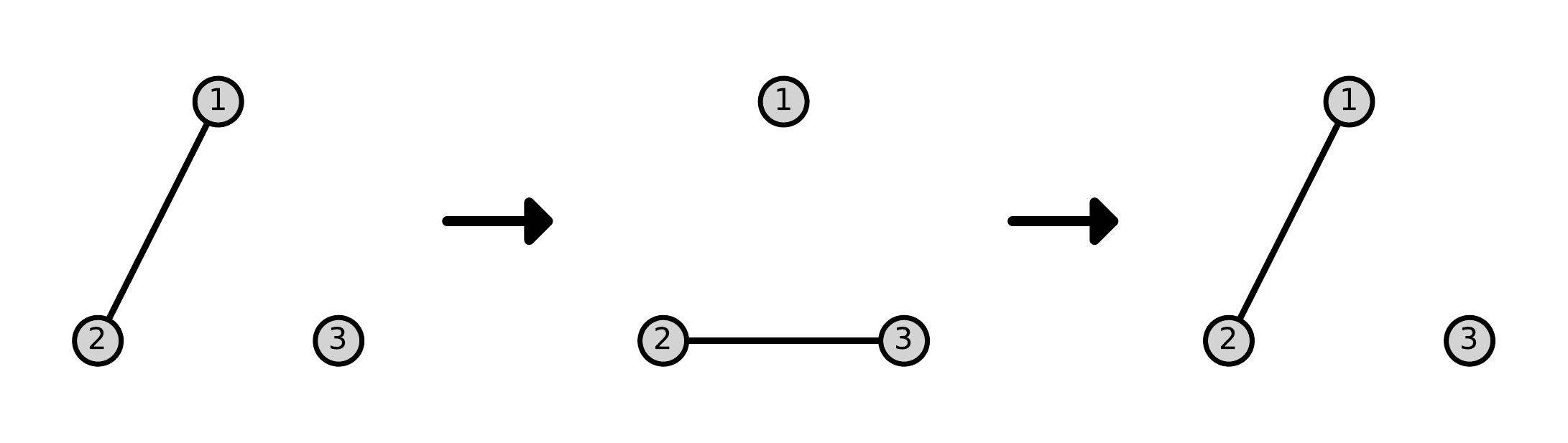}
         \vskip - 0.1 in
         \caption{$n=3$}
         \label{fig:1_peer_simple_FAIRY_3}
     \end{subfigure} \\
     \begin{subfigure}{0.28\hsize}
         \centering
         \includegraphics[width=\hsize]{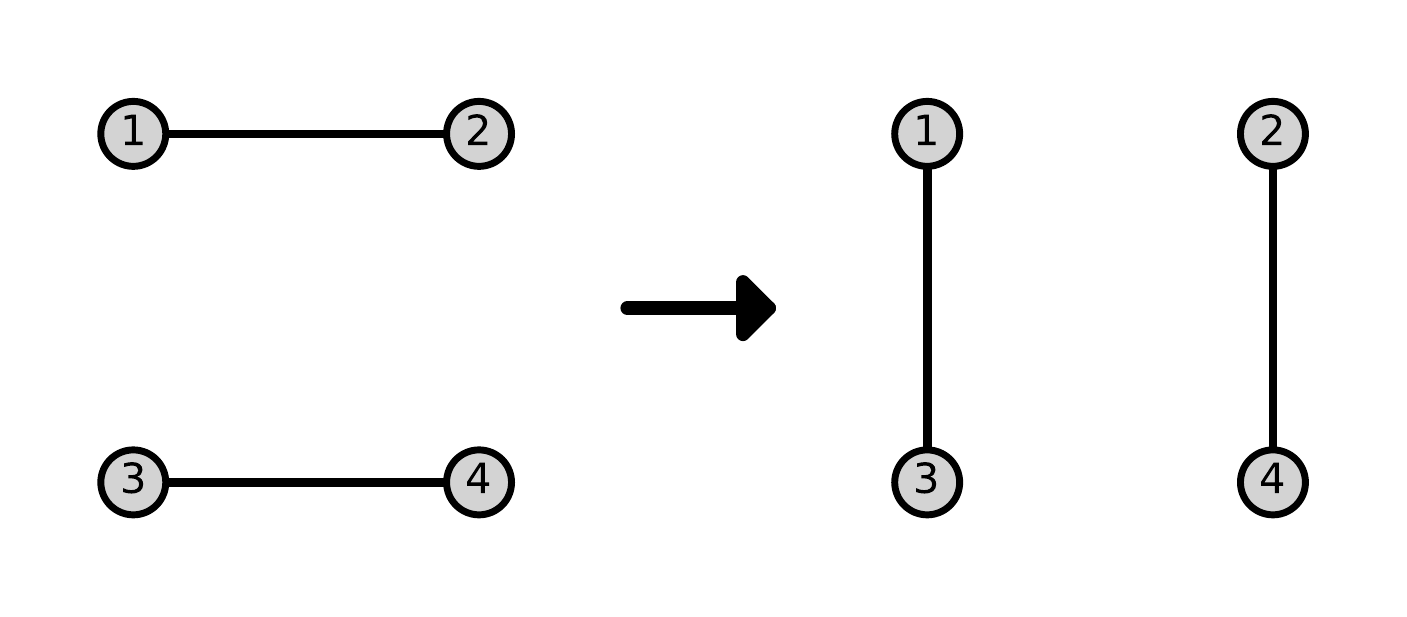}
         \vskip - 0.1 in
         \caption{$n=4$}
     \end{subfigure} \\
     \begin{subfigure}{0.72\hsize}
         \centering
         \includegraphics[width=\hsize]{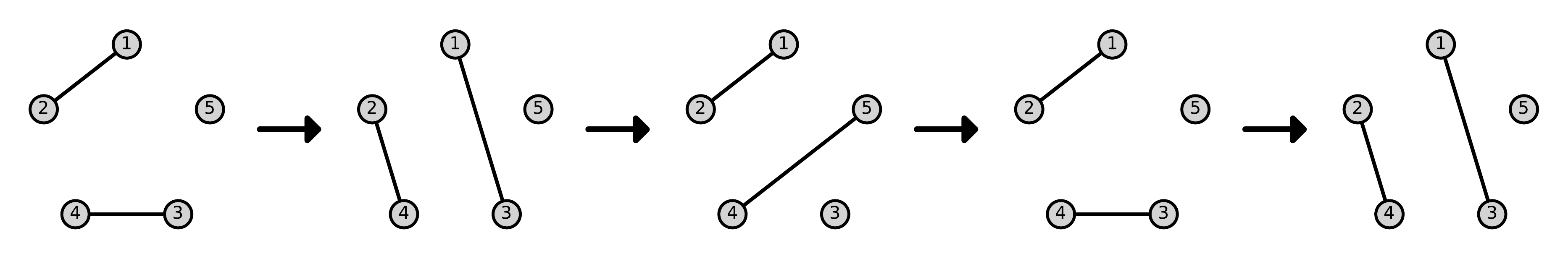}
         \vskip - 0.1 in
         \caption{$n=5$}
     \end{subfigure}
     \hfill
     \begin{subfigure}{0.72\hsize}
         \centering
         \includegraphics[width=\hsize]{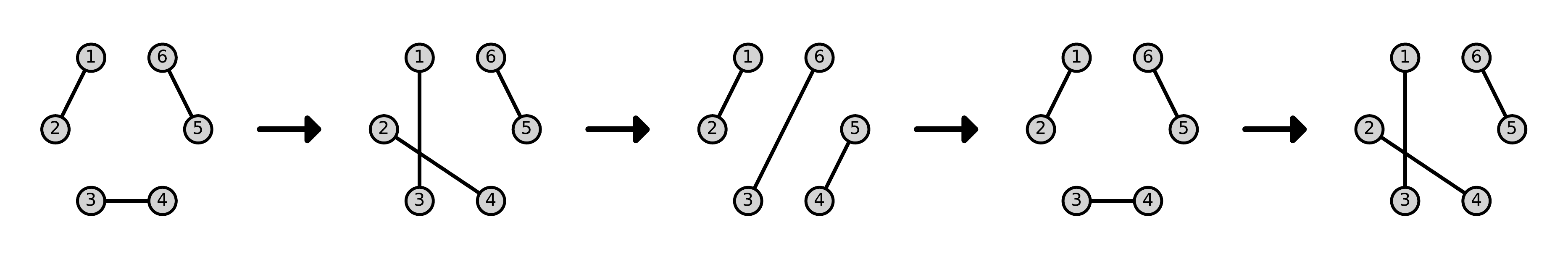}
         \vskip - 0.1 in
         \caption{$n=6$}
     \end{subfigure}
     \hfill
     \begin{subfigure}{0.72\hsize}
         \centering
         \includegraphics[width=\hsize]{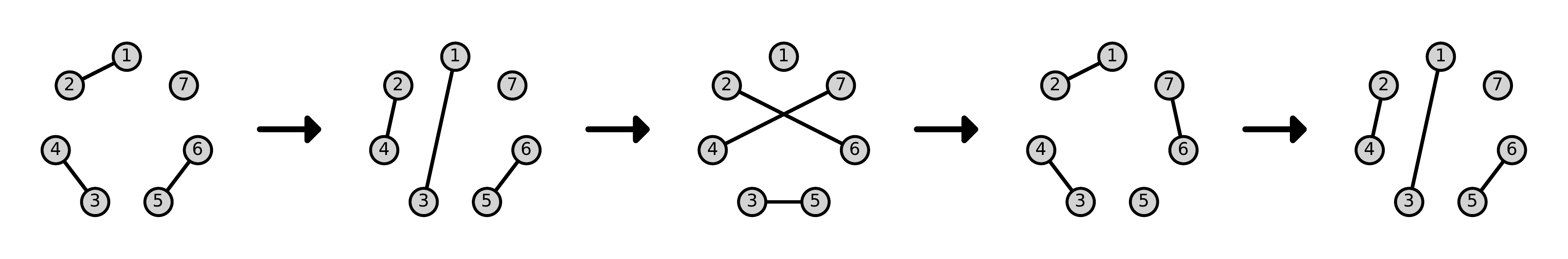}
         \vskip - 0.1 in
         \caption{$n=7$}
     \end{subfigure}
     \hfill
     \begin{subfigure}{0.43\hsize}
         \centering
         \includegraphics[width=\hsize]{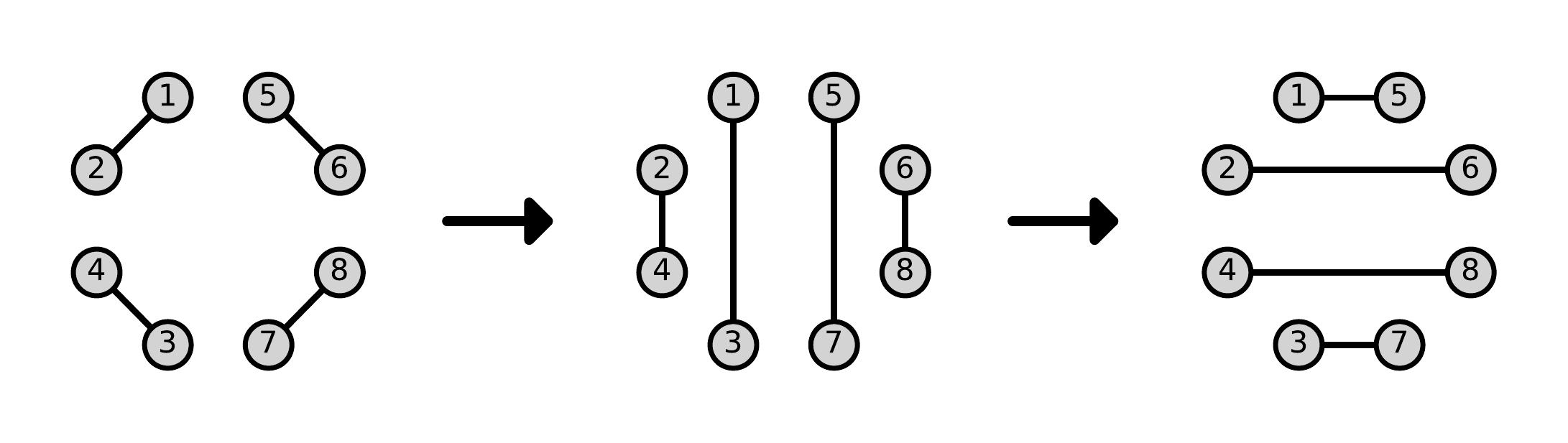}
         \vskip - 0.1 in
         \caption{$n=8$}
     \end{subfigure}
     \begin{subfigure}{\hsize}
         \centering
         \includegraphics[width=\hsize]{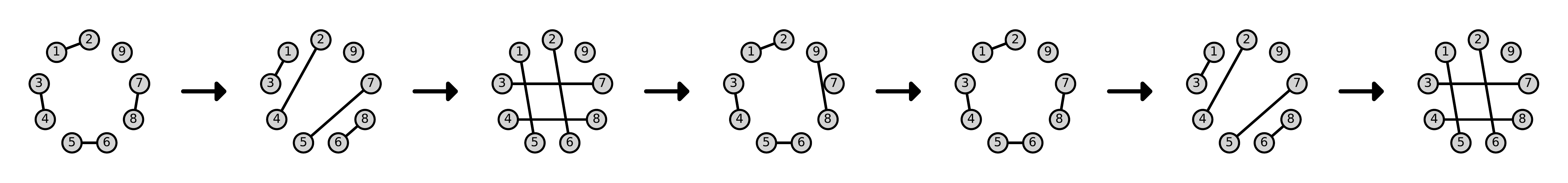}
         \vskip - 0.1 in
         \caption{$n=9$}
     \end{subfigure}
     \begin{subfigure}{1.0\hsize}
         \centering
         \includegraphics[width=\hsize]{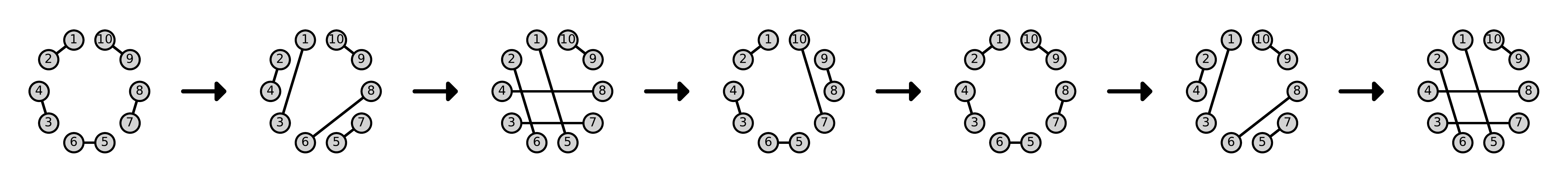}
         \vskip - 0.1 in
         \caption{$n=10$}
     \end{subfigure}
\caption{Illustration of the \simpleProposed{$2$} with the various numbers of nodes.}
\end{figure}

\newpage
\begin{figure}[h!]
     \centering
     \begin{subfigure}{0.18\hsize}
         \centering
         \includegraphics[width=\hsize]{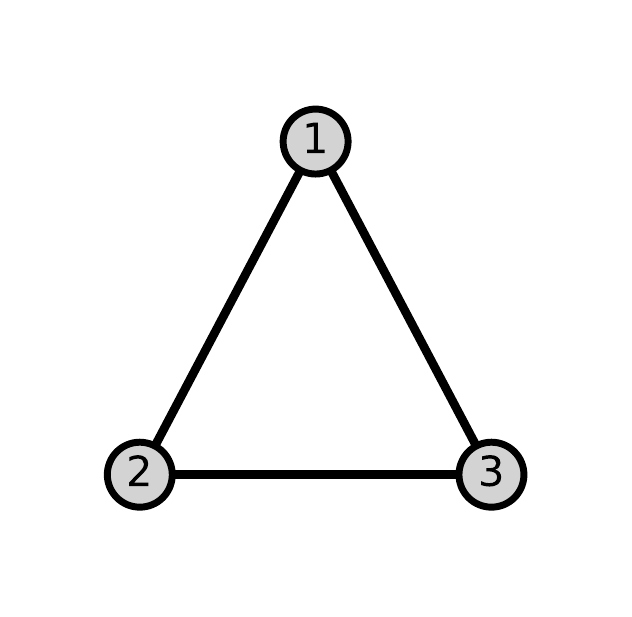}
         \vskip - 0.1 in
         \caption{$n=3$}
     \end{subfigure} \\
     \begin{subfigure}{0.36\hsize}
         \centering
         \includegraphics[width=\hsize]{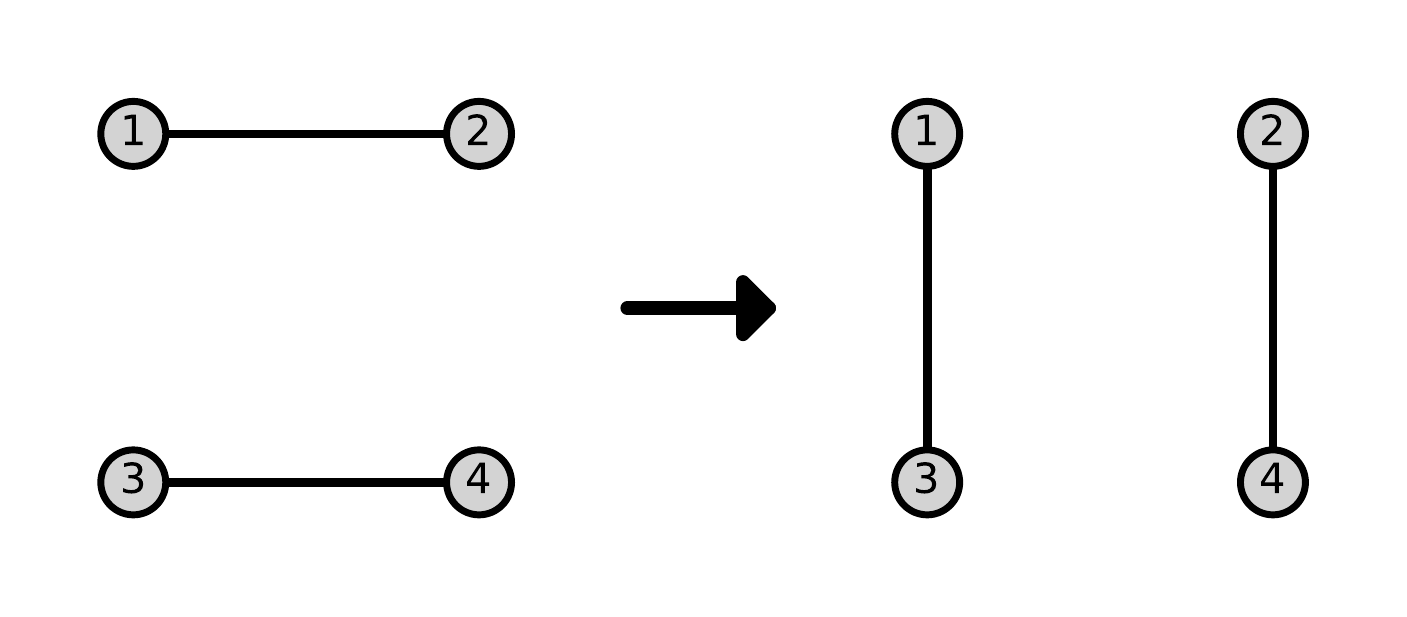}
         \vskip - 0.1 in
         \caption{$n=4$}
     \end{subfigure} \\
     \begin{subfigure}{0.54\hsize}
         \centering
         \includegraphics[width=\hsize]{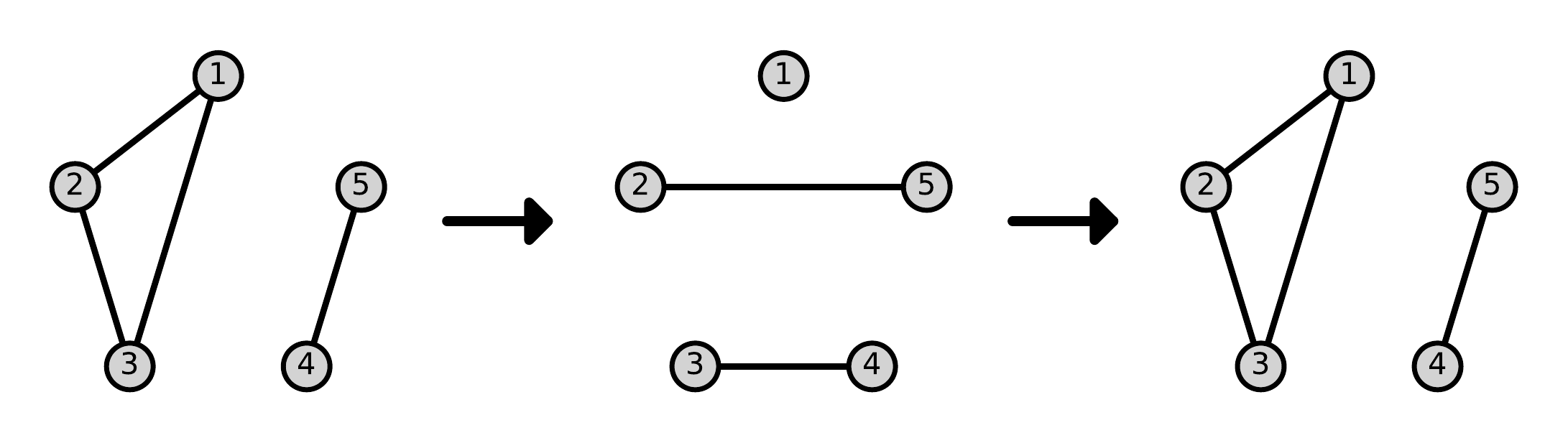}
         \vskip - 0.1 in
         \caption{$n=5$}
     \end{subfigure} \\
     \begin{subfigure}{0.36\hsize}
         \centering
         \includegraphics[width=\hsize]{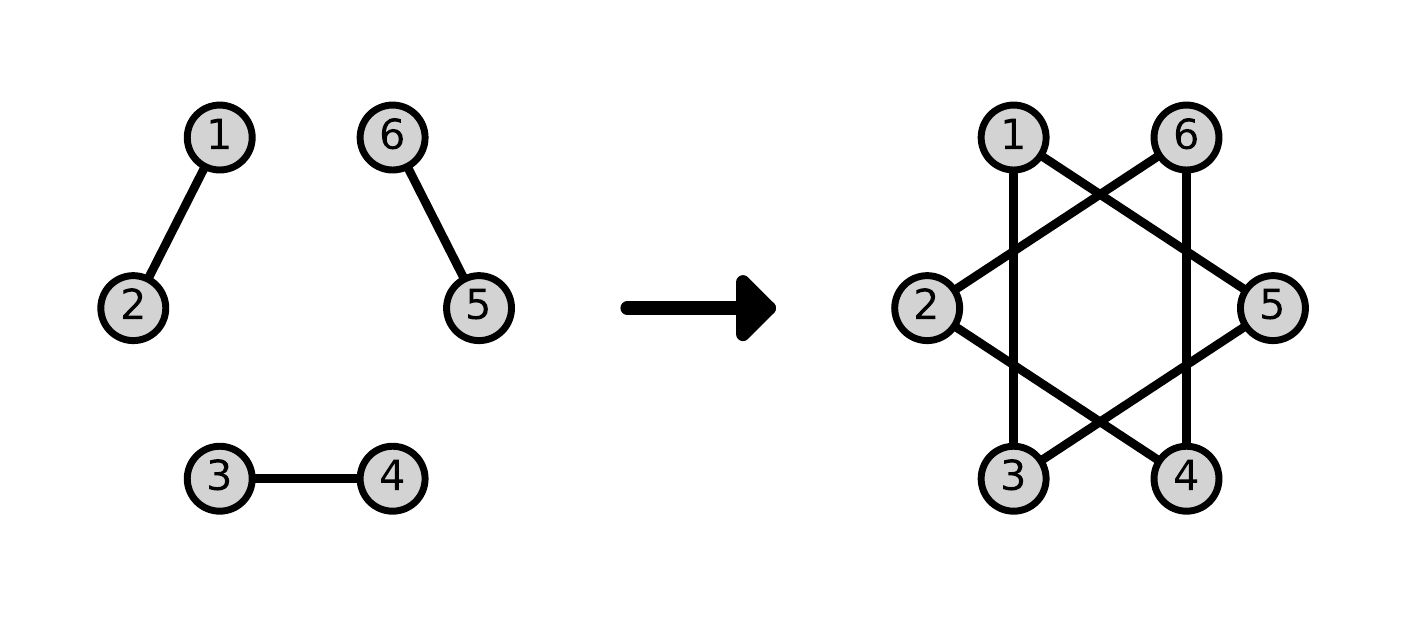}
         \vskip - 0.1 in
         \caption{$n=6$}
     \end{subfigure} \\
     \begin{subfigure}{0.72\hsize}
         \centering
         \includegraphics[width=\hsize]{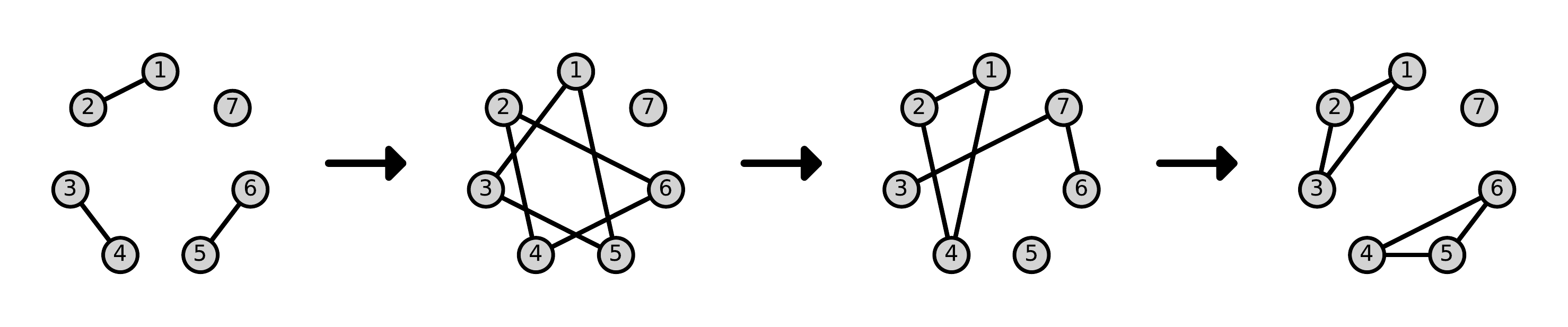}
         \vskip - 0.1 in
         \caption{$n=7$}
     \end{subfigure}
     \hfill
     \begin{subfigure}{0.54\hsize}
         \centering
         \includegraphics[width=\hsize]{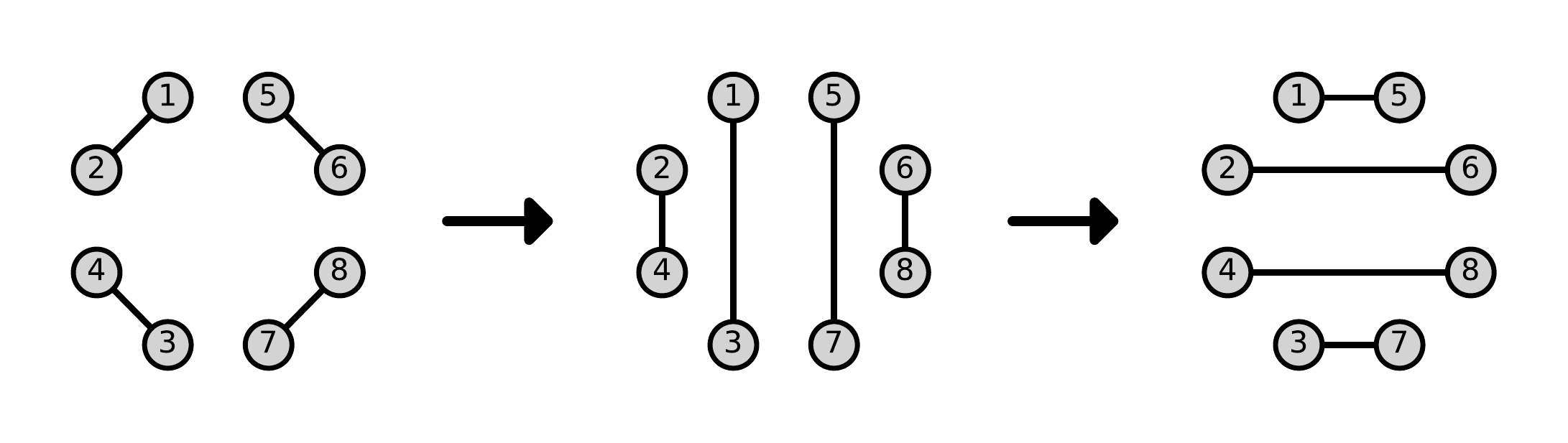}
         \vskip - 0.1 in
         \caption{$n=8$}
     \end{subfigure} \\
     \begin{subfigure}{0.36\hsize}
         \centering
         \includegraphics[width=\hsize]{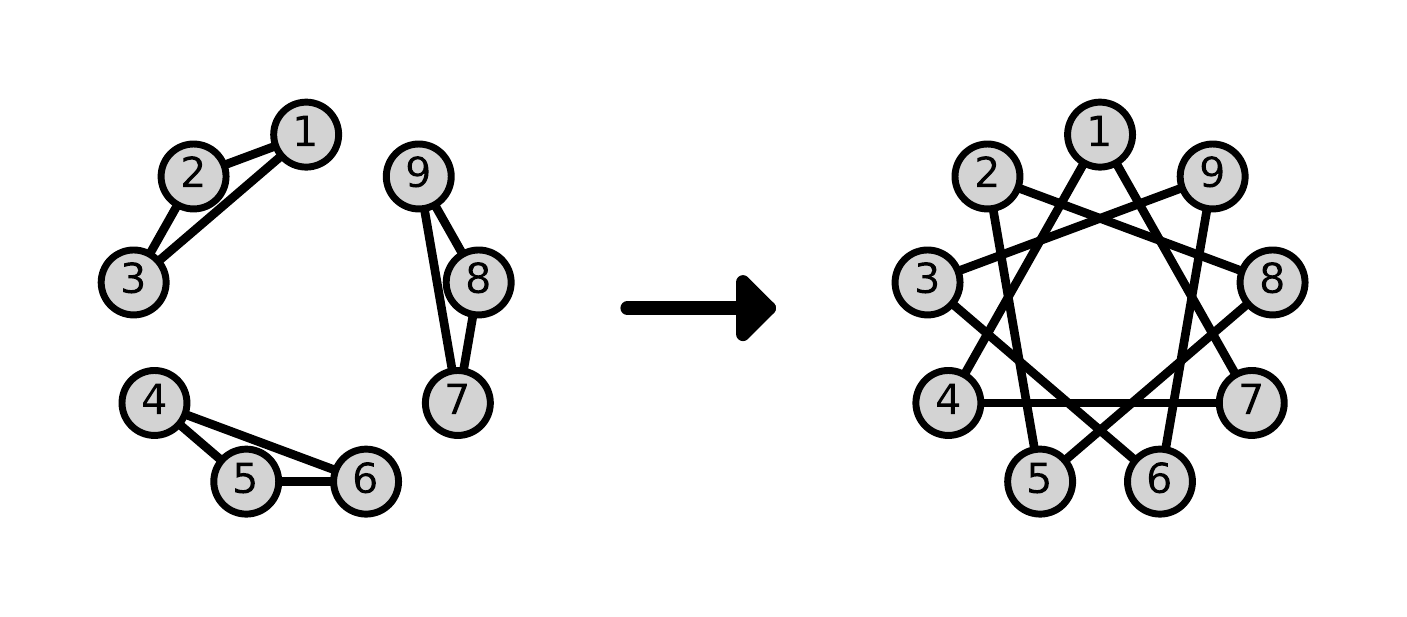}
         \vskip - 0.1 in
         \caption{$n=9$}
     \end{subfigure}
     \begin{subfigure}{0.72\hsize}
         \centering
         \includegraphics[width=\hsize]{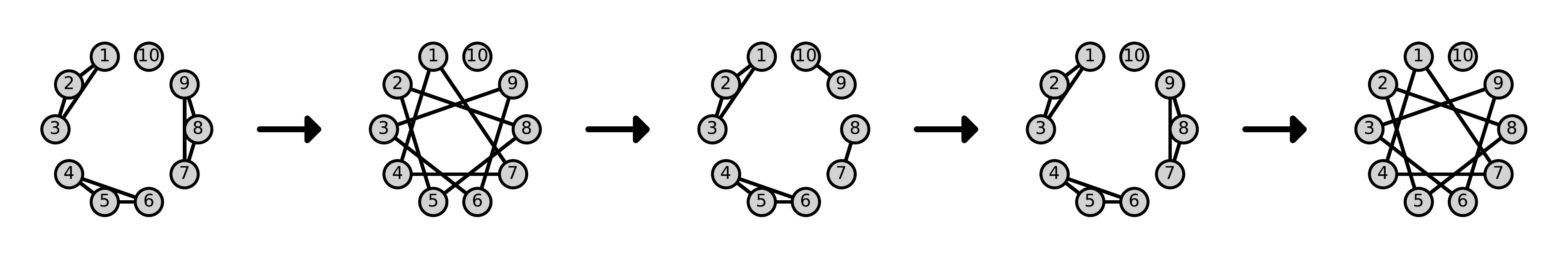}
         \vskip - 0.1 in
         \caption{$n=10$}
     \end{subfigure}
\caption{Illustration of the \simpleProposed{$3$} with the various numbers of nodes.}
\end{figure}

\newpage
\subsubsection{Base-$(k+1)$ Graph}
\label{sec:examples_of_FAIRY}
\begin{figure}[h!]
     \centering
     \begin{subfigure}{0.43\hsize}
         \centering
         \includegraphics[width=\hsize]{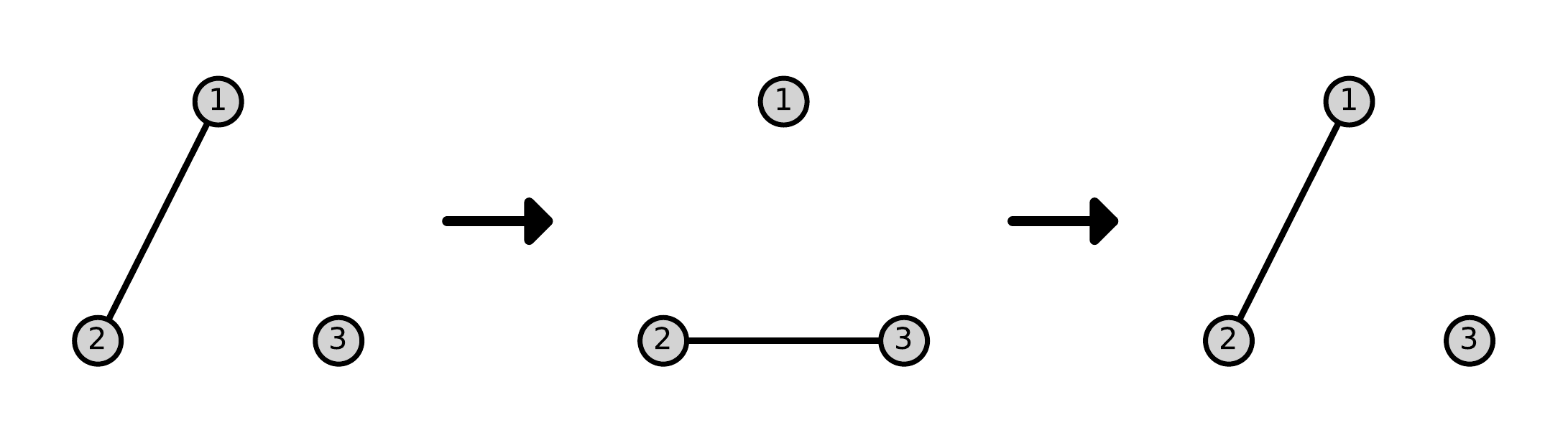}
         \vskip - 0.1 in
         \caption{$n=3$}
     \end{subfigure} \\
     \begin{subfigure}{0.28\hsize}
         \centering
         \includegraphics[width=\hsize]{pic/1_peer_ADIC_4.pdf}
         \vskip - 0.1 in
         \caption{$n=4$}
     \end{subfigure} \\
     \begin{subfigure}{0.71\hsize}
         \centering
         \includegraphics[width=\hsize]{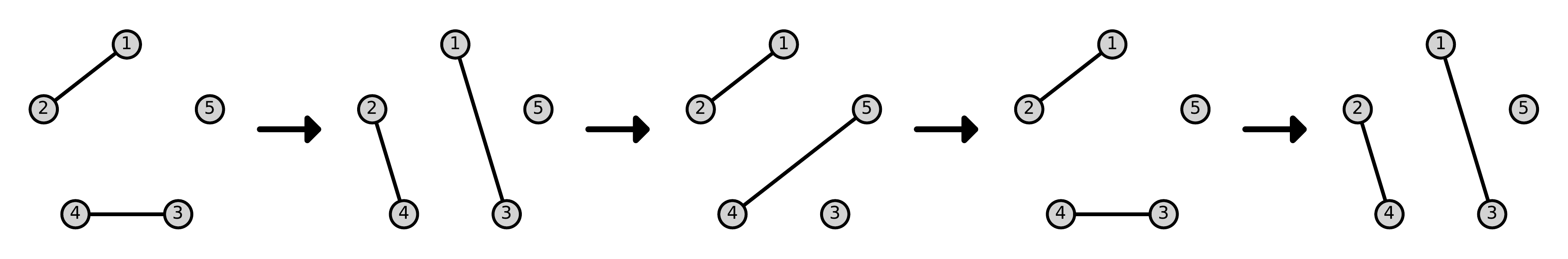}
         \vskip - 0.1 in
         \caption{$n=5$}
     \end{subfigure}
     \hfill
     \begin{subfigure}{0.57\hsize}
         \centering
         \includegraphics[width=\hsize]{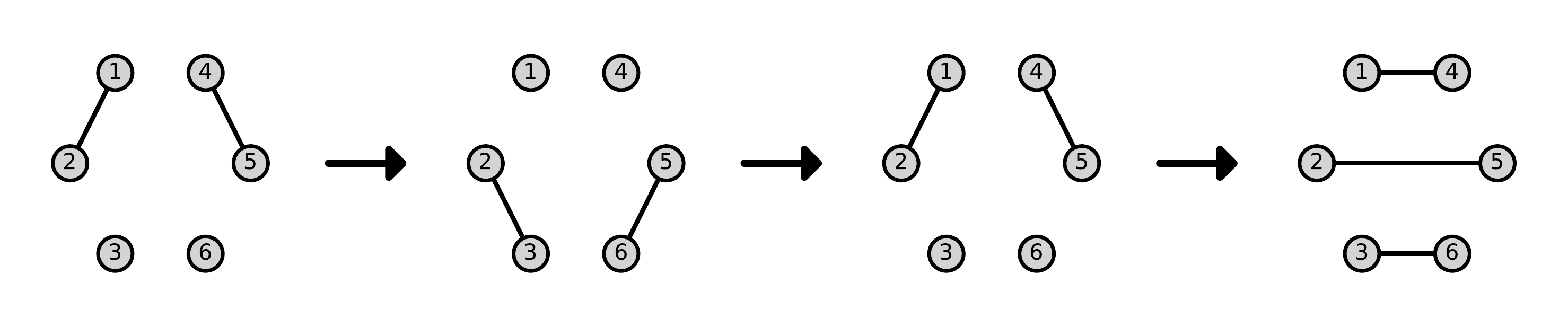}
         \vskip - 0.1 in
         \caption{$n=6$}
     \end{subfigure}
     \hfill
     \begin{subfigure}{0.71\hsize}
         \centering
         \includegraphics[width=\hsize]{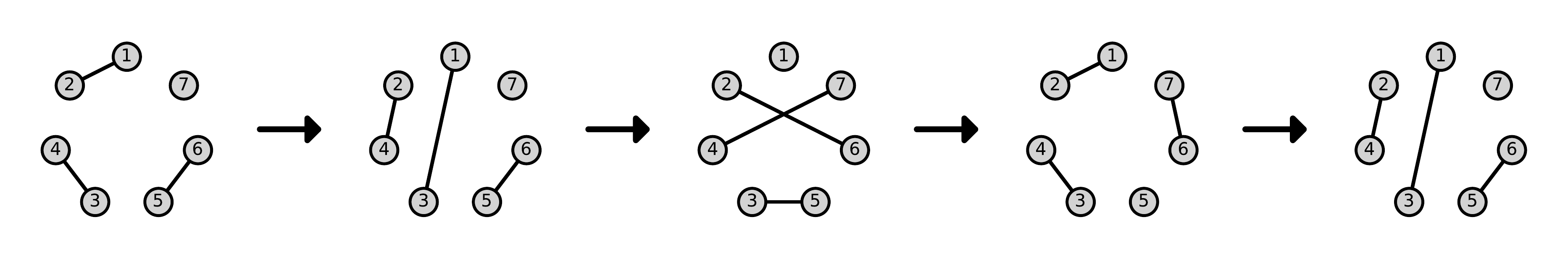}
         \vskip - 0.1 in
         \caption{$n=7$}
     \end{subfigure}
     \hfill
     \begin{subfigure}{0.43\hsize}
         \centering
         \includegraphics[width=\hsize]{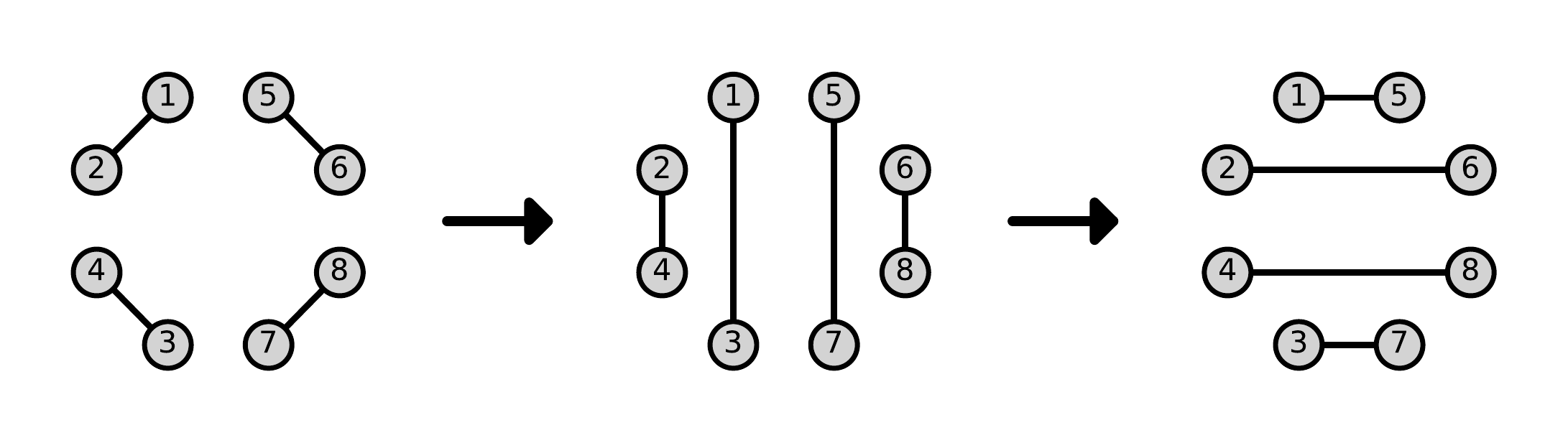}
         \vskip - 0.1 in
         \caption{$n=8$}
     \end{subfigure}
     \begin{subfigure}{\hsize}
         \centering
         \includegraphics[width=\hsize]{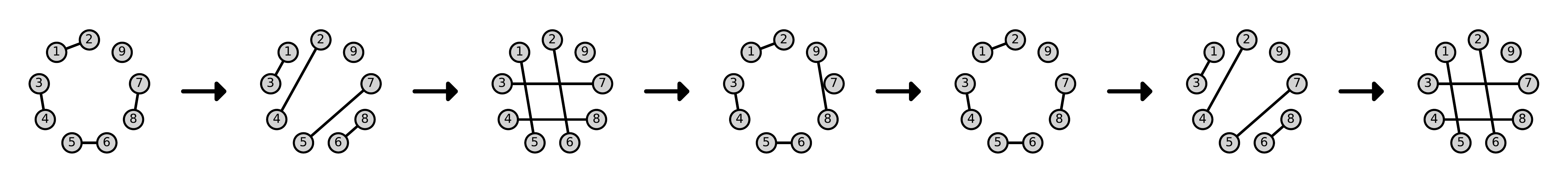}
         \vskip - 0.1 in
         \caption{$n=9$}
     \end{subfigure}
     \begin{subfigure}{1.0\hsize}
         \centering
         \includegraphics[width=0.86 \hsize]{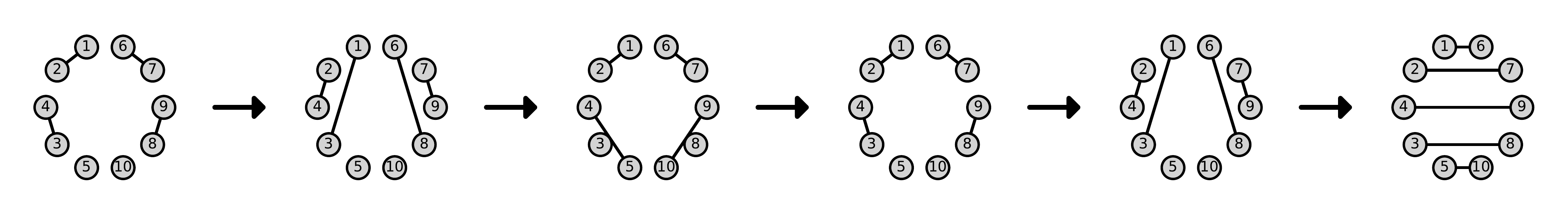}
         \vskip - 0.1 in
         \caption{$n=10$}
     \end{subfigure}
\caption{Illustration of the \proposed{$2$} with the various numbers of nodes.}
\label{fig:1_peer_FAIRY}
\end{figure}

\newpage
\begin{figure}[h!]
     \centering
     \begin{subfigure}{0.18\hsize}
         \centering
         \includegraphics[width=\hsize]{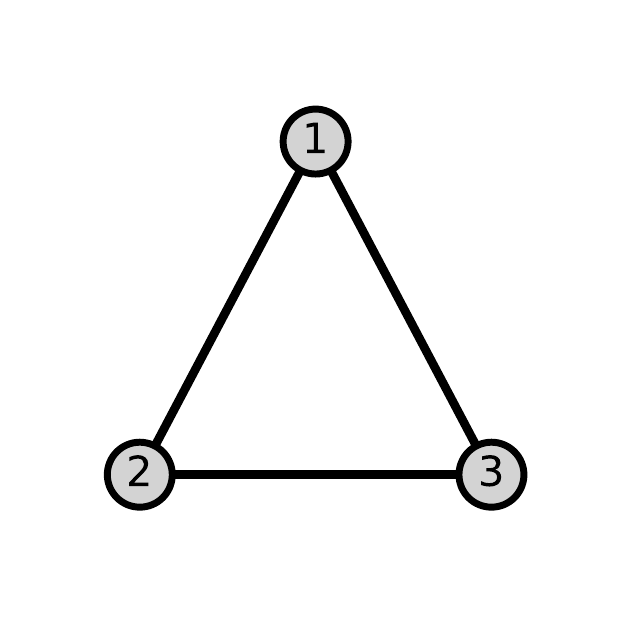}
         \vskip - 0.1 in
         \caption{$n=3$}
     \end{subfigure} \\
     \begin{subfigure}{0.36\hsize}
         \centering
         \includegraphics[width=\hsize]{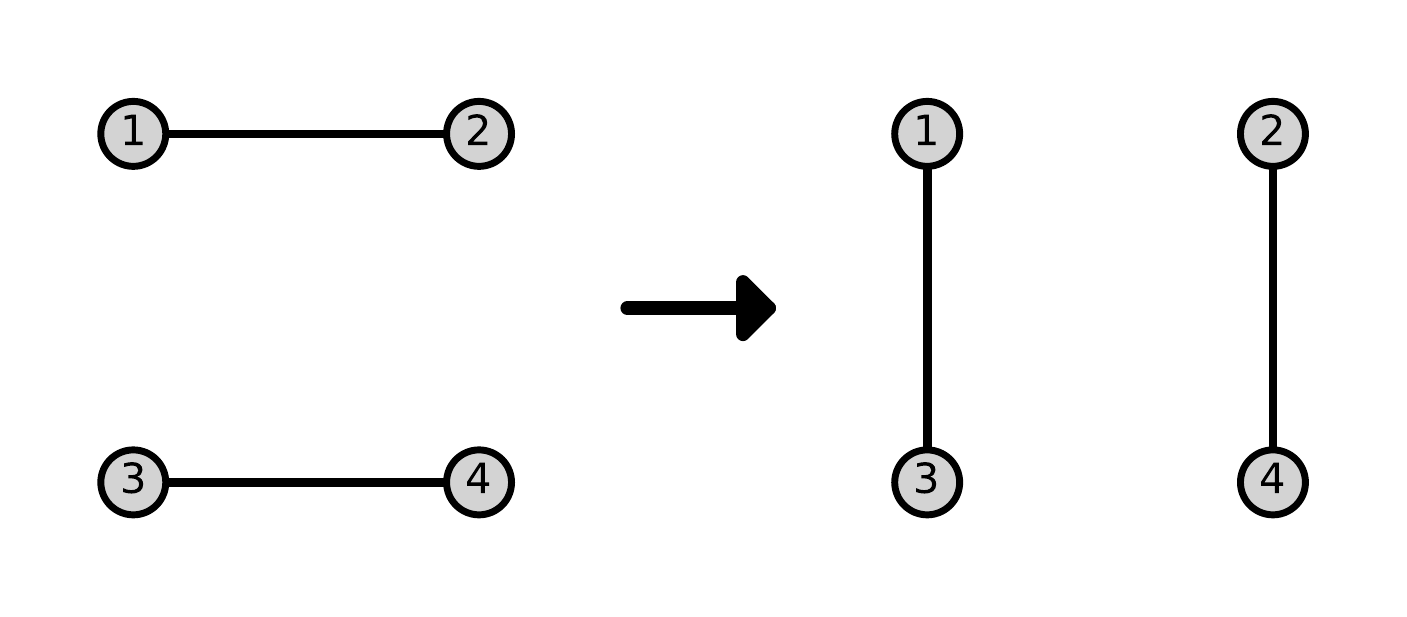}
         \vskip - 0.1 in
         \caption{$n=4$}
     \end{subfigure} \\
     \begin{subfigure}{0.54\hsize}
         \centering
         \includegraphics[width=\hsize]{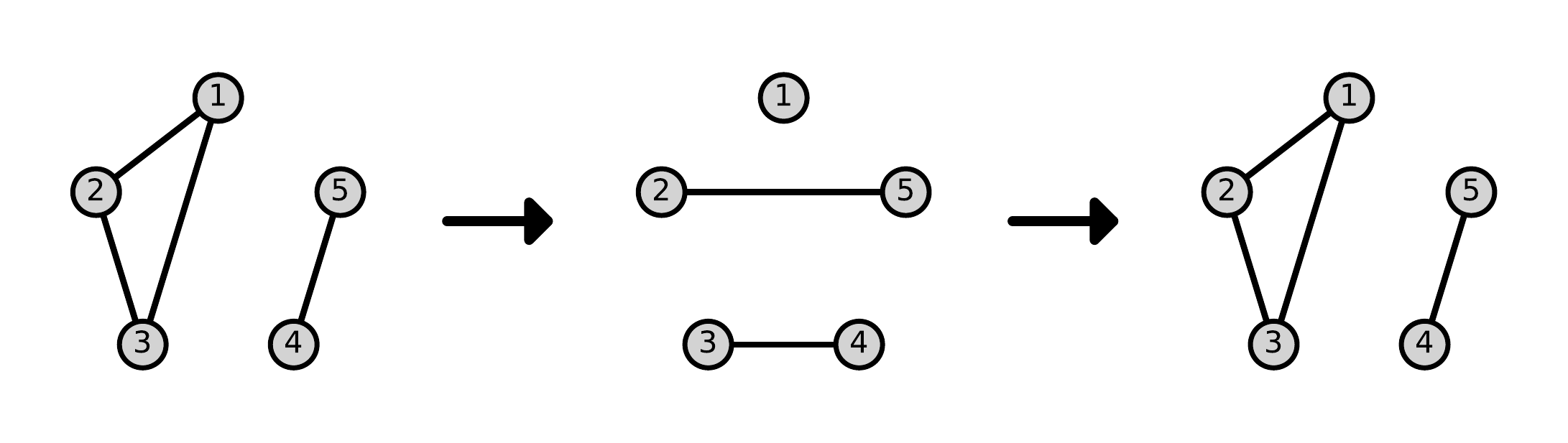}
         \vskip - 0.1 in
         \caption{$n=5$}
     \end{subfigure} \\
     \begin{subfigure}{0.36\hsize}
         \centering
         \includegraphics[width=\hsize]{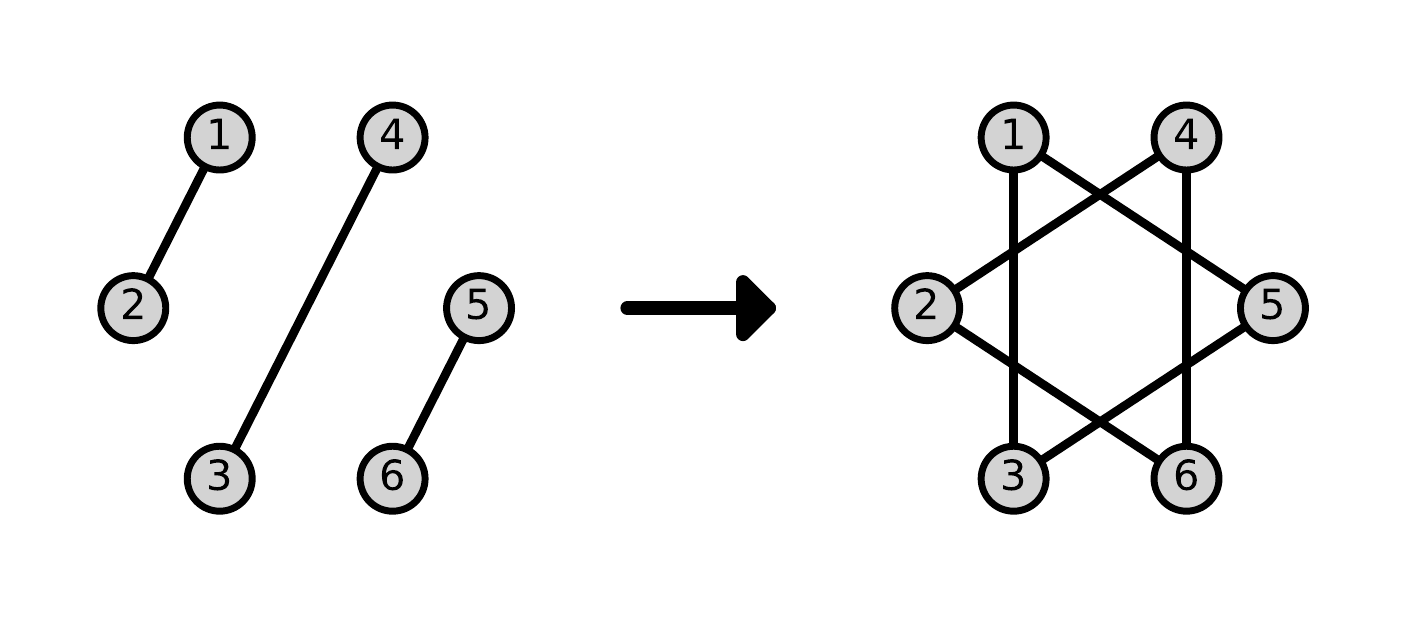}
         \vskip - 0.1 in
         \caption{$n=6$}
     \end{subfigure} \\
     \begin{subfigure}{0.72\hsize}
         \centering
         \includegraphics[width=\hsize]{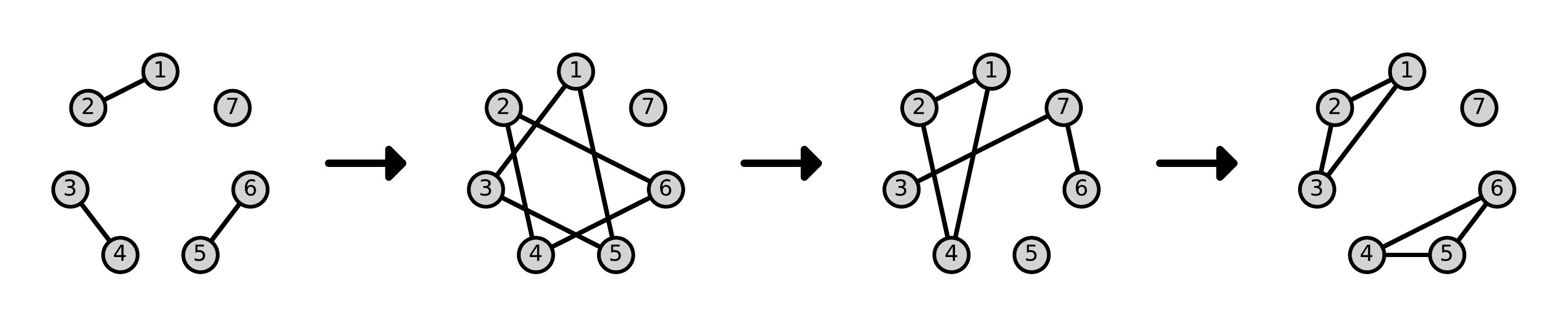}
         \vskip - 0.1 in
         \caption{$n=7$}
     \end{subfigure} \\
     \begin{subfigure}{0.54\hsize}
         \centering
         \includegraphics[width=\hsize]{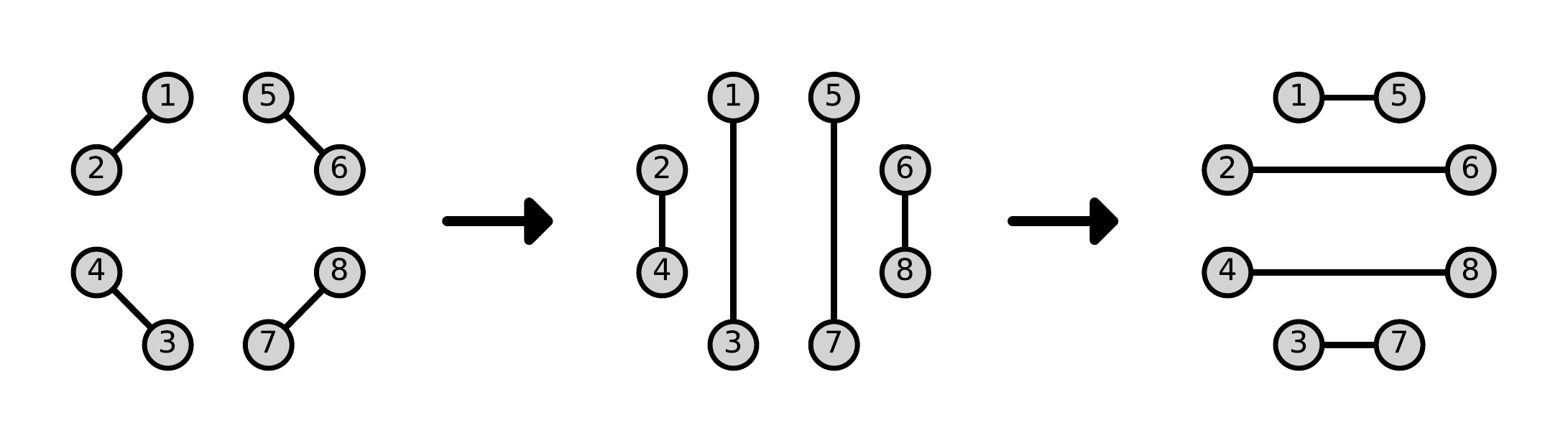}
         \vskip - 0.1 in
         \caption{$n=8$}
     \end{subfigure} \\
     \begin{subfigure}{0.36\hsize}
         \centering
         \includegraphics[width=\hsize]{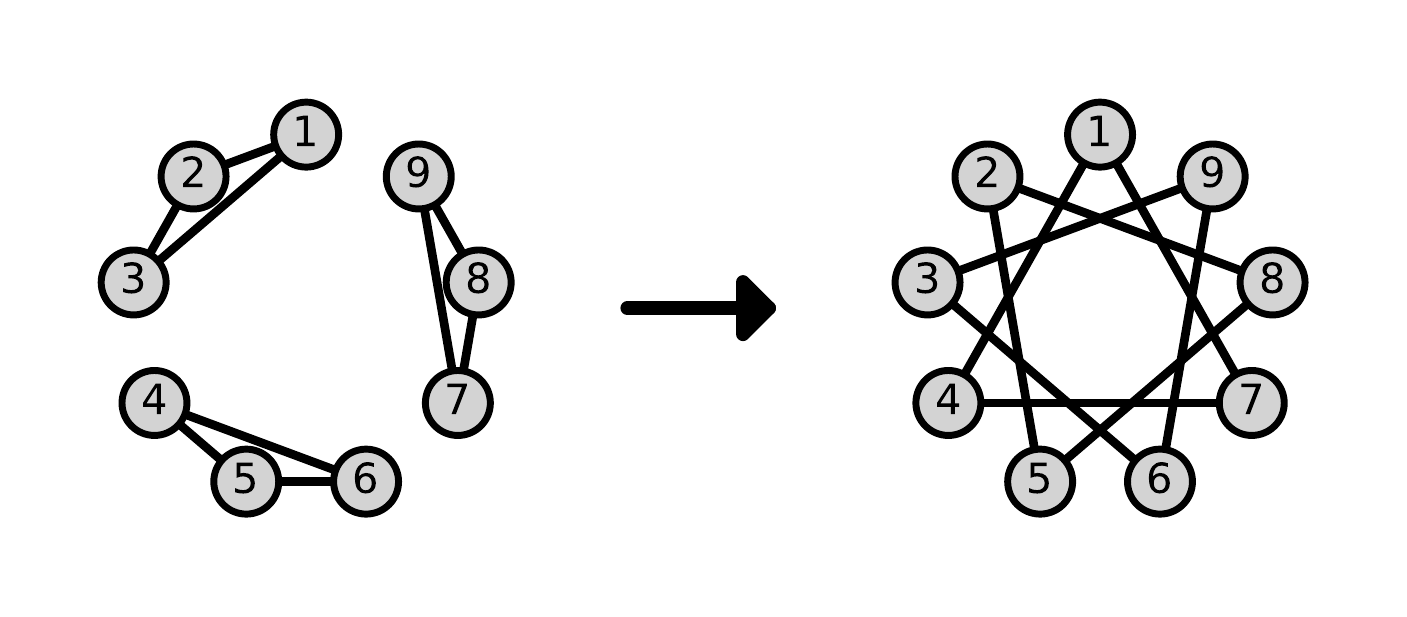}
         \vskip - 0.1 in
         \caption{$n=9$}
     \end{subfigure}
     \begin{subfigure}{0.72\hsize}
         \centering
         \includegraphics[width=\hsize]{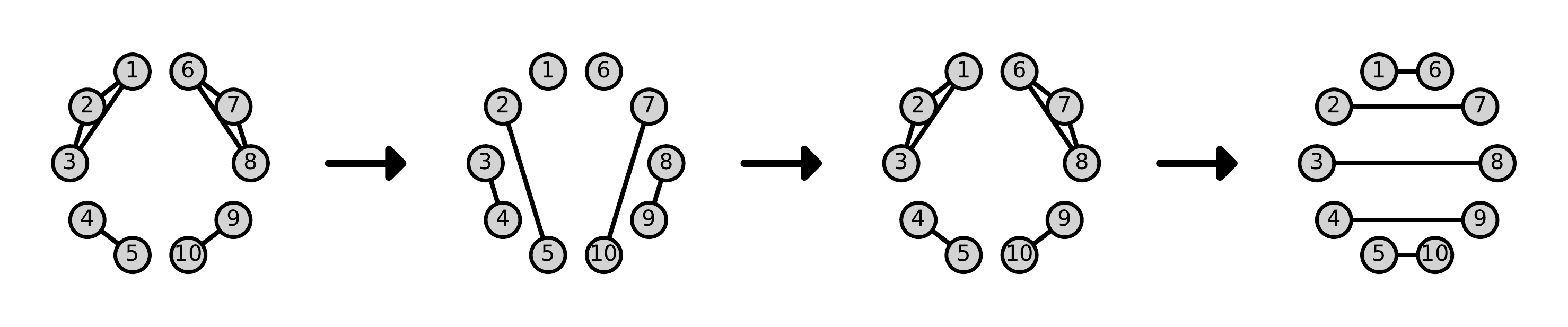}
         \vskip - 0.1 in
         \caption{$n=10$}
     \end{subfigure}
\caption{Illustration of the \proposed{$3$} with the various numbers of nodes.}
\label{fig:2_peer_FAIRY}
\end{figure}

\newpage
\subsection{$1$-peer Hypercube Graph and $1$-peer Exponential Graph}
\label{sec:1_peer_exp_and_1_peer_hypercube}
For completeness, we provide examples of the $1$-peer hypercube \cite{shi2016finite} and $1$-peer exponential graphs \cite{ying2021exponential} in 
Figs.~\ref{fig:1_peer_exp} and \ref{fig:1_peer_hypeprcube}, respectively.
\begin{figure}[h!]
     \centering
     \begin{subfigure}{0.32\hsize}
         \centering
         \includegraphics[width=\hsize]{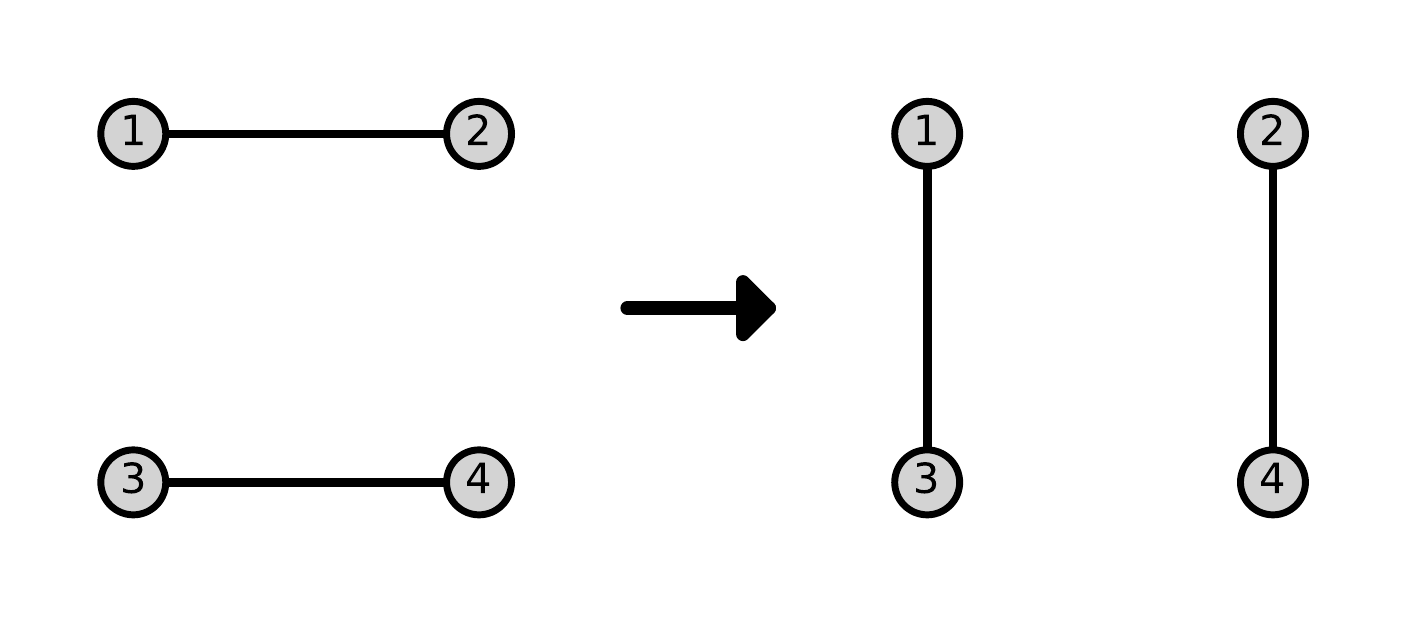}
         \vskip - 0.1 in
         \caption{$n=4$}
     \end{subfigure} 
     \hfill
     \begin{subfigure}{0.475\hsize}
         \centering
         \includegraphics[width=\hsize]{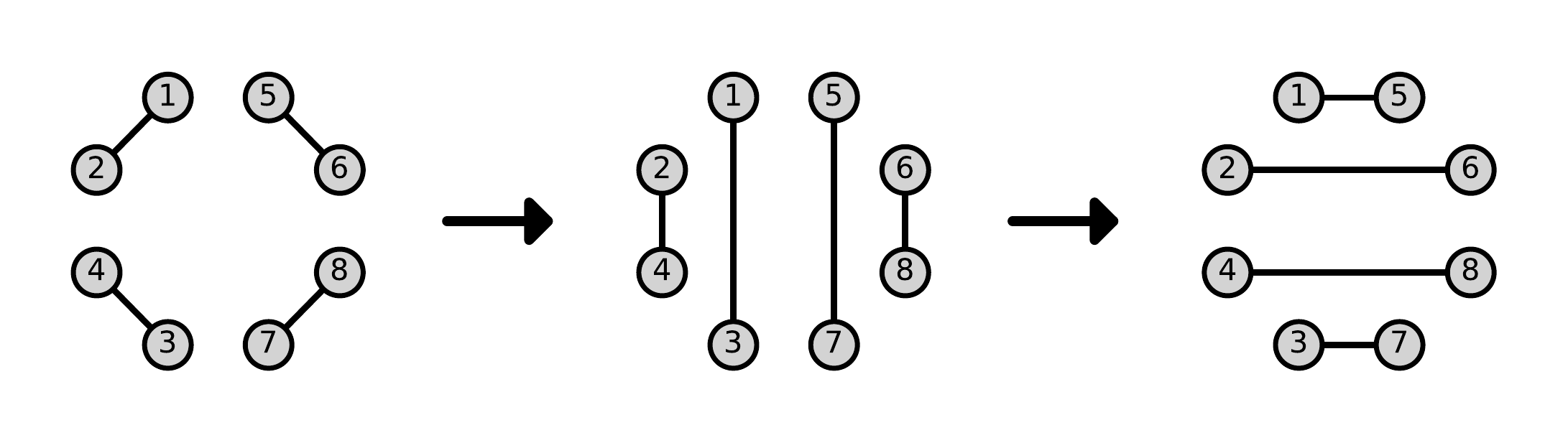}
         \vskip - 0.1 in
         \caption{$n=8$}
     \end{subfigure} 
\caption{Illustration of the $1$-peer hypercube graph. All edge weights are $0.5$.}
\label{fig:1_peer_hypeprcube}
\end{figure}
\begin{figure}[h!]
     \centering
     \begin{subfigure}{0.32\hsize}
         \centering
         \includegraphics[width=\hsize]{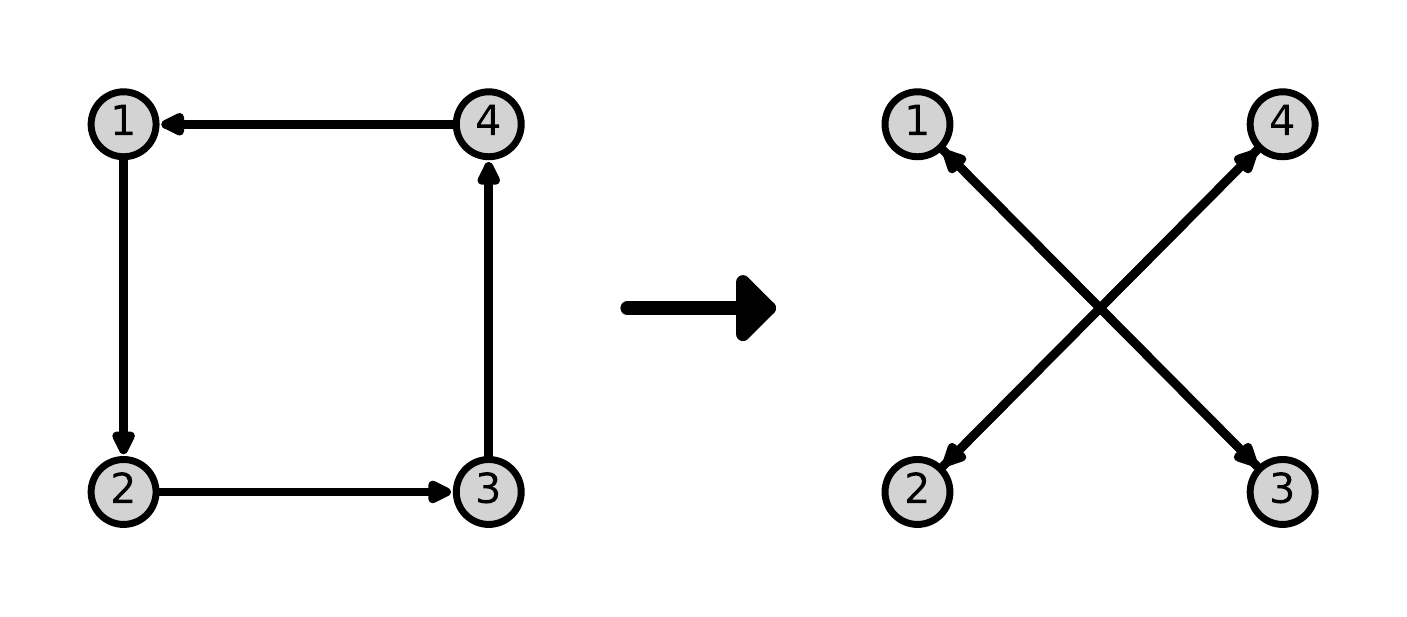}
         \vskip - 0.1 in
         \caption{$n=4$}
     \end{subfigure} 
     \hfill
     \begin{subfigure}{0.475\hsize}
         \centering
         \includegraphics[width=\hsize]{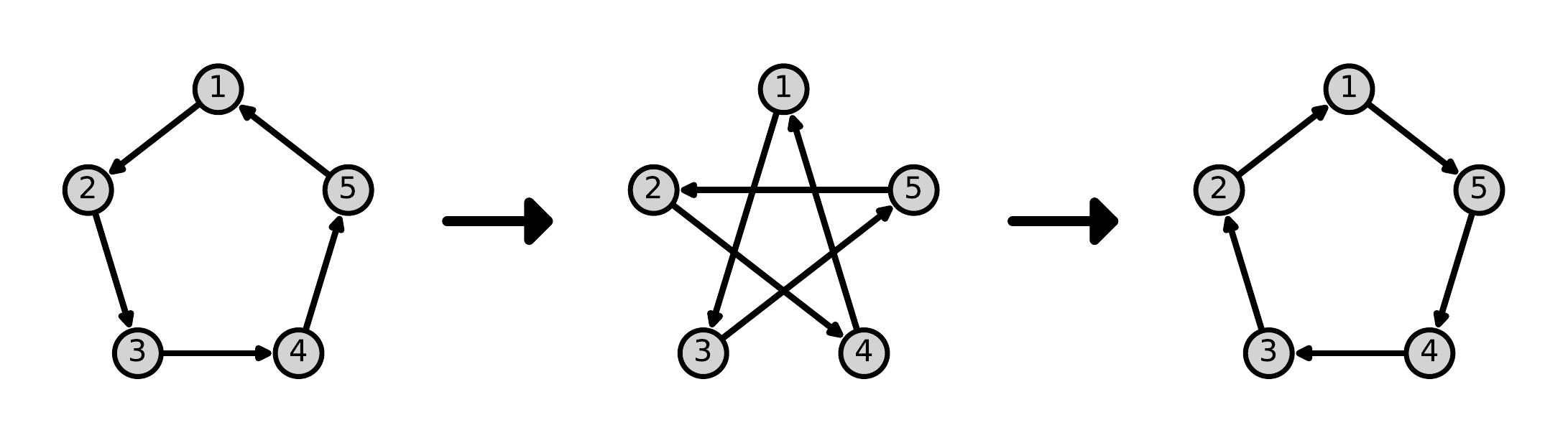}
         \vskip - 0.1 in
         \caption{$n=5$}
     \end{subfigure} 
     \begin{subfigure}{0.475\hsize}
         \centering
         \includegraphics[width=\hsize]{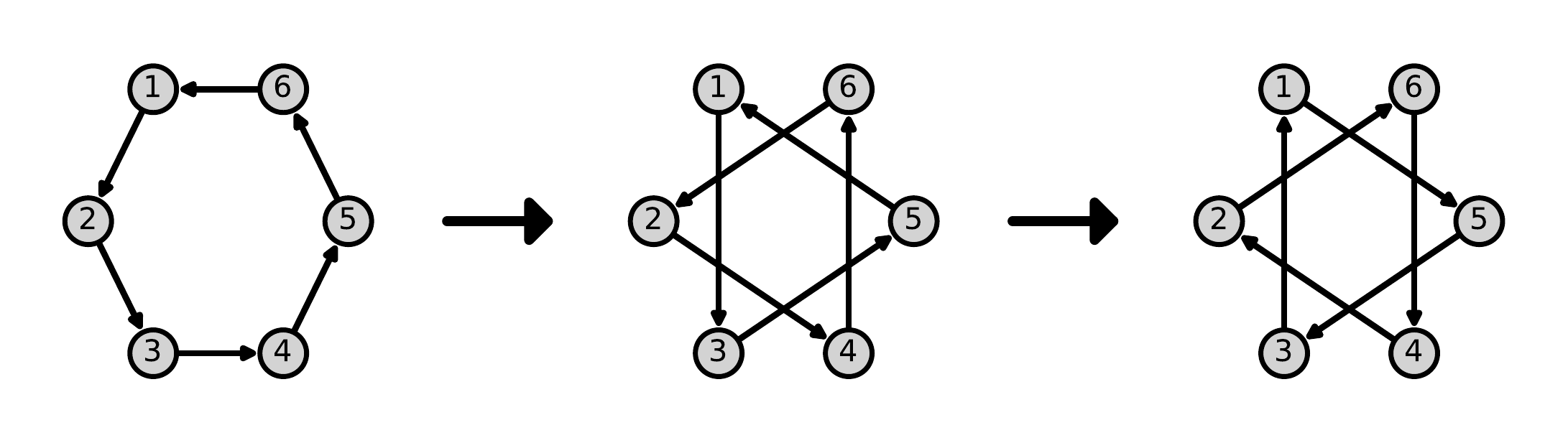}
         \vskip - 0.1 in
         \caption{$n=6$}
     \end{subfigure} 
     \hfill
     \begin{subfigure}{0.475\hsize}
         \centering
         \includegraphics[width=\hsize]{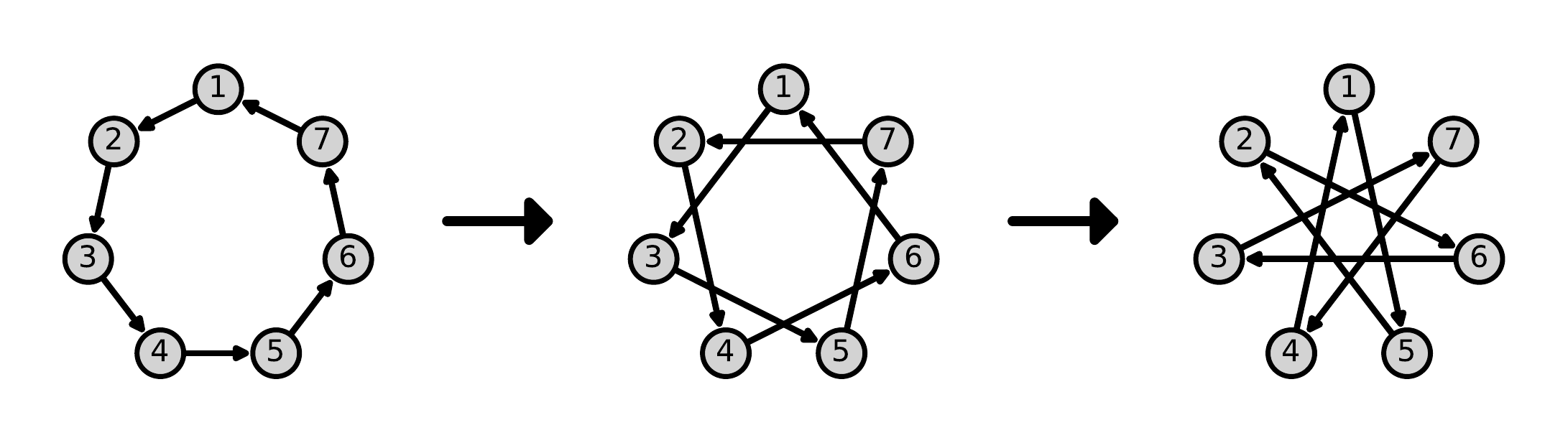}
         \vskip - 0.1 in
         \caption{$n=7$}
     \end{subfigure} 
     \begin{subfigure}{0.475\hsize}
         \centering
         \includegraphics[width=\hsize]{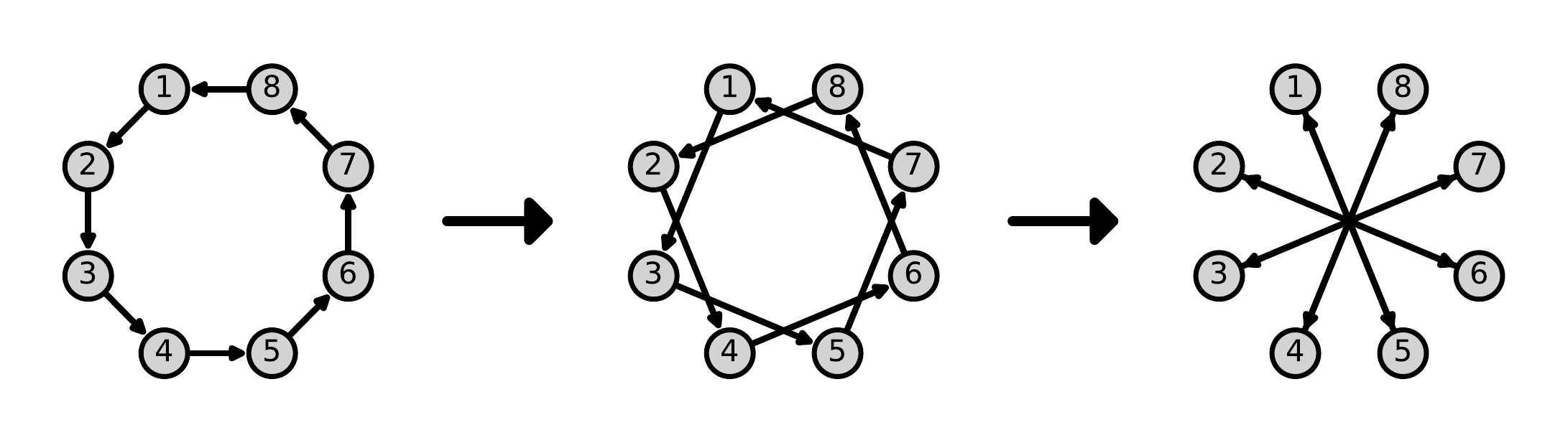}
         \vskip - 0.1 in
         \caption{$n=8$}
     \end{subfigure} 
\caption{Illustration of the $1$-peer exponential graph. All edge weights are $0.5$.}
\label{fig:1_peer_exp}
\end{figure}

\newpage
\section{Proof of Theorem \ref{theorem:length}}
\label{sec:proof_of_length}

\begin{lemma}[Length of \hyperhypercube{$k$}]
\label{lemma:hyperhypercube}
Suppose that all prime factors of the number of nodes $n$ are less than or equal to $k+1$.
Then, for any number of nodes $n \in \mathbb{N}$ and maximum degree $k \in [n-1]$,
the length of the \hyperhypercube{$k$} is less than or equal to $\max\{ 1, 2 \log_{k+2} (n)\}$.
\end{lemma}
\begin{proof}
We assume that $n$ is decomposed as $n = n_1 \times \cdots \times n_L$ with minimum $L$ where $n_l \in [k+1]$ for all $l \in [L]$.
Without loss of generality, we suppose $n_1 \leq n_2 \leq \cdots \leq n_L$.
Then, for any $i \not = j$, it holds that $n_i \times n_j \geq k+2$
because if $n_i \times n_j \leq k+1$ for some $i$ and $j$, this contradicts the assumption that $L$ is minimum. 

When $L$ is even, we have
\begin{align*}
    n = (n_1 \times n_2) \times \cdots \times (n_{L-1} \times n_L) \geq (k+2)^{\frac{L}{2}}.
\end{align*}
Then, we get $L \leq 2 \log_{k+2} (n)$.

Next, we discuss the case when $L$ is odd.
When $L \geq 3$, $n_L \geq \sqrt{k+2}$ holds because $n_{L-2} \times n_{L-1} \geq k+2$.
Thus, we get
\begin{align*}
    n = (n_1 \times n_2) \times \cdots \times (n_{L-2} \times n_{L-1}) \times n_L 
    \geq (k+2)^{\frac{L-1}{2}} \times n_{L}
    \geq (k+2)^{\frac{L}{2}}.
\end{align*}
Then, we get $L \leq 2 \log_{k+2} (n)$ when $L \geq 3$.

Thus, given the case when $L=1$, the length of the \hyperhypercube{$k$} is less than or equal to $\max\{ 1, 2 \log_{k+2} (n)\}$.
\end{proof}

\begin{lemma}[Length of \simpleProposed{$(k+1)$}]
\label{lemma:length_of_simple_FAIRY}
For any number of nodes $n \in \mathbb{N}$ and maximum degree $k \in [n-1]$,
the length of the \simpleProposed{$(k+1)$} is less than or equal to $2 \log_{k+1} (n) + 2$.
\end{lemma}
\begin{proof}
When all prime factors of $n$ are less than or equal to $k+1$,
the \simpleProposed{$(k+1)$} is equivalent to the \hyperhypercube{$k$}
and the statement holds from Lemma \ref{lemma:hyperhypercube}.
In the following, we consider the case when there exists a prime factor of $n$ that is larger than $k+1$.
Note that because when $L=1$ (i.e., $n=a_1 \times (k+1)^{p_1}$), all prime factors of $n$ are less than or equal to $k+1$,
we only need to consider the case when $L \geq 2$.
We have the following inequality:
\begin{align*}
    \log_{k+1} (n) 
    &= \log_{k+1} (a_1 (k+1)^{p_1} + \cdots + a_L (k+1)^{p_L}) \\
    &\geq p_1 + \log_{k+1} (a_1) \\
    &\geq p_1.
\end{align*}
Then, because $|V_1| = a_1 \times (k+1)^{p_1}$, it holds that $m_1 = |\mathcal{H}_k (V_1)| \leq 1 + p_1 \leq \log_{k+1} (n) + 1$.
Similarly, it holds that $|\mathcal{H}_k (V_{1,1})| = p_1 \leq \log_{k+1} (n)$ because $|V_{1,1}| = (k+1)^{p_1}$.
In Alg.~\ref{alg:simple_k_peer_FAIRY}, the update rule $b_1 \leftarrow b_1 + 1$ in line 22 is executed for the first time when $m = m_1 + 2$ because $L \geq 2$.
Thus, the length of the \simpleProposed{$(k+1)$} is at most $m_1 + |\mathcal{H}_k (V_{1,1})| + 1 \leq 2 \log_{k+1} (n) + 2$.
This concludes the statement.
\end{proof}

\begin{lemma}[Length of \proposed{$(k+1)$}]
For any number of nodes $n \in \mathbb{N}$ and maximum degree $k \in [n-1]$,
the length of the \proposed{$(k+1)$} is less than or equal to $2 \log_{k+1} (n) + 2$. 
\end{lemma}
\begin{proof}
The statement follows immediately from Lemma \ref{lemma:length_of_simple_FAIRY} and line $12$ in Alg.~\ref{alg:k_peer_FAIRY}.
\end{proof}

\newpage
\section{Convergence Rate of DSGD over Various Topologies}
\label{sec:convergence_rate_of_various_topologies}
Table \ref{table:convergence_rate_of_various_topologies} lists the convergence rates of DSGD over various topologies.
These convergence rates can be immediately obtained from Theorem 2 stated in \citet{koloskova2020unified} and consensus rate of the topology.
As seen from Table \ref{table:convergence_rate_of_various_topologies},
the \proposed{$2$} enables DSGD to converge faster than the ring and torus and as fast as the exponential graph for any number of nodes, although the maximum degree of the \proposed{$2$} is only one.
Moreover, for any number of nodes, 
the \proposed{$(k+1)$} with $2 \leq k < \ceil{\log_2 (n)}$ enables DSGD to converge faster than the exponential graph,
even though the maximum degree of the \proposed{$(k+1)$} remains to be less than that of the exponential graph.

\begin{table}[h!]
\centering
\caption{Convergence rates and maximum degrees of DSGD over various topologies.}
\label{table:convergence_rate_of_various_topologies}
\resizebox{\linewidth}{!}{
\begin{tabular}{lccc}
\toprule
\textbf{Topology}    &  \textbf{Convergence Rate} & \textbf{Maximum Degree} & \textbf{\#Nodes $n$} \\
\midrule
Ring \cite{nedic2018network}
                    & $\mathcal{O} \left( \dfrac{\sigma^2}{n \epsilon^2}
                        + \dfrac{\zeta n^2 + \sigma n}{\epsilon^{3/2}}
                        + \dfrac{n^2}{\epsilon} \right) \cdot L F_0$ & $2$ & $\forall n \in \mathbb{N}$ \\
Torus \cite{nedic2018network}              
                    & $\mathcal{O} \left( \dfrac{\sigma^2}{n \epsilon^2}
                        + \dfrac{\zeta n + \sigma \sqrt{n}}{\epsilon^{3/2}}
                        + \dfrac{n}{\epsilon} \right) \cdot L F_0$ & $4$ & $\forall n \in \mathbb{N}$ \\
Exp. \cite{ying2021exponential}                
                    & $\mathcal{O} \left( \dfrac{\sigma^2}{n \epsilon^2}
                        + \dfrac{\zeta \log_{2} (n) + \sigma \sqrt{\log_{2} (n)}}{\epsilon^{3/2}}
                        + \dfrac{\log_{2} (n)}{\epsilon} \right) \cdot L F_0$ & $\ceil{\log_2 (n)}$ & $\forall n \in \mathbb{N}$ \\
$1$-peer Exp. \cite{ying2021exponential}      
                    & $\mathcal{O} \left( \dfrac{\sigma^2}{n \epsilon^2}
                        + \dfrac{\zeta \log_{2} (n) + \sigma \sqrt{\log_{2} (n)}}{\epsilon^{3/2}}
                        + \dfrac{\log_{2} (n)}{\epsilon} \right) \cdot L F_0$ & $1$ & A power of $2$ \\
$1$-peer Hypercube \cite{shi2016finite}
                    & $\mathcal{O} \left( \dfrac{\sigma^2}{n \epsilon^2}
                        + \dfrac{\zeta \log_{2} (n) + \sigma \sqrt{\log_{2} (n)}}{\epsilon^{3/2}}
                        + \dfrac{\log_{2} (n)}{\epsilon} \right) \cdot L F_0$ & $1$ & A power of $2$ \\
\textbf{Base-$(k+1)$ Graph (ours)}
                    & $\mathcal{O} \left( \dfrac{\sigma^2}{n \epsilon^2}
                        + \dfrac{\zeta \log_{k+1} (n) + \sigma \sqrt{\log_{k+1} (n)}}{\epsilon^{3/2}}
                        + \dfrac{\log_{k+1} (n)}{\epsilon} \right) \cdot L F_0$ & $k$ & $\forall n \in \mathbb{N}$ \\
\bottomrule
\end{tabular}}
\end{table}

\newpage
\section{Additional Experiments}

\subsection{Comparison of Base-$(k+1)$ and Simple Base-$(k+1)$ Graphs}
\label{sec:comparison_of_FAIRY_and_simple_FAIRY}

Fig.~\ref{fig:length} shows the length of the \simpleProposed{$(k+1)$} and \proposed{$(k+1)$}.
The results indicate that for all $k$,
the length of the \proposed{$(k+1)$} is less than the length of the \simpleProposed{$(k+1)$} in many cases.

\begin{figure}[H]
    \centering
    \vskip - 0.1 in
    \begin{subfigure}{0.32\hsize}
         \centering
         \includegraphics[width=\hsize]{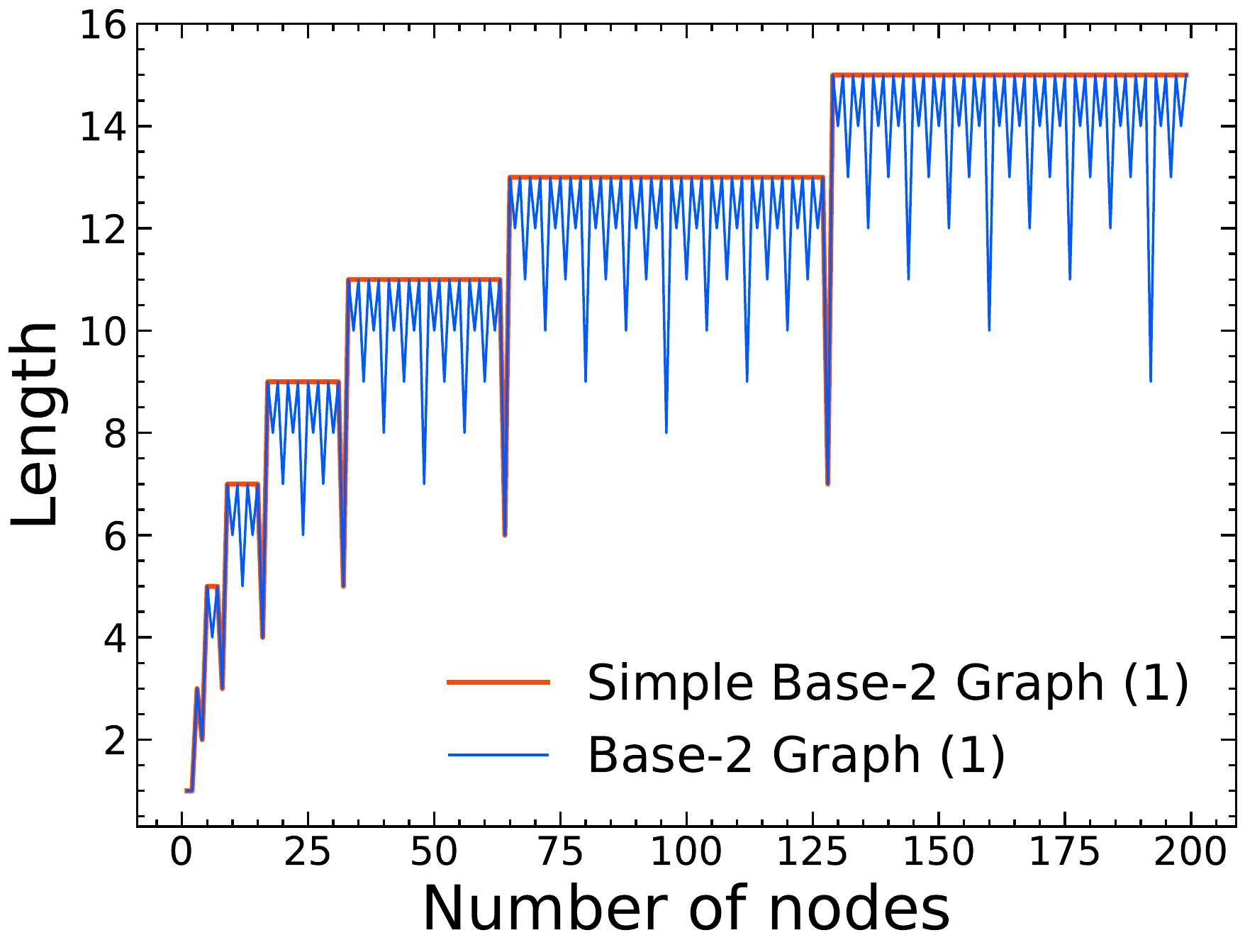}
     \end{subfigure}
    \begin{subfigure}{0.32\hsize}
         \centering
         \includegraphics[width=\hsize]{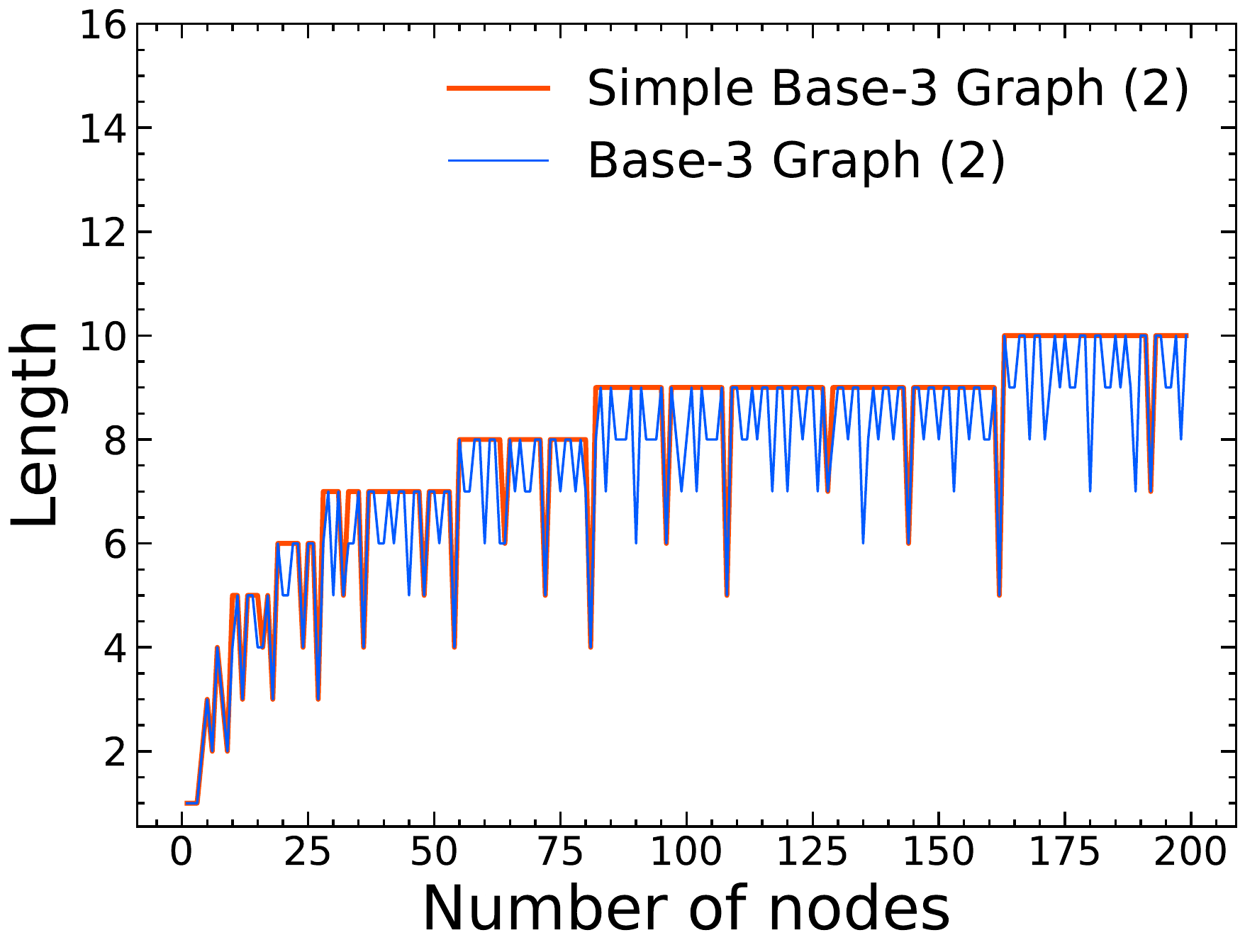}
     \end{subfigure}
    \begin{subfigure}{0.32\hsize}
         \centering
         \includegraphics[width=\hsize]{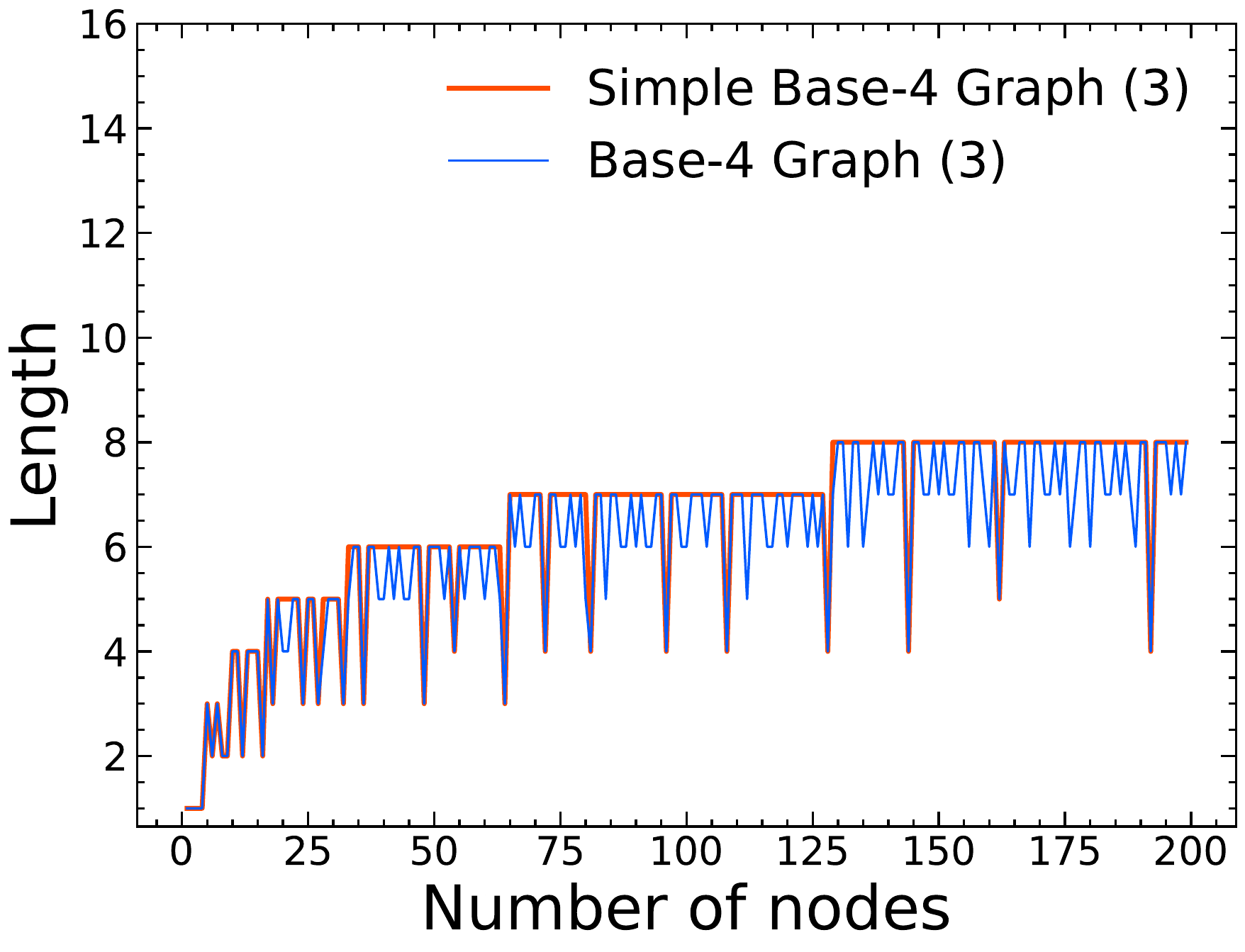}
     \end{subfigure}     
    \begin{subfigure}{0.32\hsize}
         \centering
         \includegraphics[width=\hsize]{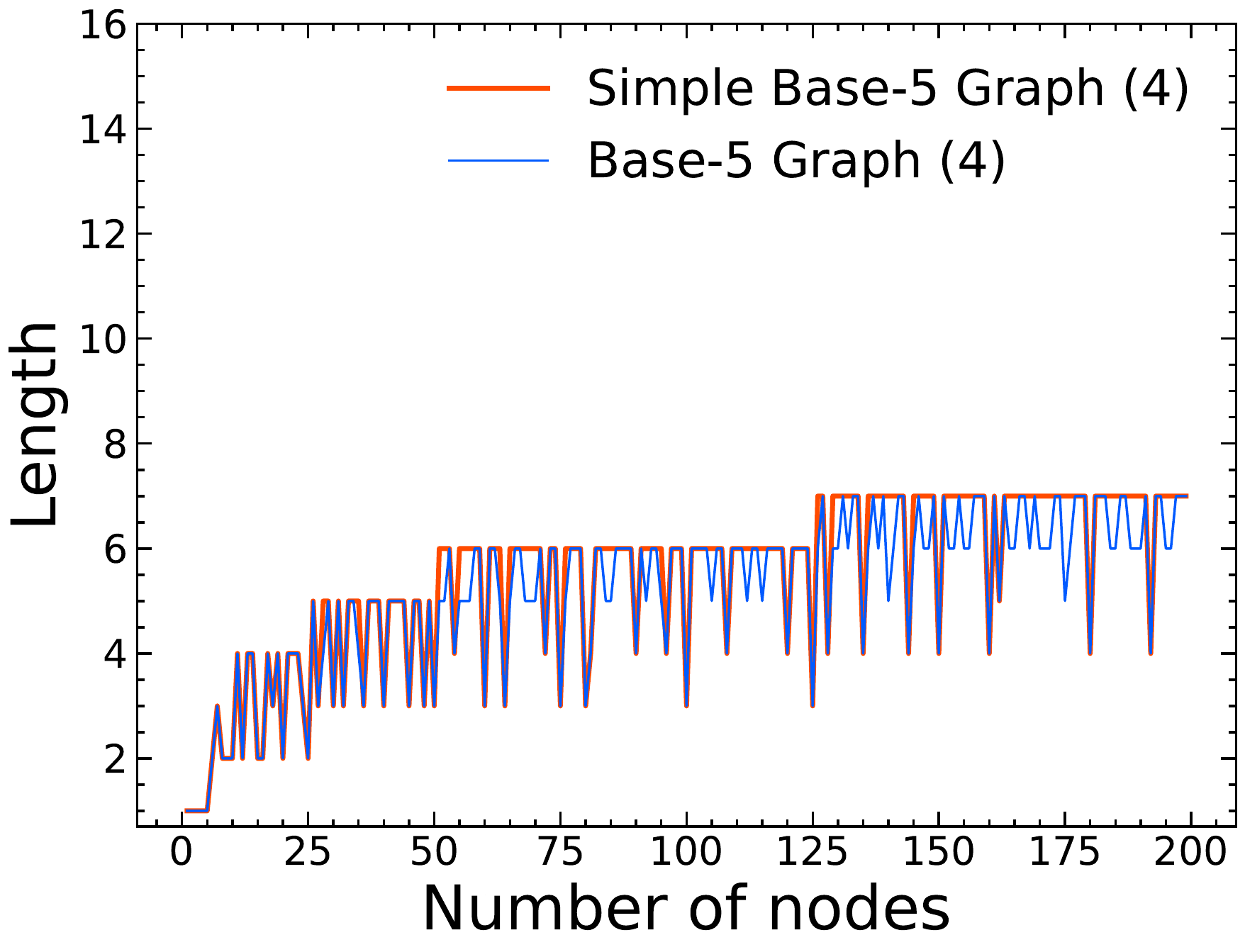}
     \end{subfigure}
    \begin{subfigure}{0.32\hsize}
         \centering
         \includegraphics[width=\hsize]{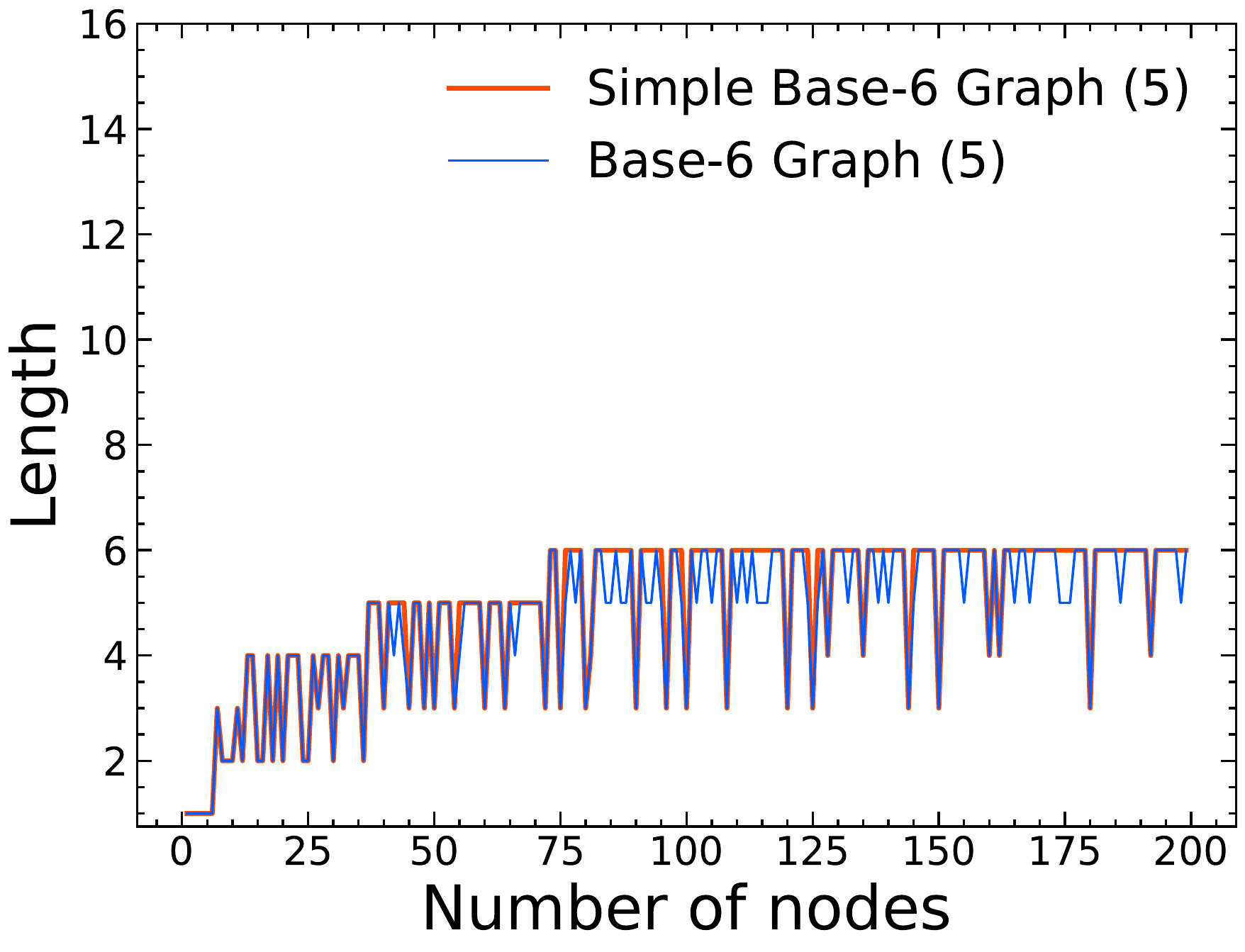}
     \end{subfigure}
    \begin{subfigure}{0.32\hsize}
         \centering
         \includegraphics[width=\hsize]{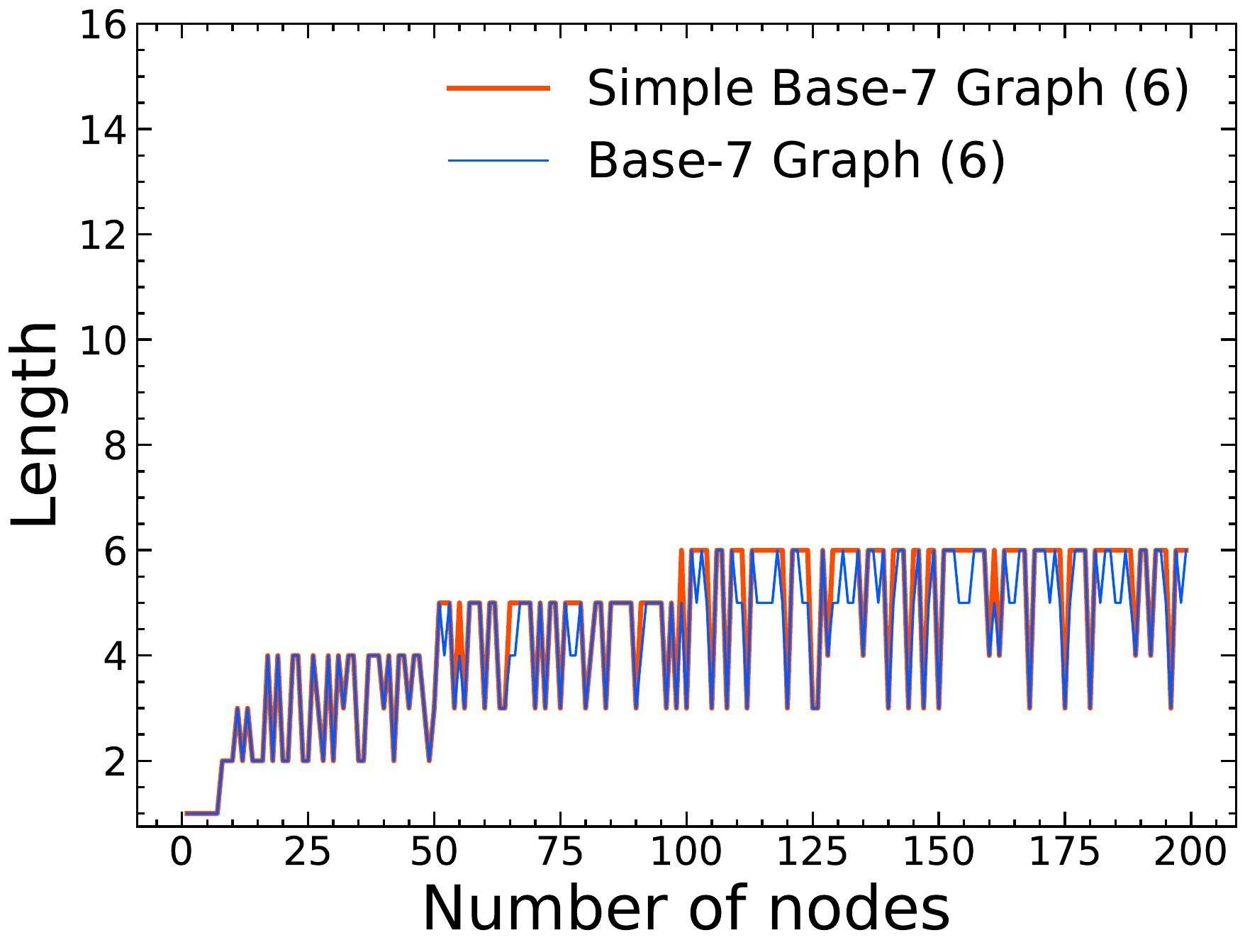}
     \end{subfigure}  
     \begin{subfigure}{0.32\hsize}
         \centering
         \includegraphics[width=\hsize]{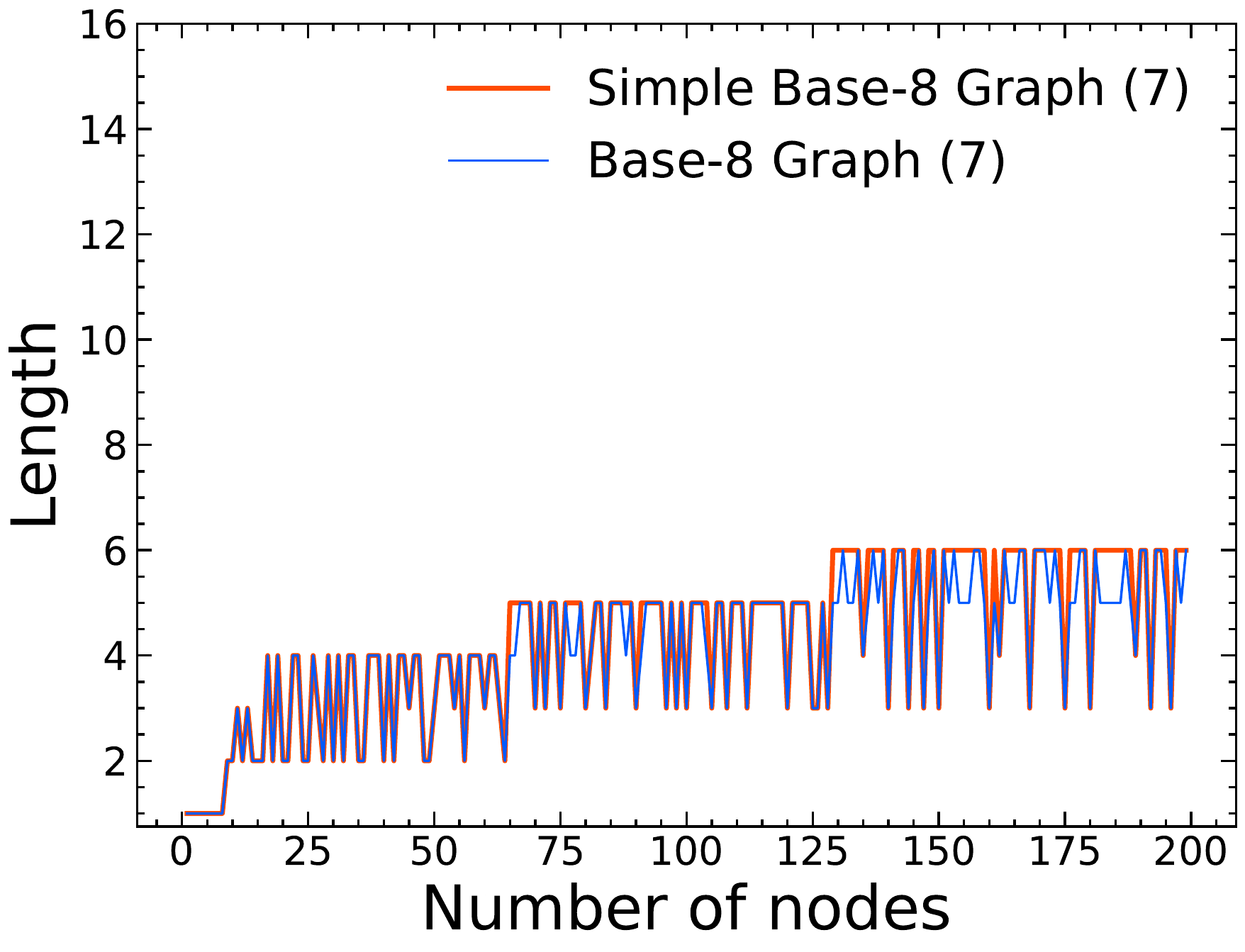}
     \end{subfigure}
    \begin{subfigure}{0.32\hsize}
         \centering
         \includegraphics[width=\hsize]{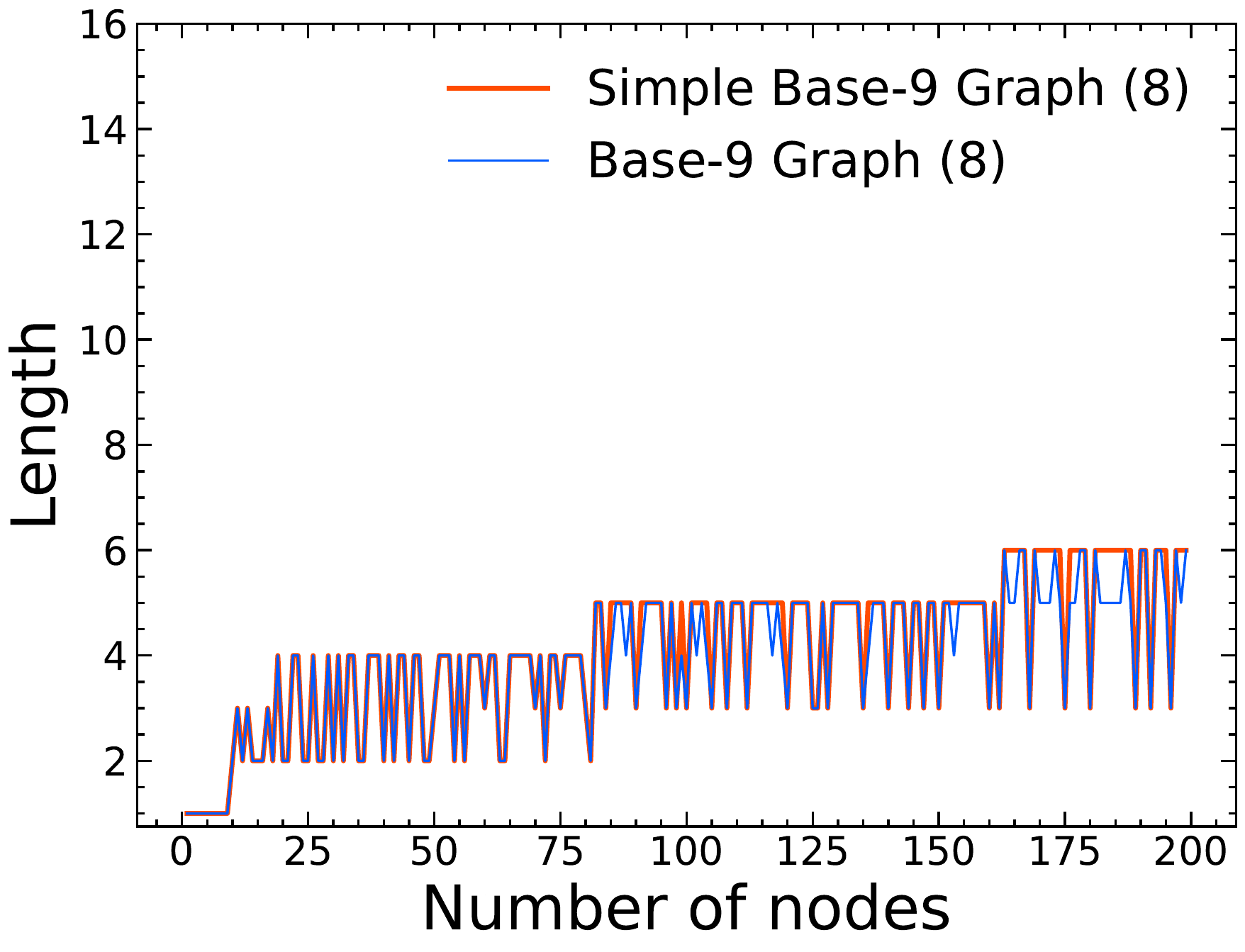}
     \end{subfigure}
    \begin{subfigure}{0.32\hsize}
         \centering
         \includegraphics[width=\hsize]{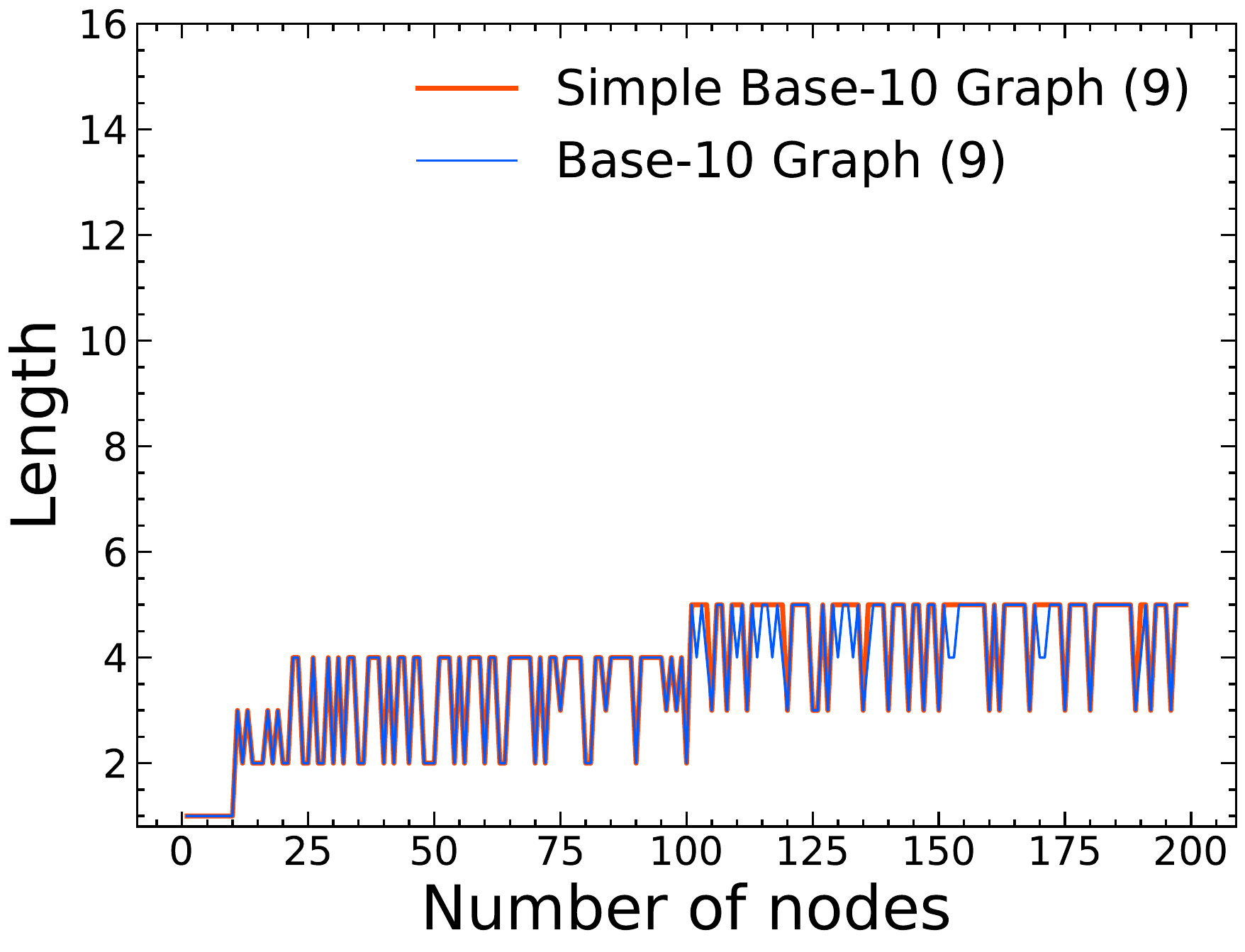}
     \end{subfigure}  
    \caption{Comparison of the length of the \simpleProposed{$(k+1)$} and \proposed{$(k+1)$}.}
    \label{fig:length}
    \vskip - 0.1 in
\end{figure}

\subsection{Consensus Rate}
\label{sec:other_consensns_optimization}

In Fig.~\ref{fig:consensus_with_power_of_two},
we demonstrate how consensus error decreases on various topologies when the number of nodes $n$ is a power of $2$.
The results indicate that the \proposed{$2$} and $1$-peer exponential graph can reach the exact consensus after the same finite number of iterations
and reach the consensus faster than other topologies.
Note that the \proposed{$2$} is equivalent to the $1$-peer hypercube graph when $n$ is a power of $2$.

\begin{figure}[h!]
    \centering
    \includegraphics[width=\hsize]{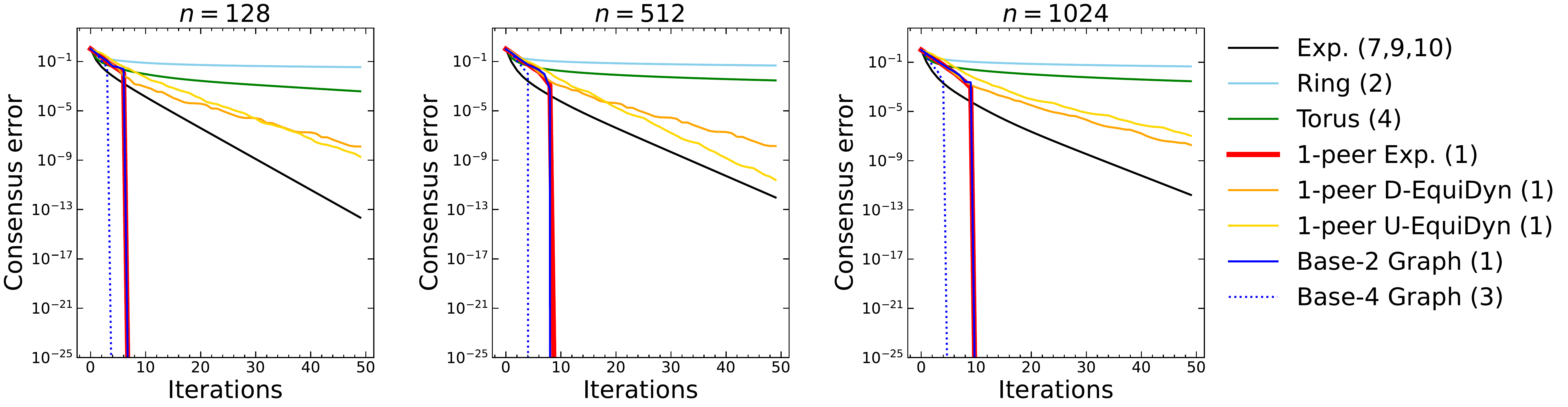}
    \caption{Comparison of consensus rates among different topologies when the number of nodes $n$ is a power of $2$. Because the \proposed{$\{3,5\}$} are the same as the \proposed{$\{2,4\}$}, respectively, when $n$ is a power of $2$, we omit the results of the \proposed{$\{3,5\}$}.}
    \label{fig:consensus_with_power_of_two}
\end{figure}

\subsection{Decentralized Learning}
\label{sec:additional_decentralized_learning}

\subsubsection{Comparison of Base-$(k+1)$ Graph and EquiStatic}
\label{sec:equistatic}

In this section, we compared the \proposed{$(k+1)$} with the \{U, D\}-EquiStatic \cite{song2022communication}.
The \{U, D\}-EquiStatic are dense variants of the $1$-peer \{U, D\}-EquiDyn,
and their maximum degree can be set as hyperparameters.
We evaluated the \{U, D\}-EquiStatic varying their maximum degrees;
the results are presented in Fig.~\ref{fig:equistatic}.
In both cases with $\alpha=10$ and $\alpha=0.1$,
the \proposed{$2$} can achieve comparable or higher final accuracy than all \{U, D\}-EquiStatic,
and the \proposed{$\{3,4,5\}$} outperforms all \{U, D\}-EquiStatic.
Thus, the \proposed{$(k+1)$} is superior to the \{U, D\}-EquiStatic from the perspective of achieving a balance between accuracy and communication efficiency.
\begin{figure}[h!]
    \centering
    \includegraphics[width=0.27\hsize]{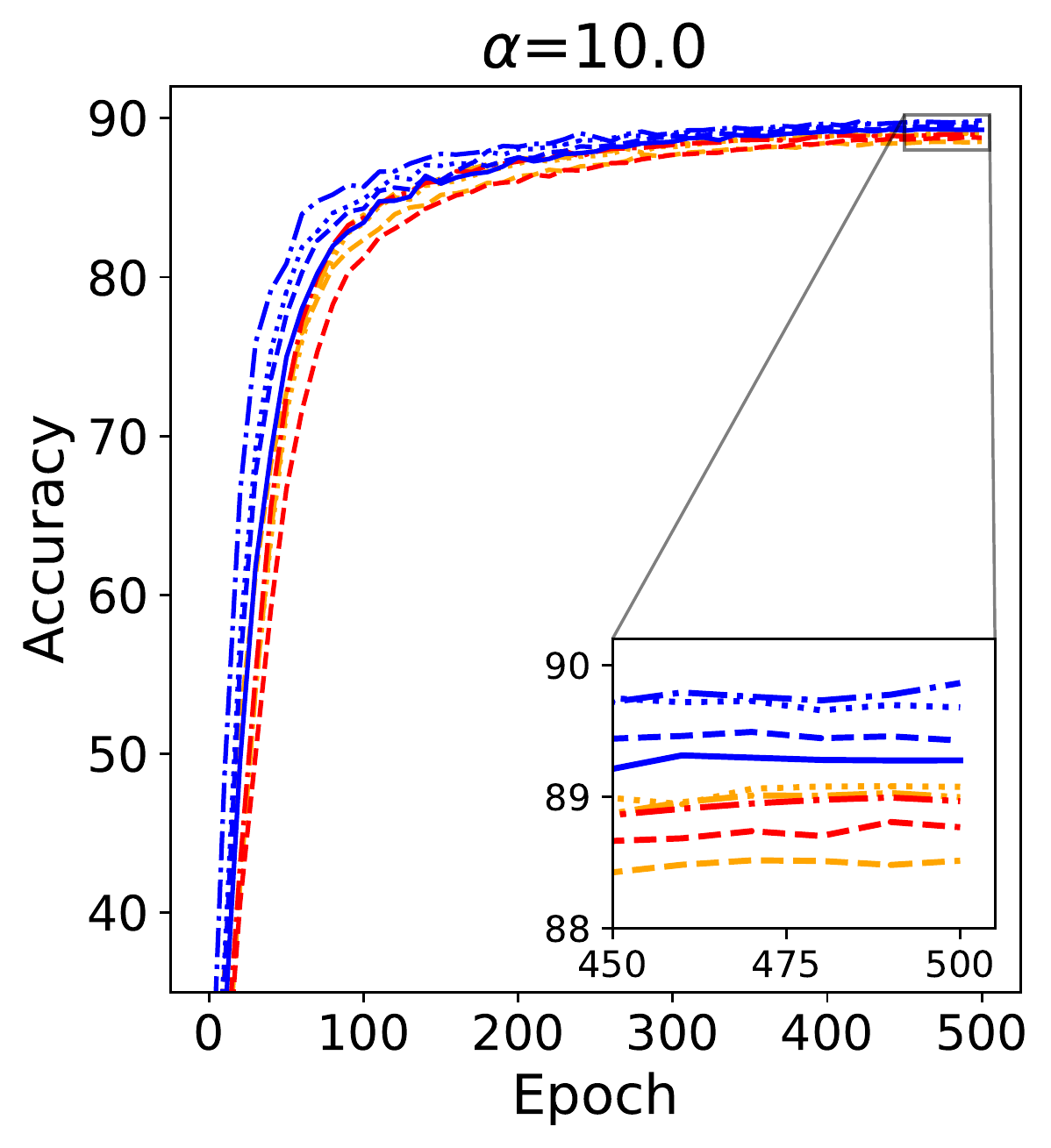}
    \includegraphics[width=0.455\hsize]{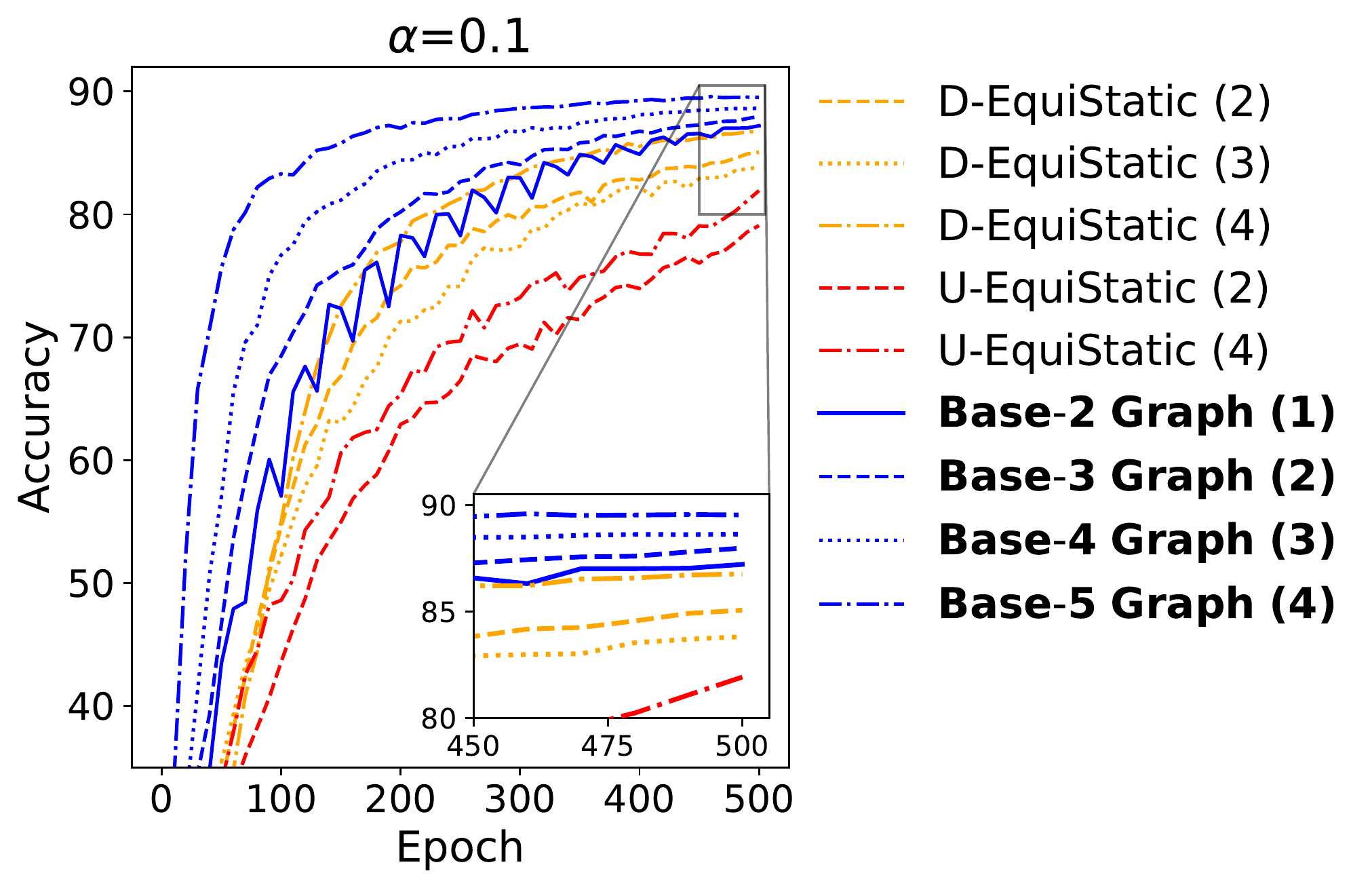}
    \vskip - 0.1 in
    \caption{Test accuracy (\%) of DSGD with CIFAR-10 and $n=25$. The number in the bracket is the maximum degree of a topology.}
    \label{fig:equistatic}
\end{figure}

\subsubsection{Comparison with Various Number of Nodes}
\label{sec:various_number_of_nodes}
In this section, we evaluated the effectiveness of the \proposed{$(k+1)$} when varying the number of nodes $n$.
Fig.~\ref{fig:various_number_of_nodes} presents the learning curves,
and Fig~\ref{fig:consensus_with_22_23_24_25} shows how consensus error decreases when $n$ is $21$, $22$, $23$, $24$, and $25$.
From Fig.~\ref{fig:various_number_of_nodes}, the \proposed{$2$} consistently outperforms the $1$-peer exponential graph
and can achieve a final accuracy comparable to that of the exponential graph.
Furthermore, the \proposed{$\{3,4,5\}$} can consistently outperform the exponential graph,
even though the maximum degree of the \proposed{$\{3,4,5\}$} is less than that of the exponential graph.

In Fig.~\ref{fig:results_with_16} presents the learning curve for $n=16$.
When the number of nodes is a power of two, the $1$-peer exponential graph is also finite-time convergence,
and the $1$-peer exponential graph and \proposed{$2$} achieve competitive accuracy.

\begin{figure}[h!]
    \centering
    \includegraphics[width=0.8\hsize]{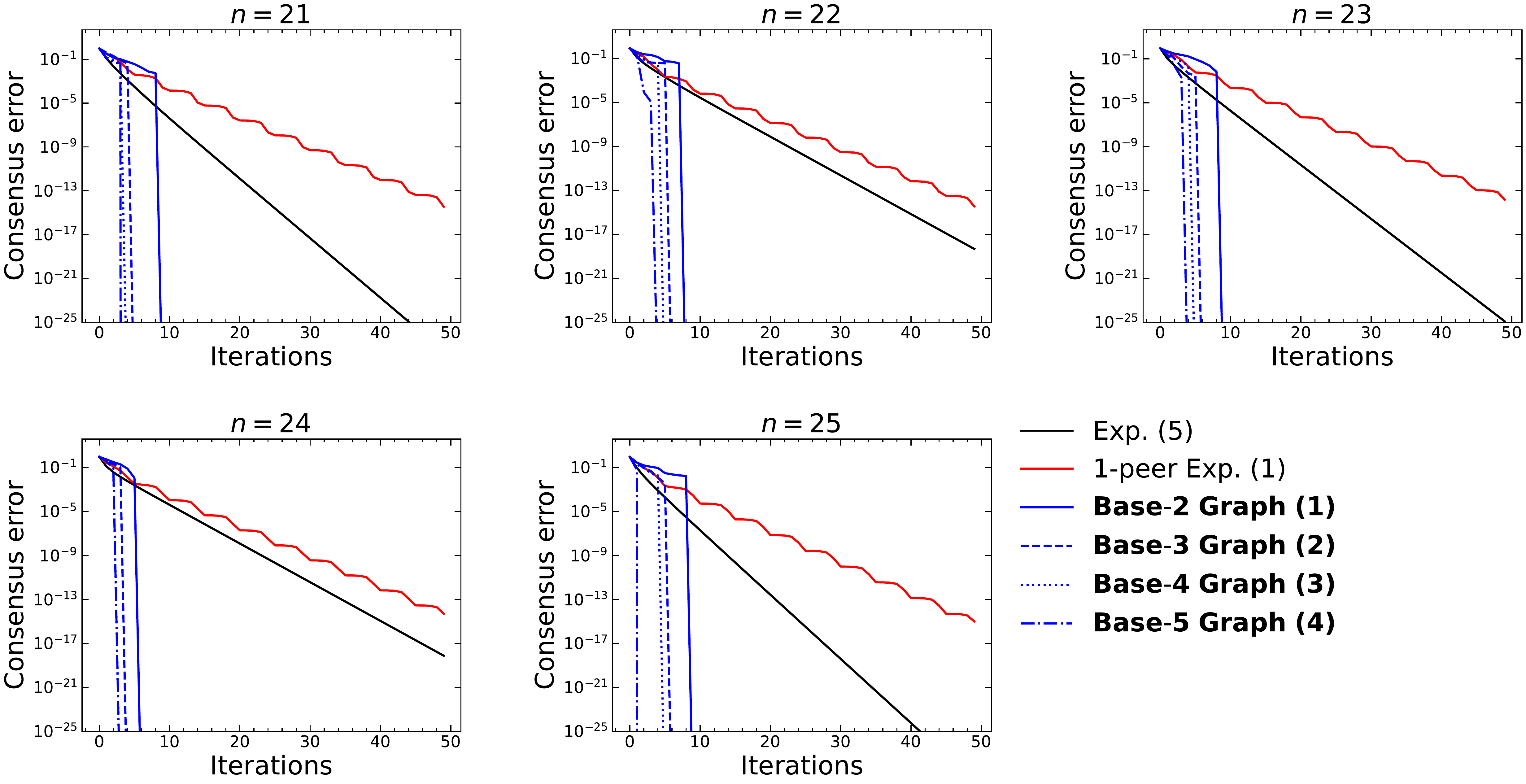}
    \caption{Comparison of consensus rates among different topologies. The number in the bracket denotes the maximum degree of a topology. We omit the results of the \proposed{$5$} when $n=24$ because the \proposed{$5$} and \proposed{$4$} are equivalent when $n=24$.}
    \label{fig:consensus_with_22_23_24_25}
\end{figure}

\begin{figure}[h!]
    \centering
    \includegraphics[width=0.28\hsize]{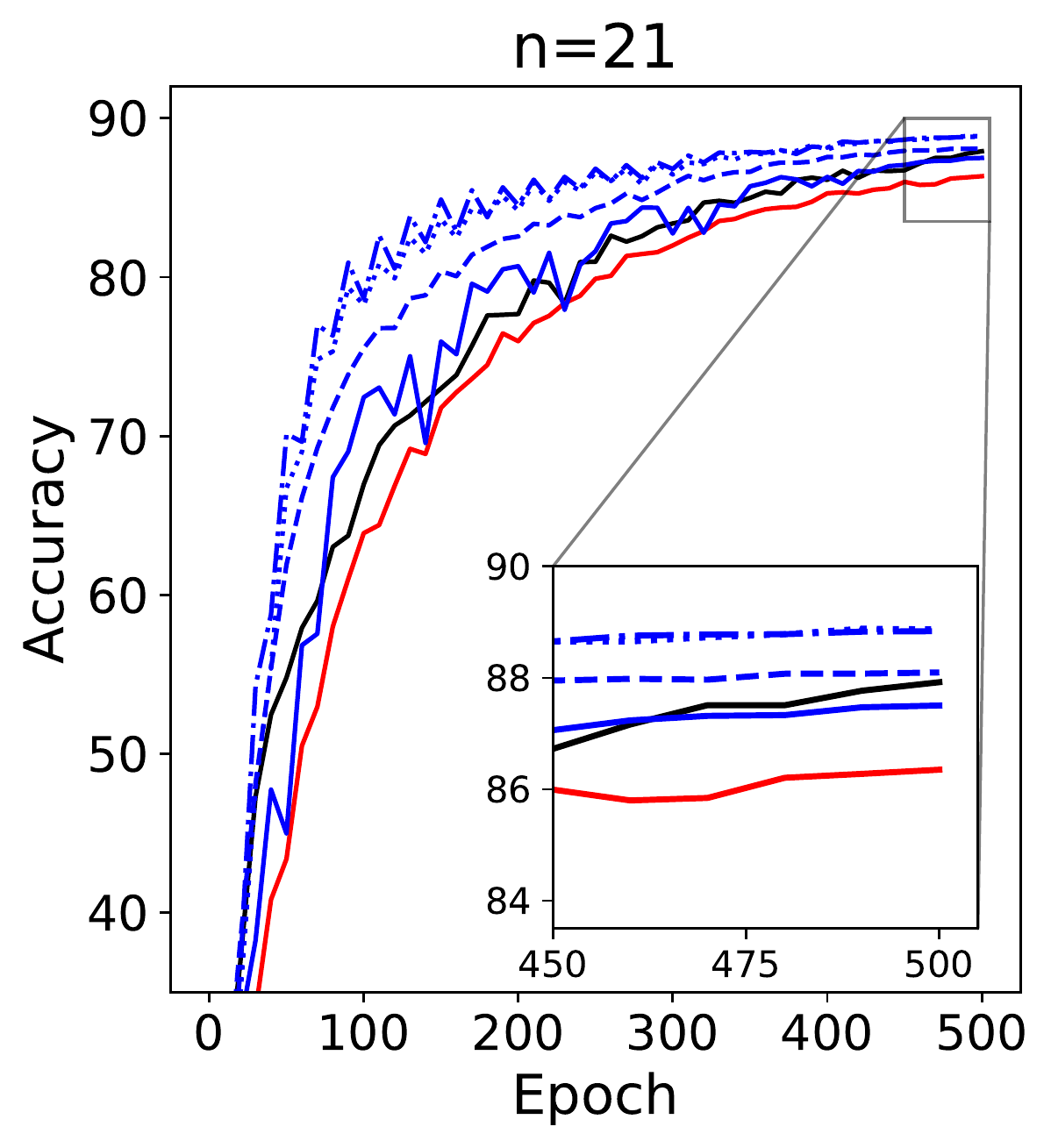}
    \includegraphics[width=0.48\hsize]{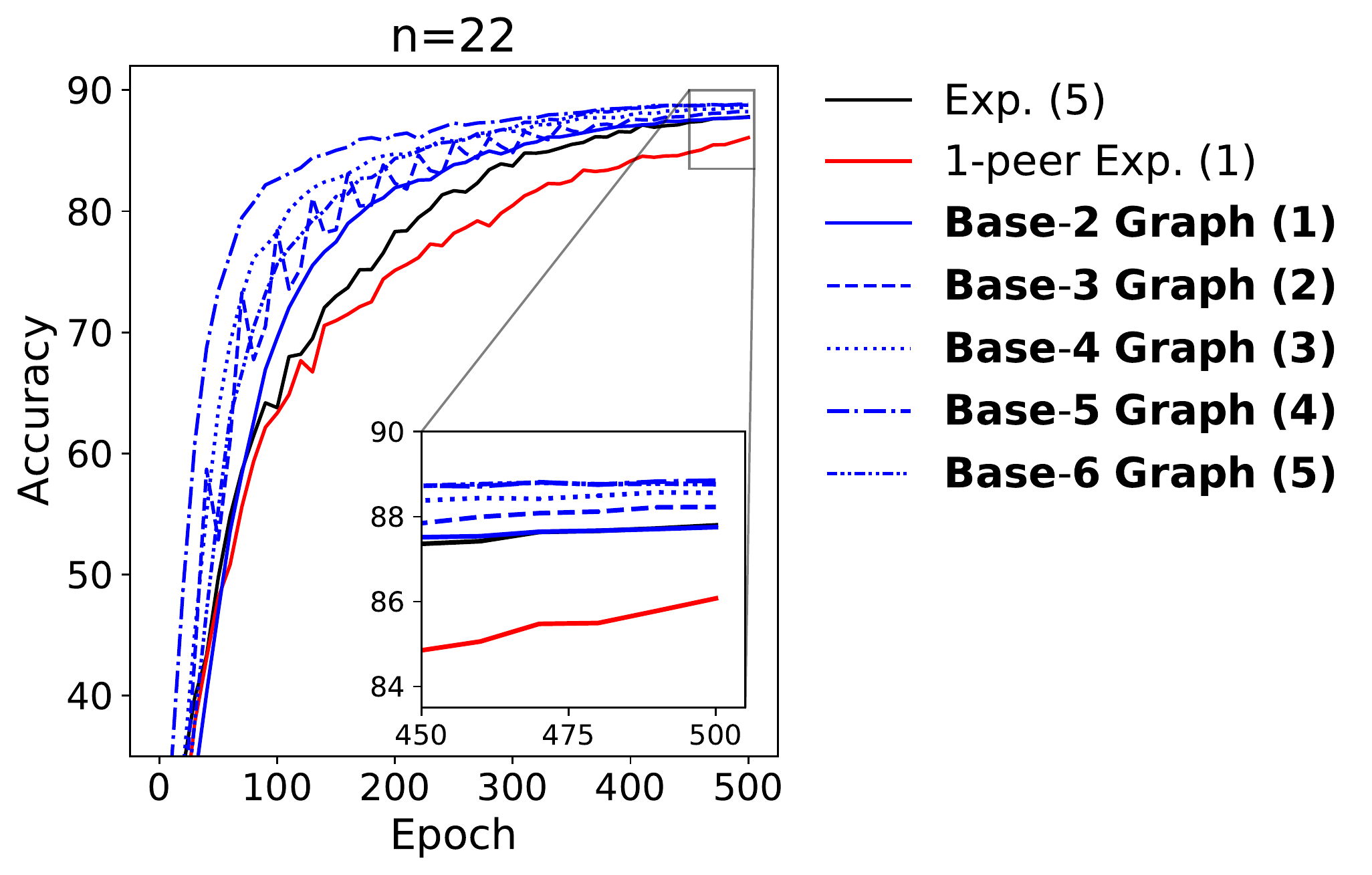}
    \includegraphics[width=0.28\hsize]{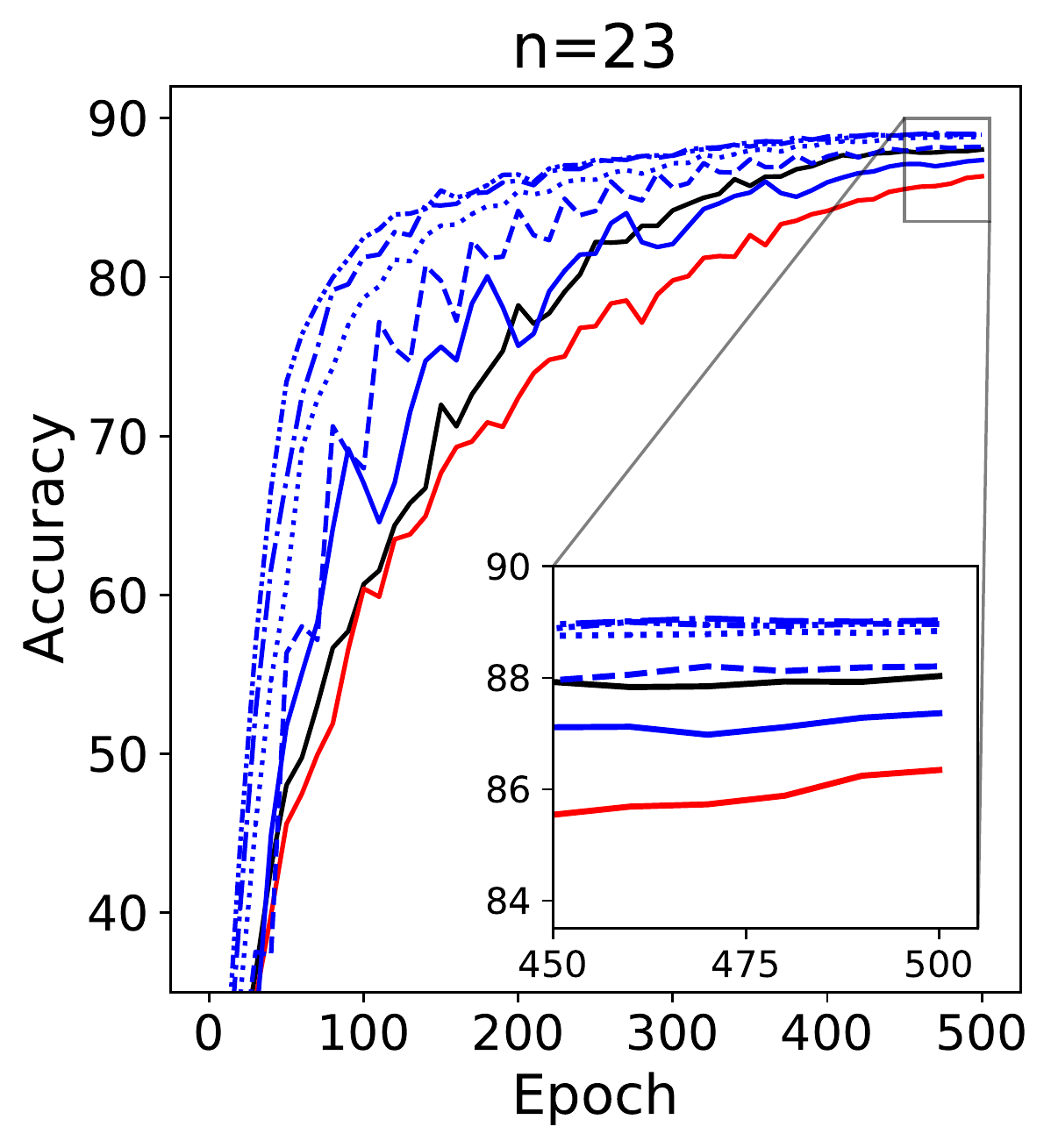} 
    \includegraphics[width=0.28\hsize]{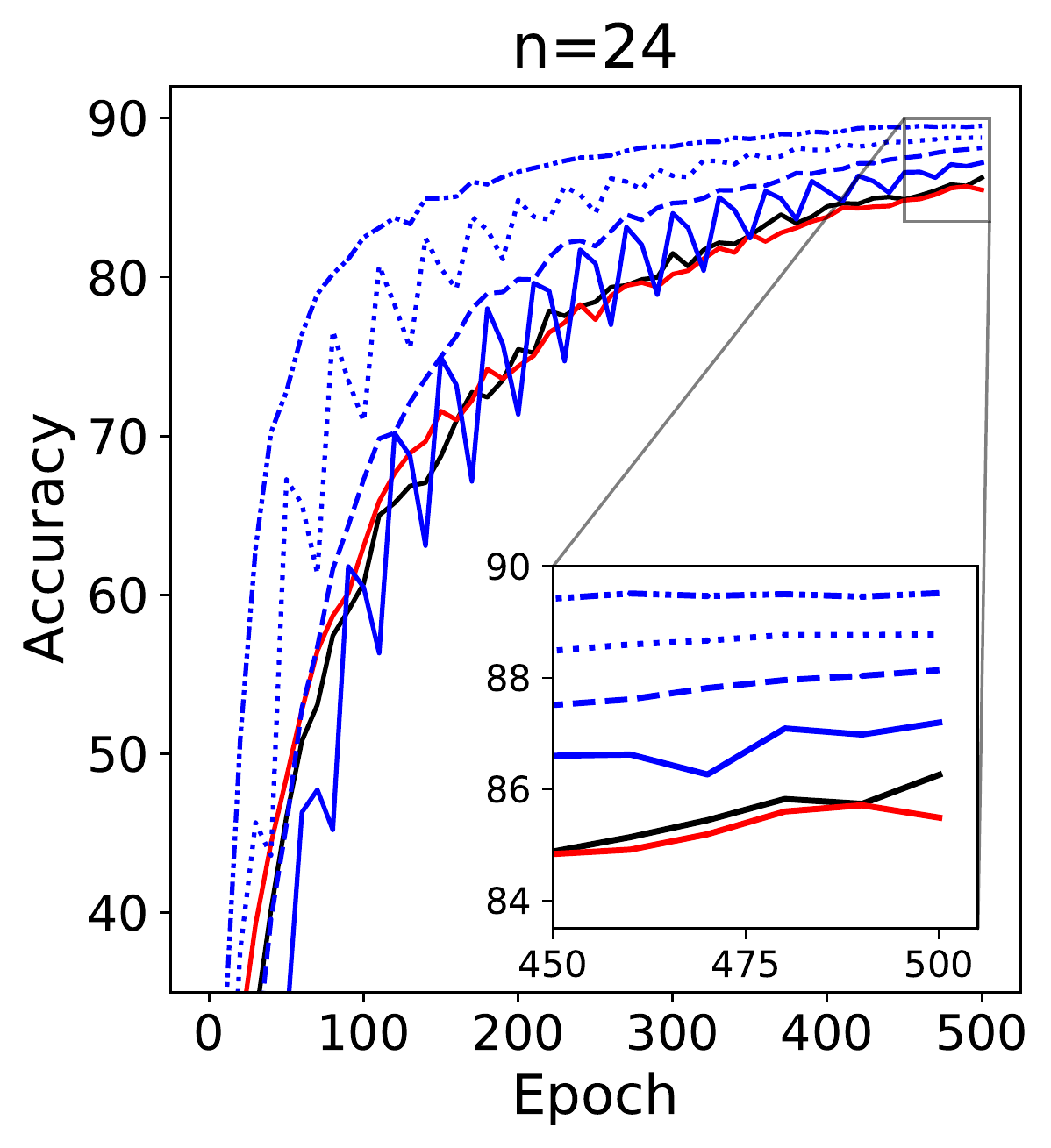}
    \includegraphics[width=0.28\hsize]{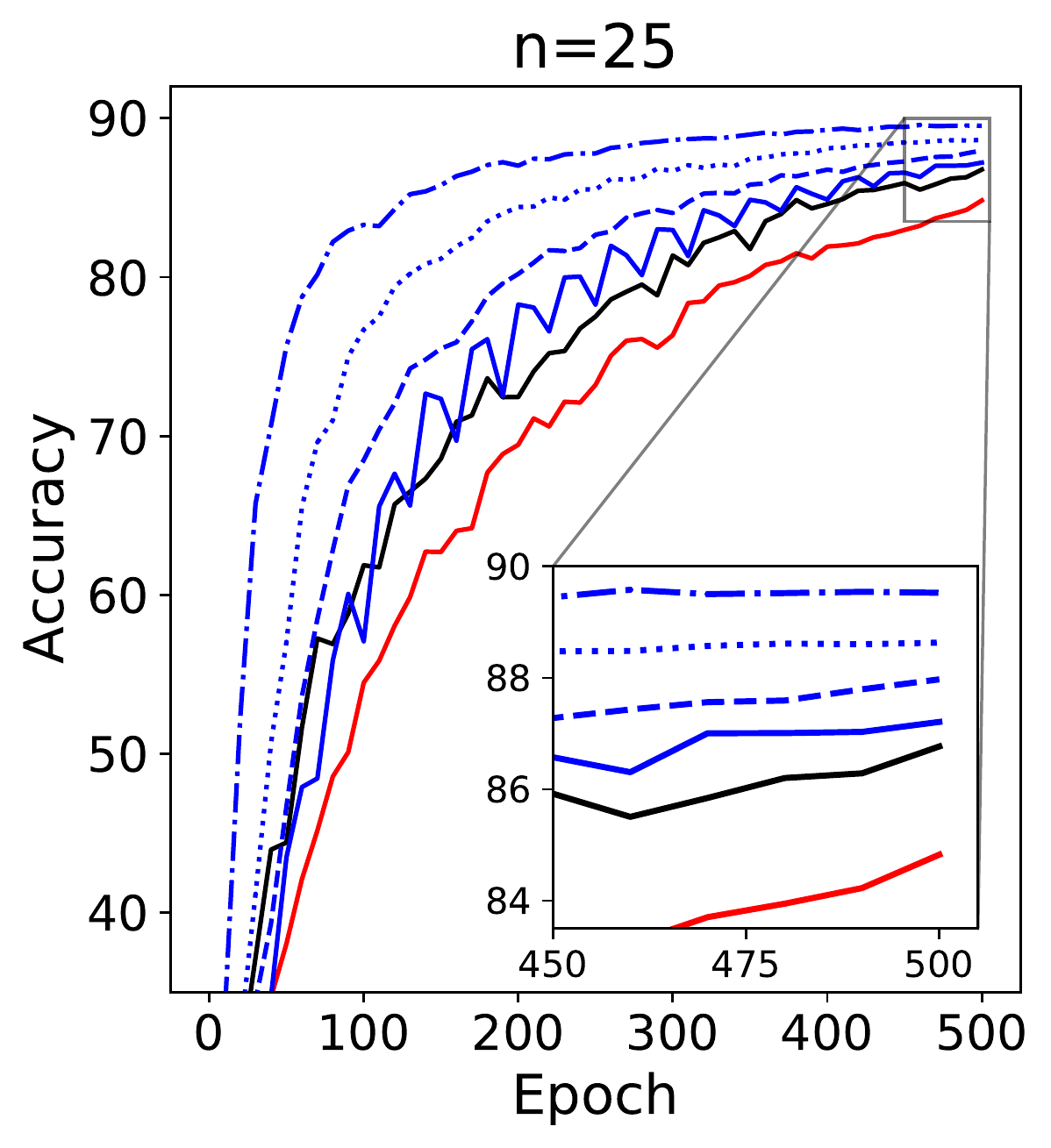}
    \vskip - 0.1 in
    \caption{Test accuracy (\%) of DSGD with CIFAR-10 and $\alpha=0.1$. The number in the bracket denotes the maximum degree of a topology. When $n=24$, we omit the results of the \proposed{$5$} because the \proposed{$5$} and \proposed{$4$} are equivalent. When $n=25$, we omit the results of the \proposed{$6$} because the \proposed{$6$} and \proposed{$5$} are equivalent.}
    \label{fig:various_number_of_nodes}
\end{figure}

\begin{figure}[h!]
    \centering
    \includegraphics[width=0.48\hsize]{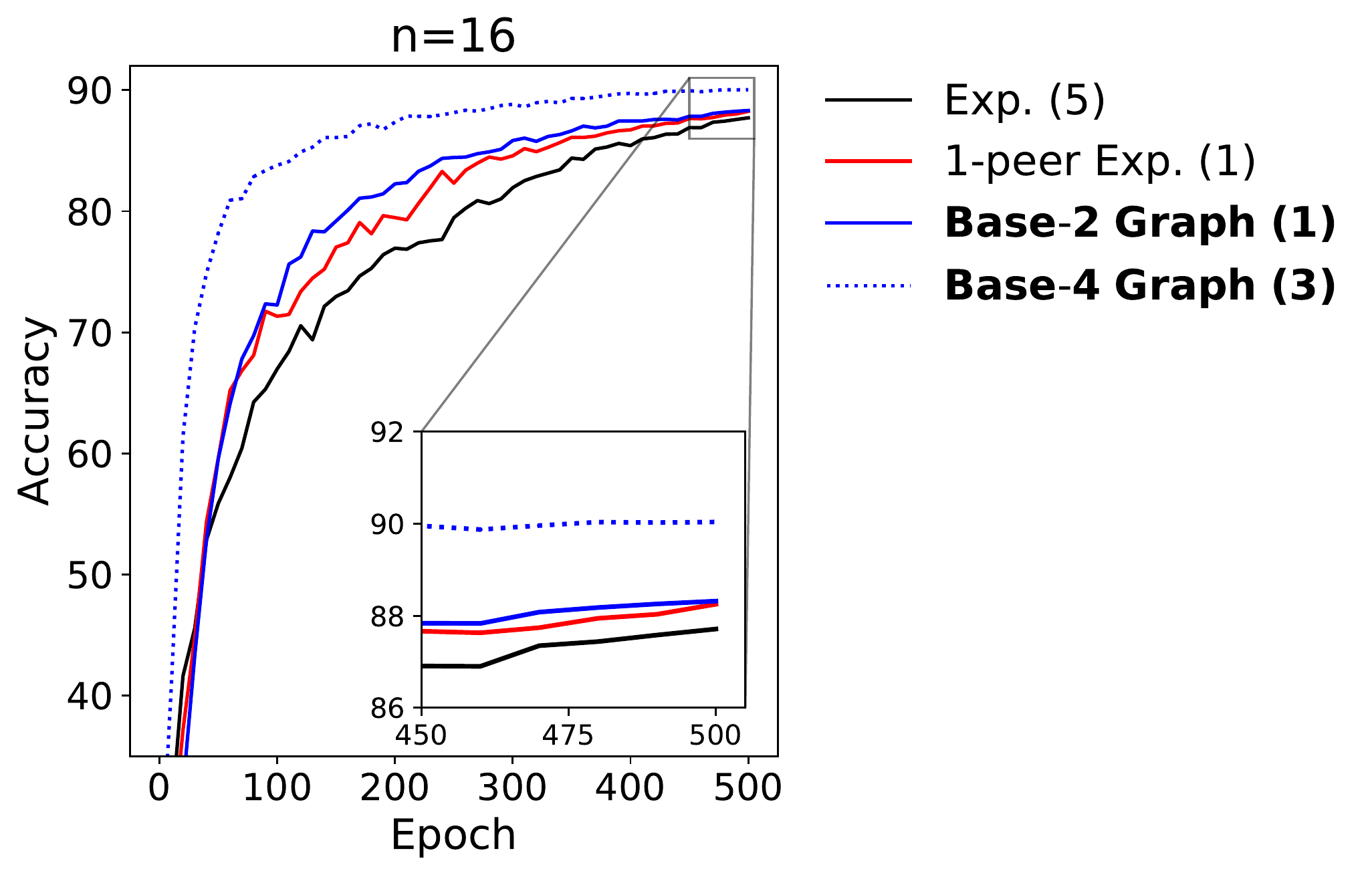}
    \caption{Test accuracy (\%) of DSGD with CIFAR-10 and $n=16$. The number in the bracket is the maximum degree of a topology. We omit the results of the \proposed{$3$} and \proposed{$5$} because these graphs are equivalent to the \proposed{$2$} and \proposed{$4$}, respectively.}
    \label{fig:results_with_16}
\end{figure}

\newpage
~
\newpage
\section{Results with Other Neural Network Architecture}
\label{sec:resnet}
In this section, we evaluate the effectiveness of the \proposed{$(k+1)$} with other neural network architecture.
Fig.~\ref{fig:resnet} shows the accuracy when we use ResNet-18 \cite{he2016deep}.
The results are consistent with the results with VGG-11 shown in Sec.~\ref{sec:experiment}.
\begin{figure}[h!]
    \centering
    \includegraphics[width=0.27\hsize]{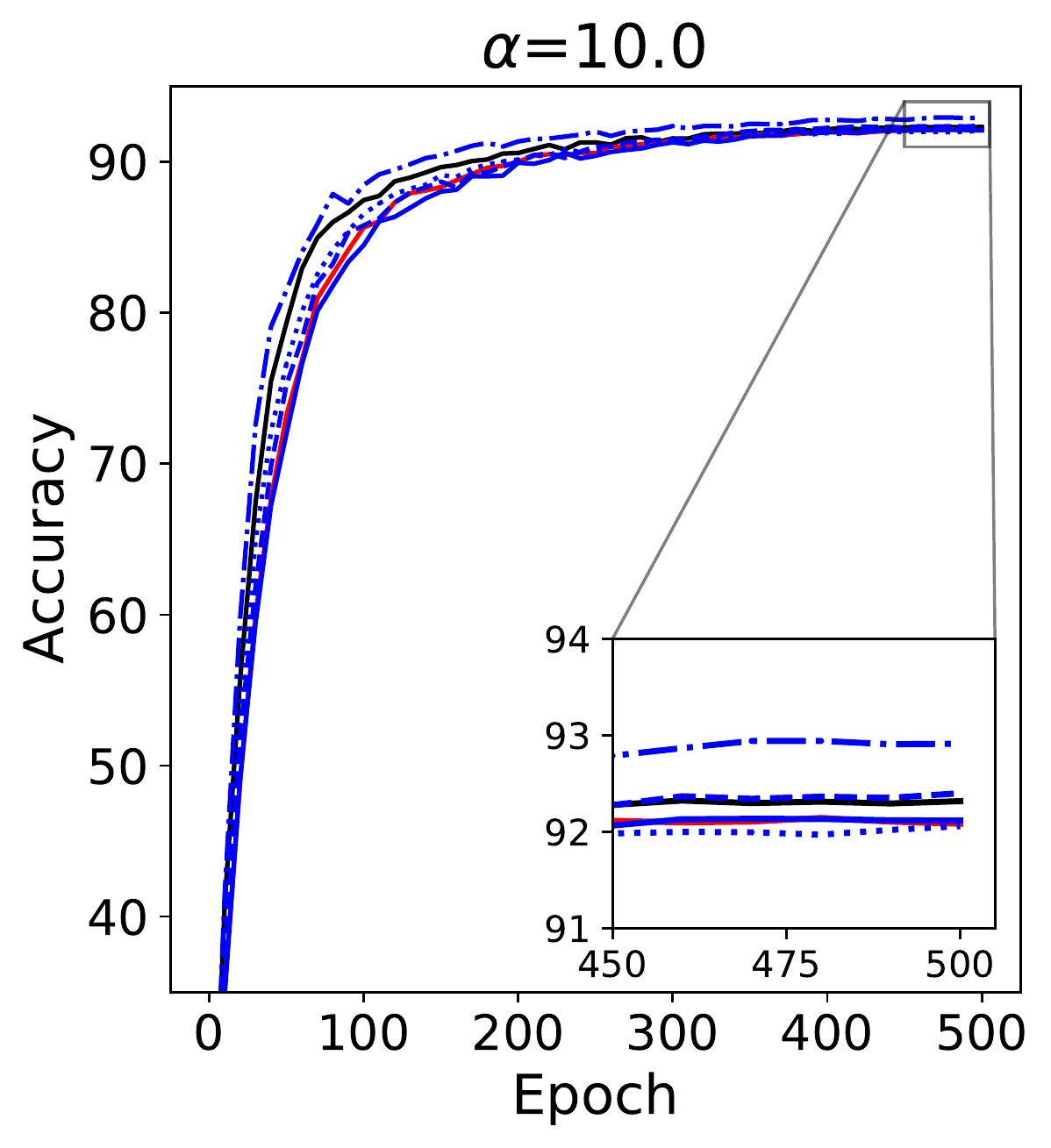}
    \includegraphics[width=0.46\hsize]{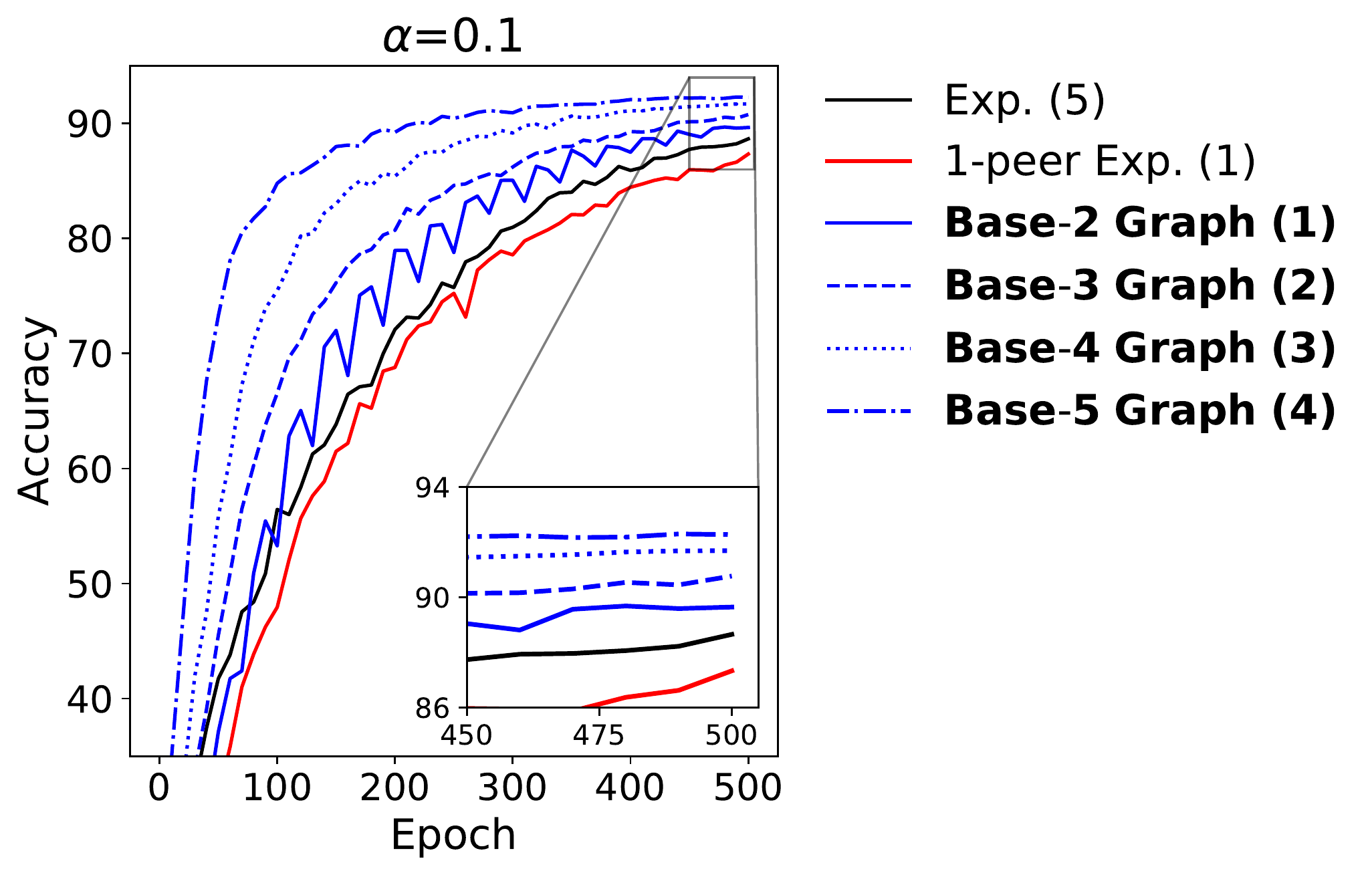}
    \vskip - 0.1 in
    \caption{Test accuracy (\%) of DSGD with $n=25$, CIFAR-10, and ResNet. The number in the bracket denotes the maximum degree of a topology.}
    \label{fig:resnet}
\end{figure}

\section{Hyperparameter Setting}
\label{sec:hyperparameter}

Tables \ref{table:hyperparameter_fashion} and \ref{table:hyperparameter_cifar} list the detailed hyperparameter settings used in Secs.~\ref{sec:experiment} and \ref{sec:additional_decentralized_learning}.
We ran all experiments on a server with eight Nvidia RTX 3090 GPUs.

\begin{table}[h!]
    \caption{Hyperparameter settings for Fashion MNIST with LeNet.}
    \label{table:hyperparameter_fashion}
    \centering
    \resizebox{0.7\linewidth}{!}{
    \begin{tabular}{ll}
    \toprule
     Dataset                     & Fashion MNIST \\
     Neural network architecture & LeNet \cite{lecun1998gradientbased} with group normalization \cite{wu2018group} \\
     \midrule
     Data augmentation & \textbf{RandomCrop} of PyTorch \\
     Step size    & Grid search over $\{ 0.1, 0.01, 0.001 \}$. \\
     Momentum     & $0.9$ \\
     Batch size   & $32$  \\
     Step size scheduler & Cosine decay \\
     Step size warmup &  10 epochs \\
     The number of epochs & $200$ \\
    \bottomrule
    \end{tabular}}
\end{table}

\begin{table}[h!]
    \caption{Hyperparameter settings for CIFAR-$\{10,100\}$ with $\{$VGG-11, ResNet-18$\}$.}
    \label{table:hyperparameter_cifar}
    \centering
    \resizebox{\linewidth}{!}{
    \begin{tabular}{ll}
    \toprule
     Dataset                     & CIFAR-$\{ 10, 100 \}$\\
     Neural network architecture & $\{$VGG-11 \cite{simonyanZ2014very}, ResNet-18 \cite{he2016deep}$\}$ with group normalization \cite{wu2018group} \\
     \midrule
     Data augmentation & \textbf{RandomCrop}, \textbf{RandomHorizontalFlip}, \textbf{RandomErasing} of PyTorch \\
     Step size    & Grid search over $\{ 0.1, 0.01, 0.001 \}$. \\
     Momentum     & $0.9$ \\
     Batch size   & $32$  \\
     Step size scheduler & Cosine decay \\
     Step size warmup &  10 epochs \\
     The number of epochs & $500$ \\
    \bottomrule
    \end{tabular}}
\end{table}

\end{document}